\documentclass[11pt]{article}

\usepackage{setspace}

\renewcommand{\baselinestretch}{1.2} 

\usepackage{geometry}
\usepackage{xfrac}
\usepackage{nicefrac}

\newcounter{daggerfootnote}
\newcommand*{\daggerfootnote}[1]{%
    \setcounter{daggerfootnote}{\value{footnote}}%
    \renewcommand*{\thefootnote}{\fnsymbol{footnote}}%
    \footnote[2]{#1}%
    \setcounter{footnote}{\value{daggerfootnote}}%
    \renewcommand*{\thefootnote}{\arabic{footnote}}%
    }


\usepackage{amsfonts,amssymb,amsthm}

\usepackage{graphicx}
\usepackage{amsfonts}
\usepackage{algorithm}
\usepackage[noend]{algpseudocode}

\renewcommand{\Pr}{\mathbb{P}}
\usepackage{dsfont}
\usepackage{amsmath, amssymb}
\usepackage{color}

\newcommand{\ident}{\mathds{1}}
\renewcommand{\H}{\mathcal{H}}
\renewcommand{\S}{\mathcal{S}}
\newcommand{\R}{\mathcal{R}}

\newcommand{\Rb}{\bar{\R}}
\newcommand{\Rt}{\widetilde{\R}}

\newcommand{\BB}{\mathcal{B}}
\newcommand{\BBil}[3]{\BB^{{#3}}_{{#1},\tau_I({#2})}}
\newcommand{\BBul}[3]{\BB^{{#3}}_{\tau_U({#1}),{#2}}}

\newcommand{\EEb}{\overline{\mathcal{E}}}

\newcommand{\Mcc}{\mathcal{M}^1} 
\newcommand{\Mcct}{\mathcal{M}^2} 
\newcommand{\St}[1]{\mathcal{S}^{#1}} 

\newcommand{\Fc}{\mathcal{F}} 
\newcommand{\Gc}{\mathcal{G}} 

\newcommand{\Lc}{\mathcal{L}}
\newcommand{\Pc}{\mathcal{P}}
\newcommand{\U}{\mathcal{U}}

\newcommand{\qu}{\texttt{q}_U}
\newcommand{\tu}{\textit{k}}

\newcommand{\qi}{\texttt{q}_I}
\newcommand{\ti}{\textit{j}} 

\newcommand{\tauh}{\widehat{\tau}}

\newcommand{\rep}{\textbf{\textmd{i}}} 

\newcommand{\eps}{\epsilon}
\newcommand{\calS}{\mathcal{S}}

\newcommand{\reg}{\mathtt{regret}}

\newcommand{\rated}{\mathsf{rated}}

\makeatletter
\newtheorem*{rep@theorem}{\rep@title}
\newcommand{\newreptheorem}[2]{%
\newenvironment{rep#1}[1]{%
 \def\rep@title{#2 \ref{##1}}%
 \begin{rep@theorem}}%
 {\end{rep@theorem}}}

\newtheorem{theorem}{Theorem}[section]
\newreptheorem{theorem}{Theorem}

\newtheorem{lemma}[theorem]{Lemma}

\newtheorem{claim}[theorem]{Claim}
\newtheorem{prop}[theorem]{Proposition}

\newcommand{\rgl}{\Omega}

\newcommand{\good}{\mathsf{good}}

\newcommand{\Asf}[1]{\mathtt{A{#1}}}

\newcommand{\Exp}[1]{\mathbb{E}\left[ #1 \right]}
\newcommand{\thr}{\mathsf{r}}

\newcommand{\thril}{\mathsf{r}}

\newcommand{\iuerr}{\mathsf{Err}}

\newcommand{\Ex}{\mathbb{E}}
\newcommand{\EE}{\mathcal{E}}

\renewcommand{\Pr}{\mathbb{P}}

\newtheorem{remark}{Remark}[section]
\theoremstyle{definition}
\newtheorem{definition}{Definition}[section]

\usepackage{enumitem}

\newcommand{\Tu}{T_{0}}

\begin{document}

\title{Regret Bounds and Regimes of Optimality for \\ User-User and Item-Item Collaborative Filtering}
\author{Guy Bresler\thanks{Guy Bresler is with the Department of Electrical Engineering and Computer Science at MIT and a member of LIDS and IDSS. {\tt\small guy@mit.edu} }  \ 
	and 
	Mina Karzand\thanks{Mina Karzand is with the Wisconsin Institute of Discovery at University of Wisconsin, Madison. 
	The work was done when the author was in the Department of Electrical Engineering and Computer Science at MIT and a member of LIDS and IDSS. 
	{\tt\small karzand@wisc.edu}}\thanks{The authors are ordered alphabetically.}
}

\maketitle

\begin{abstract} We consider an online model for recommendation systems, with each user being recommended an item at each time-step and providing `like' or `dislike' feedback. Each user may be recommended a given item at most once. A latent variable model specifies the user preferences: both users and items are clustered into \emph{types}. All users of a given type have identical preferences for the items, and similarly, items of a given type are either all liked or all disliked by a given user. We assume that the matrix encoding the preferences of each user type for each item type is randomly generated; in this way, the model captures structure in both the item and user spaces, the amount of structure depending on the number of each of the types. 
The measure of performance of the recommendation system is the expected number of disliked recommendations per user, defined as expected regret.
We propose two algorithms inspired by user-user and item-item collaborative filtering (CF), modified to explicitly make exploratory recommendations, and prove performance guarantees in terms of their expected regret. 
For two regimes of model parameters, with structure only in item space or only in user space, we prove information-theoretic lower bounds on regret that match our upper bounds up to logarithmic factors. Our analysis elucidates system operating regimes in which existing CF algorithms are nearly optimal.
\end{abstract}

\section{Introduction}
Options are good, but if there are too many options, we need help. It is increasingly the case that our interaction with content is mediated by recommendation systems. There are two main approaches taken in recommendation systems: \emph{content filtering} and \emph{collaborative filtering}. Content filtering makes use of features associated with items and users (e.g., age, location, gender of users and genre, actors, director of movies). In contrast, collaborative filtering is based on observed user preferences. Thus, two users are thought of as similar if they have revealed similar preferences, irrespective of their profile. Likewise, two items are thought of as similar if most users have similar preferences for them. 
More generally, collaborative filtering (CF) makes use of structure in the matrix of preferences, as in low-rank matrix formulations~\cite{aditya2011channel,biau2010statistical,
candes2009exact,jain2013low,keshavan2010matrix,negahban2012restricted,
rohde2011estimation,srebro2005generalization}. In this paper, since our model has no item and user features, all algorithms must do collaborative filtering.

An important aspect of most recommendation systems is that each recommendation influences what is learned about the users and items, which in turn determines the possible accuracy of future recommendations.
This introduces a tension between exploring to obtain information and exploiting existing knowledge to make good recommendations. The tension between exploring and exploiting is exactly the phenomenon of interest in the substantial literature on the multi-armed bandit (MAB) problem and its variants \cite{bubeck2012regret,lai1985asymptotically,russo2014learning}.
In the multi-armed bandit setup, optimal algorithms necessarily converge to repeated play of the same arm; in contrast, a recommendation system that repeatedly recommends the same movie, even if it is a very good movie, is surely problematic!
For this reason we will allow each item to be recommended at most once to each user (as done in~\cite{bresler2014latent,bresler2015regret}). 

It is common to think of recommendation systems as a matrix completion problem. Given a subset of observed entries, the matrix completion problem is to estimate the rest of matrix, where it is assumed that the matrix satisfies some properties. This criterion does not capture the experience of users in a recommendation system: a more appropriate measure of performance is the proportion of good recommendations made by the algorithm.

With the aforementioned issues in mind, we work within a mathematical framework for evaluating the performance of various recommendation system algorithms, related to the models studied in \cite{bresler2014latent,bresler2015regret}. The framework is detailed in Section~\ref{sec:model}, but in brief, at each time-step each user in the system is given a recommendation and then provides binary feedback in the form of 'like' or 'dislike'. 
The user preferences are described by a latent variable model in which each user is associated with a user type and each item is associated with an item type. 
Users who belong to the same user type have identical preferences for all items and items belonging to the same type have identical ratings from all users.\footnote{A similar model of data to ours, in which there is an underlying clustering of  rows and columns, has been studied in other settings~\cite{shen2009mining,xu2014jointly}.}
The basic measure of performance is \emph{expected regret}, defined as the expected number of bad recommendations made per user over a time horizon of interest. A second performance criterion is the \emph{cold start time}, the first time at which recommendations become nontrivial in quality. Our goal is to understand the dependence of these quantities on system parameters and we will therefore seek  bounds accurate only to within constant or logarithmic factors. 
                          
In the literature there are two categories of collaborative filtering (CF) algorithms.
\emph{User-user} algorithms \cite{bellogin2012using,bresler2014latent,das2007google} use structure in the user space to predict user preferences. Here, the preference of user $u$ for item $i$ is estimated from the preference of other users $u'$ believed to be similar to $u$ based on their previous ratings.
 Alternatively, \emph{item-item} algorithms \cite{bresler2015regret,linden2003amazon,sarwar2001item} use structure in the item space. This time, the preference of user $u$ for item $i$ is estimated from the preference of the same user $u$ for other items $i'$ believed to be similar to $i$ based on previous ratings from users that have rated both $i$ and $i'$. 
 In Sections~\ref{s:uu-upper} and~\ref{s:itemitem} we develop versions of user-user and item-item CF algorithms tailored to our online recommendation system model and prove performance guarantees. In order to achieve good performance, these algorithms must carefully explore the \emph{a priori} unknown relationships between users and items.  
 One of the unexpected insights that emerge from the analysis is that the item-item algorithm must limit the exploration to only a subset of the items types, where the size of this subset depends on the system parameters and time-horizon. The straightforward approach to Item-Item CF algorithms is to learn the whole preference matrix, and as described in Section~\ref{s:results} this results in a qualitatively suboptimal cold-start time that can be arbitrarily worse than the one obtained by our algorithm.

In order to focus on the information structure of the recommendation problem, and the associated exploration-exploitation tradeoff, the majority of the paper assumes that user feedback is noiseless. We generalize our user-user algorithm to handle noisy feedback and also describe how one would similarly accommodate noisy feedback in the item-item algorithm. In essence, estimation of similarity between users (or items) requires some redundancy in the information collected in order to average out the noise.

We prove nearly tight lower bounds on regret for two parameter regimes of interest, identifying settings in which the proposed algorithms cannot be significantly improved. In the \emph{user structure only} scenario, the model parameters are such that there is no structure in the item space. Analogously, in the \emph{item structure only} scenario, the parameters are such that there is no structure in the user space. We prove information-theoretic lower bounds for the performance of any algorithm in the user-structure only and item-structure only models, which match to within a logarithmic factor the performance obtained by our proposed user-user and item-item CF algorithms. These results are outlined in Section~\ref{s:results}.

One of this paper's main contributions is the development of techniques for proving lower bounds on the performance of online recommendation algorithms. Our lower bounds depend crucially on the inability to repeatedly recommend the same item to a given user, and for this reason are completely different from lower bounds for multi-armed bandit problems~\cite{lai1985asymptotically,bubeck2012regret}.
At a high level, however, the basic challenge is the same as when proving lower bounds for bandits: one must connect the information obtained by the algorithm to the regret incurred. This allows to reason that subsequent recommendations will have low regret only if prior recommendations yielded significant information, which in turn necessitated exploratory recommendations with correspondingly substantial regret. Thus, regret is a conserved quantity and cannot be avoided by employing complicated adaptive algorithms.

 The methods used for the lower bounds are elementary in nature. For example, in the user structure only model, the arguments in Section~\ref{s:item_lower}
are based on two observations. First, one cannot be confident in recommending any item to user $u$ at time $t$ if there is no user $u'$ that has rated enough items in common, and in agreement, with user $u$ by time $t-1$. In this situation, the similarity of $u$ to any other user is uncertain and so too is the outcome of any recommendation. Second, the outcome of recommending item $i$ to user $u$ at time $t$ is also uncertain if none of the users that actually are similar to $u$ have rated item $i$ by time $t-1$. These observations imply a lower bound on the necessary number of exploratory recommendations before it is possible to recommend with much better likelihood of success than chance. Similar reasoning leads to lower bounds in Section~\ref{sec:UL}
for the model with only item-structure.

 A few papers including \cite{bresler2014latent,bresler2015regret,dabeer2013adaptive}  have theoretical analyses for online collaborative filtering. The paper~\cite{bresler2014latent} analyzes a user-user CF algorithm in a similar setting to ours and \cite{bresler2015regret} analyzes an item-item CF algorithm in a somewhat different and more flexible model. Relative to these, our main distinction is obtaining nearly matching lower bounds showing optimality of our algorithms and analysis. 
 The model studied by Dabeer and coauthors~\cite{dabeer2013adaptive,aditya2011channel,barman2012analysis} is also quite similar to our setup, but their objective is different: 
 they seek an algorithm that \textit{exploits} in a provably optimal fashion asymptotically in time, but their approach does not reveal how to explore. In a different direction, Kerenidis and Prakash~\cite{kerenidis2016quantum} seek to achieve low computational complexity for recommendation in a similar setup as ours. What they show is that reconstructing the preference matrix only partially, which is what our item-item CF algorithm does, is useful also with regards to computation. 

Hybrid algorithms exploiting both structure in user space and item space have been studied before in \cite{song2016blind,wang2006unifying,kim2010imp}. Both Song et al. \cite{song2016blind} and 
Borgs et al. \cite{borgs2017thy}
study a more flexible latent variable model 
in the offline (matrix completion style) setting and propose collaborative filtering algorithms using both item and user space. In a forthcoming paper we analyze a hybrid algorithm within the same framework studied here. 

 \subsection{Outline}
 The model and performance metric are described in Section~\ref{sec:model}. Section~\ref{s:results} overviews the main results of this paper and includes numerical simulations to complement the theoretical analyses.
 Our version of user-user CF is introduced and analyzed in Section~\ref{s:uu-upper}.
 In Section~\ref{sec:UL} we prove that the proposed algorithm is almost information-theoretically optimal in the setup with user structure only. Our version of item-item CF is described and analyzed in Section~\ref{s:itemitem}, and the corresponding lower bound in the setting with item structure only is given in Section~\ref{s:item_lower}. Appendix~\ref{s:lemmas} contains a few basic probabilistic lemmas, and Appendix~\ref{s:any-time-reg-alg} relates so-called anytime regret (unknown time horizon) to known time horizon. 
 	
\subsection{Notation}

For an integer $a$ we write $[a]=\{1,\cdots,a\}$ and for real-valued $x$ let $(x)_+ =\max\{x,0\}$. All logarithms are to the base of $2$. The set of natural numbers (positive integers) is denoted by $\mathbb{N}$. We note here that variables or parameters in Figure~\ref{f:notation} have the same meaning throughout the paper, but any others may take different values in each section. For real-valued $x$, $\lfloor x \rfloor$ denotes the greatest integer less than or equal to $x$ and $\lceil x \rceil$ denotes the smallest integer greater than or equal to $x$. Numerical constants ($c, c_1,c_2$ and so forth) may take different values in different theorem statements unless explicitly stated otherwise.

\section{Model} \label{sec:model}
\subsection{Problem setup}

There is a fixed set of users $\{1,\dots,N\}$. At each time $t= 1,2,3,\dots$ the algorithm recommends an item $a_{u,t}\in\mathbb{N}$ to each user $u$ and receives feedback $L_{u,a_{u,s}}\in \{+1,-1\}$ (`like' or `dislike'). For the reasons stated in the introduction, we impose the condition that each item may be recommended at most once to each user. In order that the algorithm never run out of items to recommend, we suppose there are infinitely many items to draw from and identify them with the natural numbers. 

The history $\H_{t}=\{a_{u,s}, L_{u,a_{u,s}}, \text{ for } u\in [N], s\in [t]\}$ is the collection of actions and feedback up to time $t$.
We are interested in online learning algorithms, in which the action $a_{u,t}$ is a (possibly random) function of the history up through the end of the previous time-step $\H_{t-1}$. This additional randomness is encoded in a random variable $\zeta_{u,t}$, assumed to be independent of all other variables. In this way, $a_{u,t}=f_{u,t}(\H_{t-1},\zeta_{u,t})$, for some deterministic function $f_{u,t}$.

Algorithm performance will be evaluated after some arbitrary number  of time-steps $T$. The performance metric we use is expected regret (simply called regret in what follows), defined as the expected number of disliked items recommended per user: 
\begin{equation}\label{eq:regdef}
\reg(T)=\Ex\sum_{t=1}^{T}\frac{1}{N}\sum_{u=1}^{N}\ident[L_{u,a_{u,t}}=-1]\,.
\end{equation}
Here the expectation is with respect to the randomness in both the model and the algorithm.
The algorithms we describe depend on knowing the time-horizon $T$, but by a standard doubling trick (explained in Appendix~\ref{s:any-time-reg-alg}) it is possible to convert these to algorithms achieving the same (up to constant factors) regret without this knowledge~(see, e.g., \cite{cesa2006prediction}). This latter notion of regret, where the algorithm does not know the time-horizon of interest and must achieve good performance across all time-scales, is called \emph{anytime regret} in the literature. 

The time at which point recommendations become nontrivial in quality is another important performance criterion, because until that point users invest effort but get little in return. In the recommendation systems literature the notion of cold start describes the difficulty of providing useful recommendations when insufficient information is available about user preferences. 
We define the \emph{cold start time} to be the first time at which the slope of regret as a function of $T$ is bounded by some value $\gamma$:
$$
\textsf{coldstart}(\gamma) = \min\Big\{ T:\frac{\reg(T)}{T}\leq \gamma\Big\}\,.
$$
This is similar to (but somewhat simpler than) the definition in \cite{bresler2015regret}. 

\subsection{User preferences}\label{ss:pref}
We study a latent-variable model for the preferences (`like' or `dislike') of the users for the items, based on the idea that there are relatively few \emph{types of users} and/or few \emph{types of items}.
 Each user $u \in [N]$ has a user type $\tau_U(u)$ i.i.d. uniform on $[\qu]$, where $\qu$ is the number of user types. 
We assume that $\qu\leq N$, because if $\qu> N$ then most users have their own type and all of the results remain unchanged upon replacing $\qu$ by $N$.
 Similarly, each item $i\in\mathbb{N}$ has a  random item type $\tau_I(i)$ i.i.d. uniform on $[\qi]$, where $\qi$ is the number of item types\footnote{Throughout the paper, we will assume that the number of user types $\qu$ and the number of item types $\qi$ are both $\Omega(\log N)$, since the whole problem becomes easy if either of these parameters are too small.}.
The random variables $\{\tau_{U}(u)\}_{1\leq u \leq N}$ and $\{\tau_I(i)\}_{1\leq i}$ are assumed to be jointly independent. 

\begin{figure}
	\begin{center}
		\begin{tabular}{ |c|c| } 
			\hline
			$N$ & Number of users \\ 
			$T$ & Time horizon \\ 
			$\qu$ & Number of user types \\ 
			$\qi$ & Number of item types \\ 
			$\tau_U(u)$ & User type of user $u$ \\ 
			$\tau_I(i)$ & Item type of item $i$ \\ 
			$a_{u,t}$ & Item recommended to user $u$ at time $t$ \\ 
			$L_{u,i}$ & Rating of user $u$ for item $i$\\
			$\xi_{\tu,\ti}$ & Preference of user type $\tu$ for item type $\ti$\\
			$\Xi$ & Preference matrix \\
			\hline
		\end{tabular}
	\end{center}
	\caption{Notation for the recommendation system model.}
	\label{f:notation}
\end{figure}
 
All users of a given type have identical preferences for all the items, and similarly all items of a given type are rated in the same way by any particular user. The entire collection of user preferences $(L_{u,i})_{u,i}$ is therefore encoded into a much smaller \textit{preference matrix} $\Xi=(\xi_{\tu,\ti})\in\{-1,+1\}^{\qu\times \qi}$, which specifies the preference of each user type for each item type. The preference $L_{u,i}$ of user $u\in [N]$ for item  $i\in \mathbb{N}$ is the preference $\xi_{\tau_U(u),\tau_I(i)}$ of the associated user type $\tau_{U}(u)$ for the item type $\tau_{I}(i)$ in the matrix $\Xi$, \textit{i.e.},
$$L_{u,i}=\xi_{\tau_U(u),\tau_{I}(i)}\,.$$ 
We assume that the entries of $\Xi$ are i.i.d., $\xi_{\tu,\ti}=+1$ w.p. $1/2$ and $\xi_{\tu,\ti}=-1$ w.p. $1/2$. Generalizing our results to i.i.d. entries with bias $p$ is  straightforward. However, the independence assumption is quite strong and an important future research direction is to obtain results for more realistic preference matrices. 
We also consider a noisy model with $L_{u,i}=\xi_{\tau_U(u),\tau_{I}(i)} \cdot z_{u,i}$ where $z_{u,i}$ are i.i.d. random variables with $\Pr[z_{u,i} = +1]=1-\gamma$ and  $\Pr[z_{u,i} = -1]=\gamma$.

\subsection{Two regimes of interest}
Two specific parameter regimes play a central role in this paper, capturing settings with structure only in user space or only in item space. As described in Section~\ref{s:results}, each of user-user or item-item CF is almost optimal in the corresponding regime.

\begin{definition}[User structure only ($\qi = 2^{\qu}$)]\label{d:user}
The \emph{user structure model} refers to the case that there is no structure in the item space. To simplify matters, we assume that the preference matrix $\Xi\in\{-1,+1\}^{\qu \times 2^{\qu}}$ is deterministic and has columns consisting of all sequences in $\{-1,+1\}^{\qu}$. Essentially the same preference matrix would arise (with high probability) if $\qi$ is much larger than $2^{\qu}$ (when the entries are i.i.d. as specified above in Subsection~\ref{ss:pref}). 
\end{definition}

\begin{definition}
[Item structure only ($\qu = N$)]\label{d:item}
The \emph{item structure model} refers to the case that there is no structure in the user space. This happens when 
$\qu$ is much larger than $N$, since then most user types have no more than one user. For the purpose of proving near-optimality of item-item CF, it suffices to take $\qu = N$ (and we do so).
\end{definition}

\section{Main results}
\label{s:results}
We will analyze a version of each of user-user and item-item CF within the general setup described in Section~\ref{sec:model}. The resulting regret bounds appear in Theorems~\ref{t:user-user} and~\ref{th:Item-upper}. These theorems are complemented by information theoretic-lower bounds, Theorems~\ref{t:user-userL} and~\ref{t:item-itemL}, showing that no other algorithm can achieve much better regret  (up to multiplicative logarithmic factors) in the specific extreme parameter regimes with user-structure only and item-structure only. The simplified versions of these theorems appear in this section. Towards the end of this section we present simulation results supporting the theorems.

\subsection{User-user collaborative filtering}
User-user CF exploits structure in the user space: the basic idea is to recommend items to a user that are liked by similar users. We analyze an instance of user-user CF described in detail in Section~\ref{ss:userAlg}, obtaining the regret bound given in Theorem~\ref{t:user-user} below. Essentially, the algorithm clusters users according to type by recommending random items for an initial phase, and then uses this knowledge to efficiently explore the preferences of each user type (as opposed to each user individually). The subsequent savings is due to the fact that the cost of exploration can be shared amongst users of the same type. 

The random recommendations made during the initial phase incur regret with slope $1/2$, because a random recommendation is disliked with probability half. Afterward, the users are clustered according to type. Recommending an item to $\qu$ users, one from each type, gives us the preferences of all $N$ users for the item, and each such recommendation is disliked with probability $1/2$. This results in a slope of $\qu/2N$ for regret in the second phase of the algorithm.

\begin{reptheorem}{t:user-user}[Regret upper bound in user-user CF, simplified version]
Consider the recommendation system model described in Section~\ref{sec:model} with $N$ users, $\qu$ user types, and $\qi>126\log N$ item types. 
There exists numerical constants $c, C$ so that Algorithm \ref{alg:user} achieves regret
	\[\reg(T)\leq\begin{cases} \frac{T}{2}\,, \quad\quad &\text{if } T\leq c \log N\\
C\big(\log N +  \frac{\qu}{N} T \,\big)	\,, \quad\quad &\text{if } T>c \log N\,.\end{cases}\]
\end{reptheorem}
The cold-start time, the time until the slope of the regret drops below $\gamma$, is evidently $\Theta(\log N)$ for any $\gamma\in (\frac{C\qu}N,\frac12)$.
It follows from the next theorem that if there is no structure in the item space and the number of user types is $\qu=N^{\alpha}$ for fixed $0<\alpha<1$, then the user-user CF algorithm achieves both regret and cold start time that are optimal up to multiplicative constants.

\begin{reptheorem}{t:user-userL}
[Regret lower bound with user structure only, simplified version]
There exist a numerical constant $c$ such that in the user structure model (Defn~\ref{d:user}) with $\qu>(\log N)^{1.1}$ user types and $N>N_0$ users, any recommendation algorithm must incur regret 
\[\reg(T)\geq \begin{cases} 0.49T-4\,, \quad\quad &\text{if } T\leq c\log\qu\\
0.2 \frac{\qu}{N} T \,, \quad\quad &\text{if } T>c
\log\qu\,.\end{cases}
\]
\end{reptheorem}

The reasoning for the first part of the lower bound is as follows. If a user has been recommended fewer than $\log\qu$ items, then its similarity with respect to other users cannot be determined. This implies that any recommendation made to this user has uncertain outcome. The second part of the lower bound is obtained by showing that when an item is recommended for the first time to a user from a given user type the outcome of that recommendation is uncertain, and lower bounding the number of such recommendations. This is where we use the condition that each item is recommended at most once to each user.

The lower bound shows that the poor initial performance of user-user CF, as bad as simply recommending random items, is unavoidable in the setting with only user structure and that its duration depends on the number of user types. In \cite{bresler2015regret} it was shown that a version of item-item CF obtains much smaller cold start time than user-user CF in a model with item structure only. Our results on item-item CF, described next, corroborate this.

\subsection{Item-item collaborative filtering}
Item-item CF exploits structure in the item space: users are recommended items similar to those they have liked.
We analyze an instance of item-item CF in Section~\ref{ss:itemalg}, obtaining the regret bound given in Theorem~\ref{th:Item-upper} below. 
The algorithm creates several clusters of items, as well as a set of unclustered items. Similarity of two items is estimated by having random users rate both items. 
Users then explore a single item from each cluster and liked clusters are subsequently recommended.
The effort of clustering is shared amongst all the users, and the savings is due to liked explorations yielding an entire cluster of items to recommend.

Crucially, this version of item-item CF has the feature that only a \emph{subset} of the item space is explored  (\textit{i.e.}, only a subset of the item types are clustered, with the others cast aside)\footnote{In comparison, learning the entire preference matrix requires $\qi$ explorations per user, resulting in a cold start time of at least $\qi$ versus our $\qi/N$. The more fundamental savings in clustering only a portion of the space is in reducing the shared cost, and thereby again reducing the cold start time. A naive clustering of $M$ items from the entire space requires $\qi (\log \qi)M$ comparisons (each item is compared $\log \qi$ times to a representative from each of $\qi$ types), and for the clusters to be larger than constant size on average, $M$ must be significantly larger than $\qi$. The resulting cold start time is thus at least $\qi^2 (\log \qi)/N$, which is much larger than our cold start time on the order of $\qi/N$.}. 
To the best of our knowledge, the benefit of limiting the scope of item exploration has not been made explicit before; this only became evident to us in seeking to match the lower bound. 
The total number of items compared and the number of clusters are chosen depending on the system parameters to give the best regret bound.

%

\begin{reptheorem}{th:Item-upper}[Regret upper bound in item-item CF, simplified version]
	Consider the recommendation system model described in Section~\ref{sec:model} with $N>5$ users, $\qi>13 \log N$ item types, and $\qu>16\log (N\qi)$ user types.  There are numerical constants $C, c_1,c_2$ and $c_3$ such that Algorithm~\ref{alg:item-fixed} obtains regret per user at time $T$ upper bounded as 
\begin{align*}
&\hspace{-.1cm}\reg(T)
\leq
C 
\begin{cases} 
T, 
\quad \quad
&\text{ if } \,
T \leq \max\{c_1, c_2 \frac{\qi \log (N\qi)}{N}\}
\\
{\log T +\sqrt{\frac{\qi \log (N\qi)}{N}T }}, 
\quad\quad
&\text{ if } \,
\max\{c_1, c_2 \frac{\qi \log (N\qi)}{N}\}
< 
T
\leq 
c_3\frac{N\qi}{\log (N\qi)}
\\[8pt]
\frac{\log (N\qi)}{N}T,
\quad \quad\quad \quad
&\text{ if }\,
c_3\frac{N\qi}{\log (N\qi)}
\leq T \,.
 \end{cases} 
\end{align*}
\end{reptheorem}

If there is no structure in the user space and the number of item types is $\qi=N^{\beta}$ for fixed $\beta>0$, then the item-item CF algorithm is optimal up to a logarithmic factor.

\begin{reptheorem}{t:item-itemL}[Regret lower bound for item structure only, simplified version]
In the item-structure model (Defn.~\ref{d:item}) with $\qi>25\,(\log N)^5$  item types and $N>32$ users, there exist numerical constants $C, c_1, c_2, c_3,$ and $c_4$
such that  any recommendation algorithm must incur regret
%

\begin{align*}
\reg(T)>
C \,\,\begin{cases}
T, 
\quad\quad \quad&\text{if }  \,
T\leq c_1\frac{\sqrt{\qi}\log \qi}{N}
\\[8pt]
\frac{T}{\log \qi},
\quad\quad\quad &\text{if } \,
c_1\frac{\sqrt{\qi}\log \qi}{N}\leq T< c_2\frac{\qi(\log\qi)^2}{N}
\\[8pt]
\sqrt{\frac{T\qi}{N}},  
\quad\quad\quad &\text{if } \,
c_2\frac{\qi(\log\qi)^2}{N}\leq T<c_3\frac{N\qi}{(\log\qi)^2}
\\[8pt]
\frac{{\log\qi}}{N}T,  
\quad\quad\quad &\text{if } \,
c_3\frac{N\qi}{(\log\qi)^2}\leq T\,.
\end{cases}
\end{align*}
\end{reptheorem}

It follows that the cold start time $\textsf{coldstart}(\gamma)$ with $\gamma = C$ in the item-structure only regime is lower bounded as $\widetilde{\Omega}(\sqrt{\qi}/N)$, while the upper bound based on our proposed algorithm is $\widetilde{O}(\qi/N)$. Note that the cold start time with $\gamma=1/\log \qi$ is $\widetilde\Theta(\qi/N)$. The gap in the upper and lower bounds on cold start time is a consequence of the existence of the second regime in the lower bound given above, and appears to be an artifact of our proof.

The proof of the lower bound is based on two main observations. First, if an item has been recommended to fewer than $\log \qi$ users, then its similarity with respect to other items cannot be determined; this implies that recommending this item to any user has uncertain outcome. Second, when a user is recommended an item from a given item type for the first time, the outcome of that recommendation is uncertain since this reveals a new variable in the preference matrix. Lower bounding the number of such uncertain recommendations gives the lower bound for regret.

If there is structure in the item space it is possible to avoid the long cold-start time of algorithms using only user structure: even for a very short time horizon, they can guarantee nontrivial bounds on regret. In particular, the near-optimal algorithm proposed here suffers from a constant value of regret for an initial period. 
Note that as $N$ increases, the regret upper bound (given in Theorem~\ref{th:Item-upper}) in the initial phase (constant $c_1$) does not change, but  the length of the initial phase increases. Thus increasing $N$ makes it easier to make meaningful recommendations. The same phenomenon is true more generally: the upper bound on regret at any time $T$ is a decreasing function of~$N$.

\subsection{Numerical Simulations}
We simulated our versions of \textsc{User-User} and \textsc{Item-Item} Algorithms (As described in Sections~\ref{s:uu-upper} and~\ref{s:itemitem}). 
In Figure~\ref{fig:User}, we plot the regret  as a function of time for the \textsc{User-User} Algorithm (Alg.~\ref{alg:user} in Section~\ref{s:uu-upper}). We observe that the slope of regret in the asymptotic regime increases by increasing $\qu$ for fixed $N$. We also observe that increasing $N$ decreases the asymptotic slope but does not decrease the cold start time of the algorithm.

\begin{figure}[h]
\includegraphics[width=8cm]{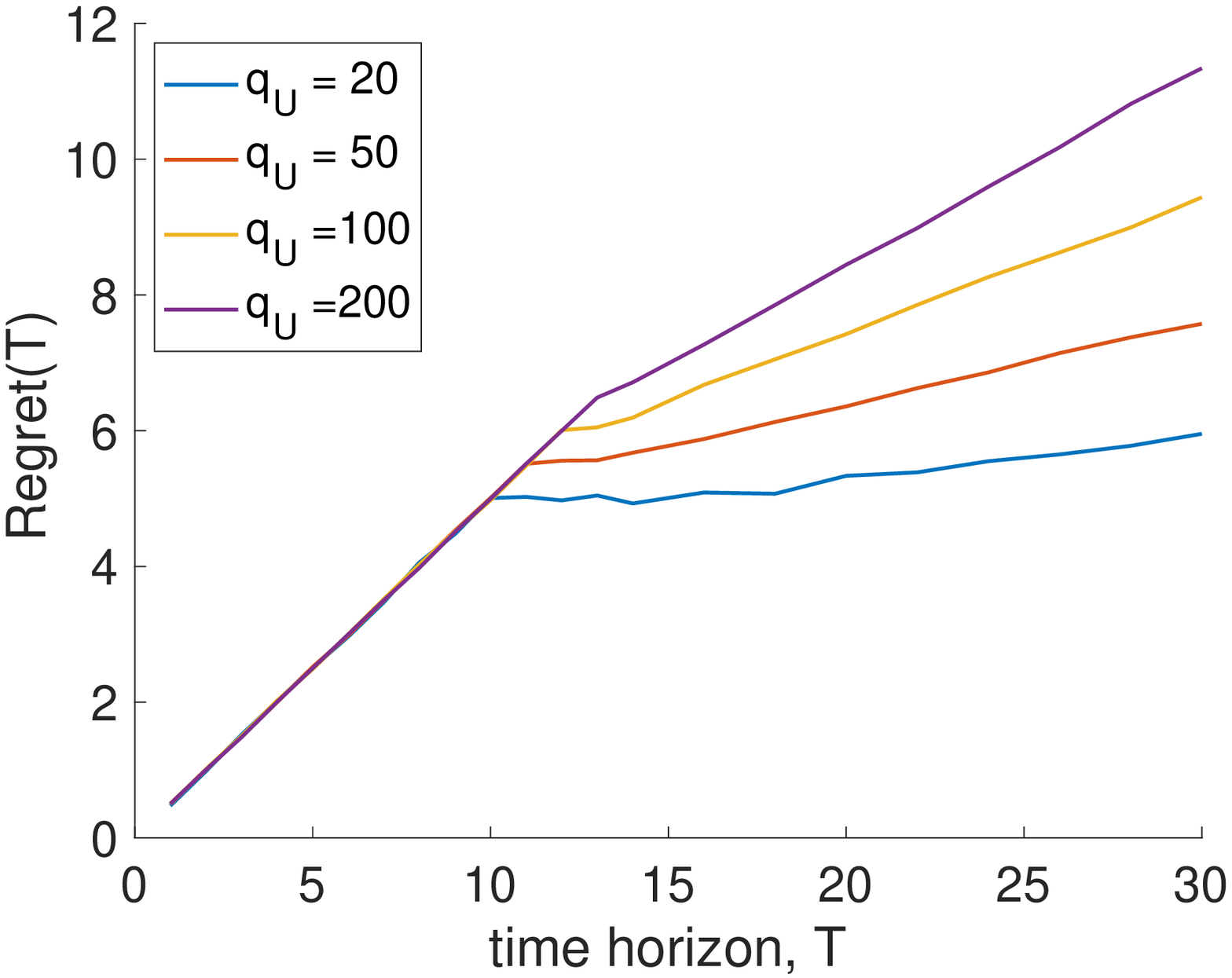} 
\includegraphics[width=8cm]{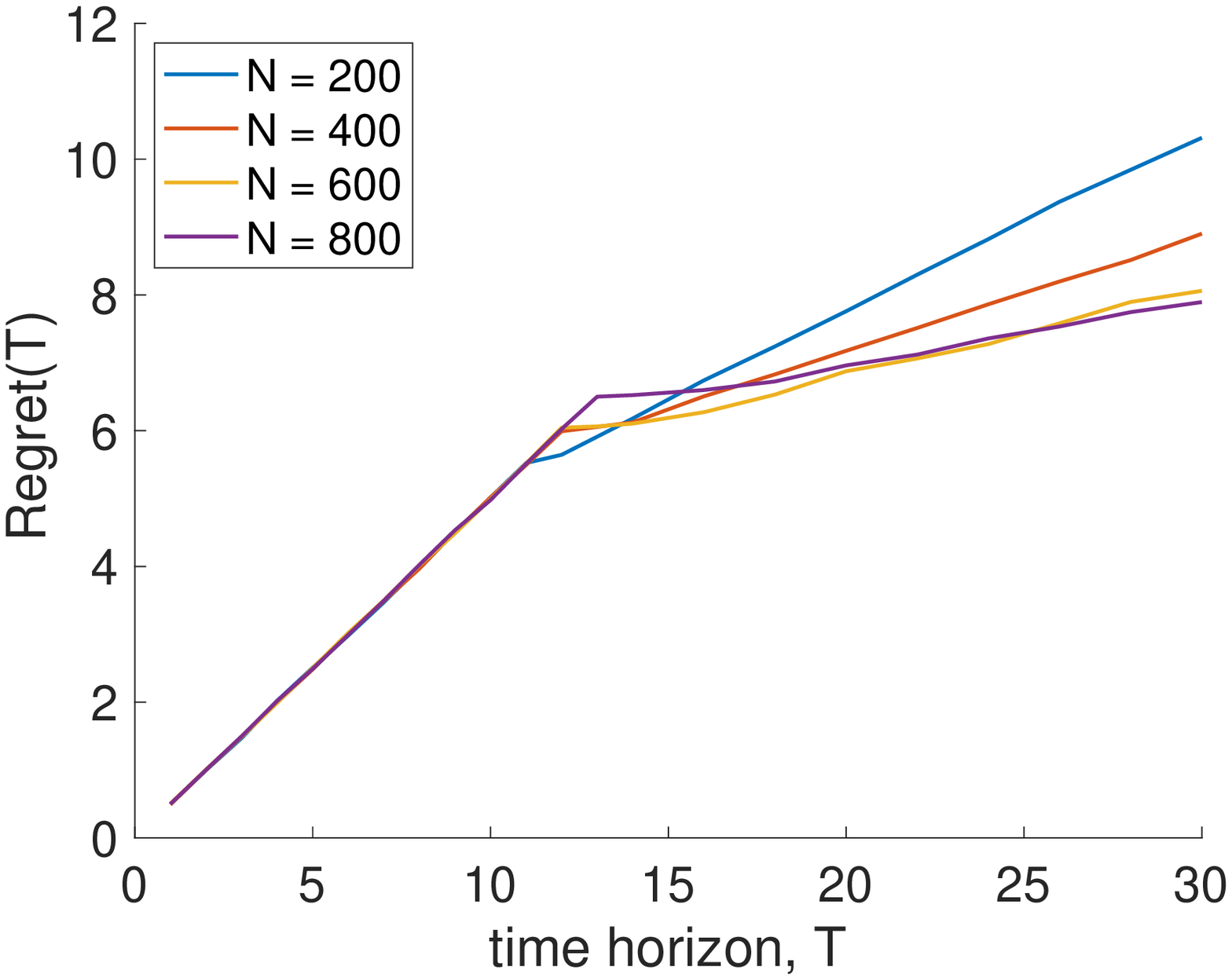} 
\caption{Simulated performance for Algorithm \textsc{User-User}. System parameters are (a) $N=400$ and $\qi = 100$ and (b) $\qu=80 $ and $\qi=100$.}
\label{fig:User}
\end{figure}

In Figure~\ref{fig:Item}, we plot the regret  as a function of time for the \textsc{Item-Item} Algorithm (Alg.~\ref{alg:item-fixed} in Section~\ref{s:itemitem}). We observe that with fixed $N$, increasing $\qi$ increases the cold-start time. But with fixed $\qi$, the cold-start time shrinks linearly in $N$.
We also observe that the slope of regret after the cold start time increases with increasing $\qi$ and decreasing $N$, consistent with the statement of Theorem~\ref{th:Item-upper}.

\begin{figure}[h]
\includegraphics[width=8cm]{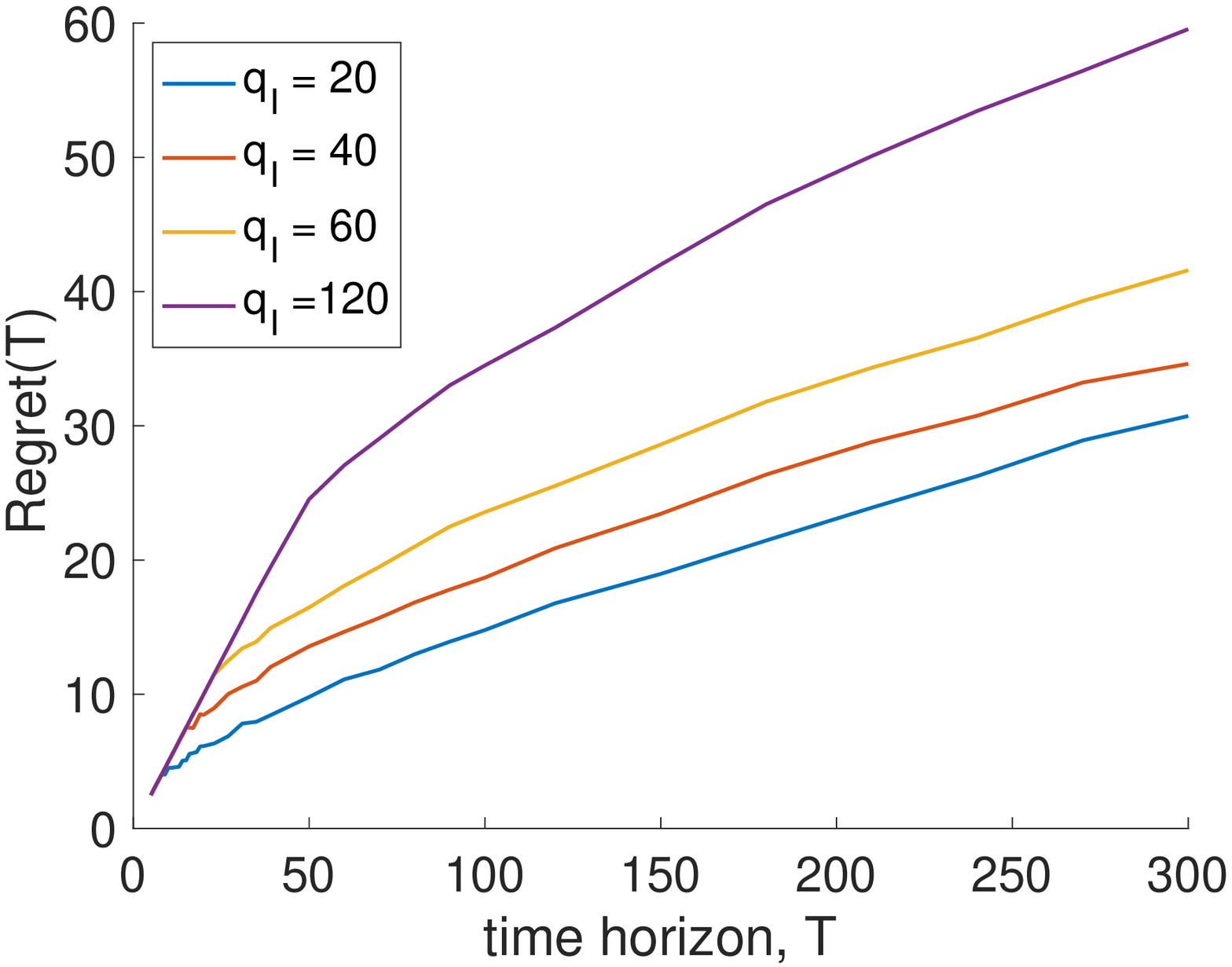}
\includegraphics[width=8cm]{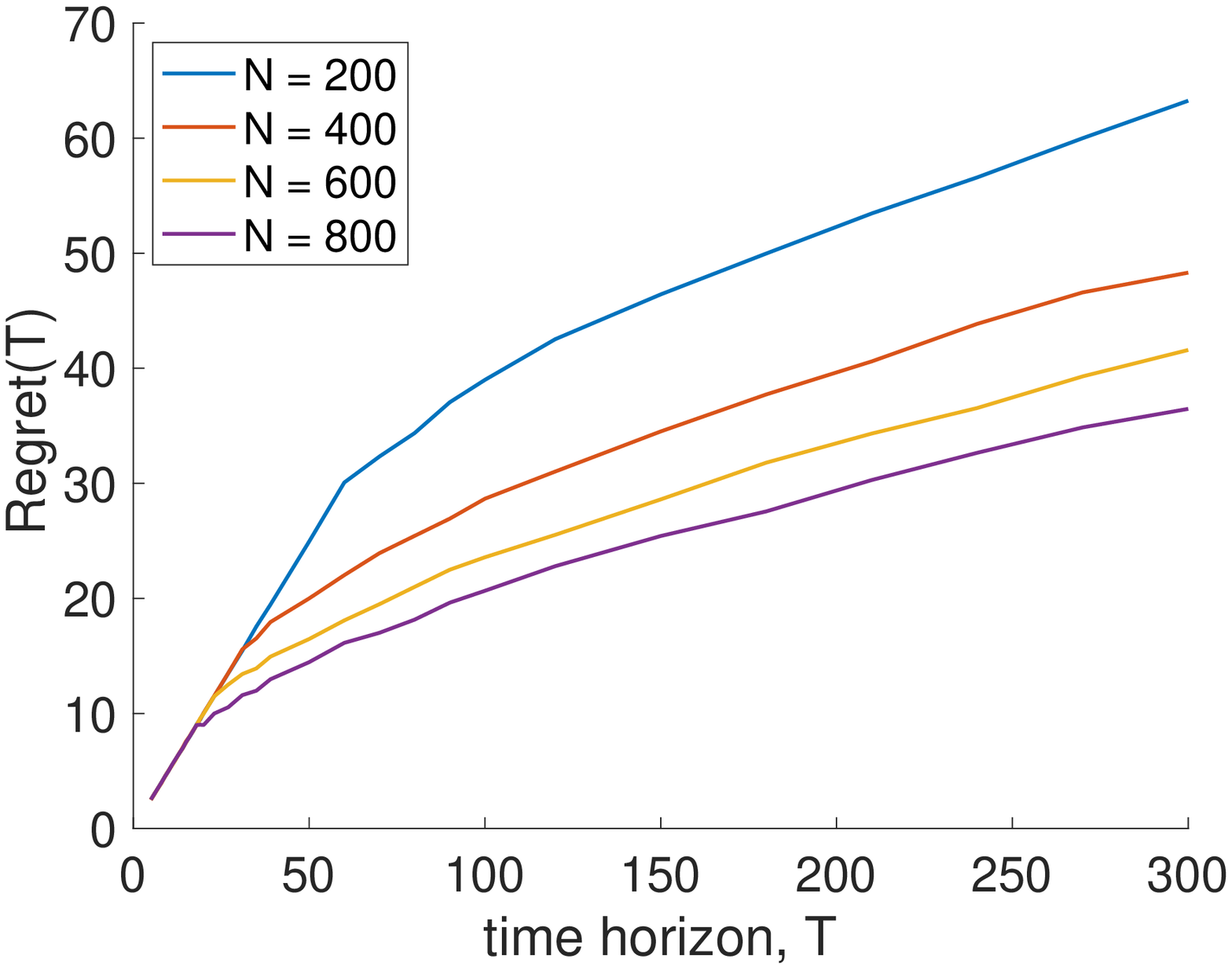}
\caption{Simulated performance for Algorithm \textsc{Item-Item}. System parameters are (a) $N=600$ and $\qu = 100$ and (b) $\qi = 60$ and $\qu = 100$.}
\label{fig:Item}
\end{figure}

\section{User-user algorithm and analysis}
\label{s:uu-upper}
In this section, we describe a version of user-user CF and then analyze it within the latent variable model introduced in Section~\ref{sec:model}.

\subsection{Algorithm}\label{ss:userAlg}
Pseudocode for algorithm \textsc{User-User} appears as Algorithm \ref{alg:user}. In Step 1, random items are recommended to all of the users. The ratings of these items are used to construct a partition $\{\Pc_k\}_k$ of users that recovers the user types correctly with high probability. In Step 2, users are recommended new random items (\emph{exploration}) until an item is liked. If the user is in group $\Pc_k$ of the partition, the item is added to a set $\mathcal{S}_k$ of items to be recommended to all other users in the same partition (\textit{exploitation}). Step 2 (find and recommend items) is repeated indefinitely. 

\begin{algorithm}[ht]
\caption{\textsc{User-User}($T, \qu, N$) }\label{alg:user}
\begin{algorithmic}[1]
\Statex \emph{Step 1: partition users}
\State $\epsilon \gets {1}/{N}$, $\thr\gets \lceil2\log({\qu^2}/{\epsilon})\rceil$
\For{$t = 1, \cdots, \thr$}
\State Pick random item $i$
\State $a_{u,t} \gets i,$ for all $u\in [N]$
\EndFor
\State Partition users into fewest possible groups such that each group agrees on all items.  Let $\tauh_U(u)\in [\qu]$ be the label of user $u$'s partition. \label{Line:UU-partition1}
\State $\Pc_k=\{u\in[N]:\tauh_U(u)=k\},$ for all $k\in[\qu]$ \label{Line:UU-partition2}
\Statex
\Statex \emph{Step 2: find and recommend items}
\State $\mathcal{S}_{\tu}\gets\varnothing,$ for all $\tu \in [\qu]$
\For{$t= \thr+1, \cdots, T$}
\For{$u\in [N]$}
\If{$\mathcal{S}_{\tauh_U(u)}\setminus\{a_{u,1},\dots,a_{u,t-1}\}\neq \varnothing $ (i.e., $u$ has not rated all items in $\mathcal{S}_{\tauh_U(u)}$) } 
\State $a_{u,t}\gets$ an unrated item in $\mathcal{S}_{\tauh_U(u)}$ (\textit{exploit})
\Else     
\State $a_{u,t}\gets$ random item not rated by any user (\textit{explore}) 
\If{$L_{u,a_{u,t}}=+1$} 
\State $ \mathcal{S}_{\tauh_U(u)} \gets \mathcal{S}_{\tauh_U(u)} \cup \{a_{u,t}\}$
\label{line:user-constructS}
\EndIf
\EndIf
\EndFor
\EndFor
\end{algorithmic}
\end{algorithm}
\begin{remark}
\label{rem:UU-error}
 Our model assumes that users of the same type have identical ratings. Hence, users of the same type are always in the same group after partitioning. However, due to random sampling of the items in exploration, users from different types can have identical ratings for the items recommended in Step~1, in which case they will end up in the same partition. 
It follows that the total number of groups in the user partition is at most $\qu$.
\end{remark}
We make a few additional remarks regarding the algorithm:
\begin{itemize}
\item The labeling of user groups in the partitioning step is arbitrary (and may be different from the similarly arbitrary labeling of user types).
\item In Step 2, the sets of items $\{\S_{\tu}\}$ at each time contain the items exploitable by users in the $\tu$-th group in the partition. The algorithm predicts that all users in the $\tu$-th group like items in $\S_{\tu}$.

\item The algorithm takes $T$, $\qu$, and $N$ as input. As mentioned in Section~\ref{sec:model}, a doubling trick described in Appendix~\ref{s:any-time-reg-alg} converts the algorithm to one oblivious to $T$. It is also fairly straightforward to modify the algorithm to be adaptive to $\qu$ The adaptive algorithm initializes with a trivial partition placing all users in one group. The algorithm subsequently refines the partition whenever a user's feedback indicates that they have been grouped incorrectly. We chose not to do so since it complicates the analysis.
\end{itemize}

\begin{theorem}\label{t:user-user}
Consider the model introduced in Section~\ref{sec:model} with $N$ users, $\qu$ user types and $\qi$ item types. Let $\thr=\lceil 2\log(N\qu^2)\rceil$. If $\qi>18\thr$, then \textsc{User-User} achieves regret
	\[\reg(T)\leq\begin{cases} \frac{1}{2}T\,, \quad &\text{if } T\leq \thr\\[8pt]
 \frac{1}{2} \thr + \frac{2\qu+2}{N} T + 2	\,, \quad &\text{if } T>\thr\,.\end{cases}\]
\end{theorem}

The assumption $\qi>18\thr$ ensures that with probability $1-o(\frac1N)$ for each user type, there is at least one item type that is liked. 
This assumption also  ensures that with probability $1-o(\frac1N)$, for any pair of user types, there is at least one item type which is rated differently by them. If there is no such item type, then  the two user types rate everything similarly and are indistinguishable.

The  theorem indicates that up until time $\thr$, the algorithm is making meaningless (randomly chosen independent of feedback) recommendations. Random recommendations have probability half of being liked, hence incur regret with slope $1/2$.  After that, the algorithm achieves the asymptotic slope indicating that on average $\qu$ recommendations out of $N$ are random. 
The simplified version of this theorem in Section \ref{s:results} is obtained using $2\log N<\thr <7\log N$ (since $\qu\leq N$). We also pick the constant $C$ large enough so that $\frac{\thr}{2} + \frac{2\qu+2}{N}T + 2\leq C(\log N + \frac{\qu}{N}T)$ for $T>\thr$.

\subsection{Proof of Theorem~\ref{t:user-user}}
We first bound the probability that the partition created by the algorithm is correct in Lemma~\ref{l:user-partition-prob}. Next, to prove the theorem we will show that conditioned on the partition being correct, the number of exploratory recommendations (and hence the regret) is upper bounded.

\begin{lemma} \label{l:user-partition-prob}
Let $B_{uv}= \{\ident_{\{\tauh_U(u)=\tauh_U(v)\}}= \ident_{\{\tau_U(u)=\tau_U(v)\}}\}$ be the event that users $u$ and $v$ are partitioned correctly with respect to each other in Step~1 of \textsc{User-User}. Let $\epsilon$ and $\thr$ be as defined there. If $\qi >4 \thr$, then $\Pr[B_{uv}^c]\leq \frac{2\epsilon}{\qu^2}$. It follows that if $B = \bigcap B_{uv}$ is the event that all users are partitioned correctly, then $\Pr[B]> 1-\epsilon$.
\end{lemma}

\begin{proof}
	As observed in Remark~\ref{rem:UU-error}, users from the same partition rate items identically. Therefore the only way an error in partitioning occurs is if users of different types are grouped together.  This happens when two users rate all exploratory items identically in Step~1. 
In Step~1, the first $\thr$ items recommended to all users are chosen uniformly at random independent of feedback, so the types of these items are uniformly distributed on $[\qi]$. Let $\mathrm{s}$ be the number of items with distinct item types among the $\thr$ exploratory items from Step~1. 
This is a balls and bins scenario with $\thr$ balls into $\qi$ bins, and Lemma~\ref{l:ballsbins} states that
 if $\qi> 4\thr$, then $\Pr[\mathrm{s}< \thr/2]\leq \exp(-\thr/2)\leq \epsilon/\qu^2$. By symmetry, each of the types of the $\mathrm{s}$ items with distinct types is uniformly distributed on $[\qi]$. 
 
Since all users rate the same items and users of the same type have identical preferences, as far as the lemma is concerned we only consider how the user types themselves rate items in Step~1.
Two user types $\tu\neq\tu'$ rate 
 $\mathrm{s}$ independently chosen
 items of distinct types in the same way with probability $2^{-\mathrm{s}}$. On the event
$\mathrm{s}\geq \thr/2$,  we have $2^{-\mathrm{s}}<2^{-\thr/2}\leq\epsilon/\qu^2$.  

The above two statements show that for users $u$ and $v$ with $\tau_U(u)=\tu$ and $\tau_U(v)=\tu'$,
 $$\Pr[(B_{uv})^c]
 \leq 
 \Pr\big[\mathrm{s}< \thr/2\big] + \Pr\big[(B_{uv})^c\,\big|\, \mathrm{s}\geq \thr/2\big]
 \leq  
 2\epsilon/\qu^2\,.$$
The second statement in the lemma follows by union bounding over ${\qu \choose 2}\leq \qu^2/2$ pairs of user types.
\end{proof}

\begin{proof}[Proof of Theorem~\ref{t:user-user}]
For $t\leq \thr$, the algorithm recommends random items chosen independently of feedback to all users. So at these times $\Pr[L_{u,a_{u,t}}=-1]=1/2$ for all users $u\in[N]$. It follows that, for $T\leq \thr$,
\begin{equation}\label{e:userPartitionRegret}
\Exp{\sum_{t=1}^T\frac{1}{N} \sum_{u=1}^N \ident[L_{u,a_{u,t}}=-1]}
   = \sum_{t=1}^T \frac{1}{N} \sum_{u=1}^N  \Pr[L_{u,a_{u,t}}=-1]=   \frac{T}{2}\,.
\end{equation}
Now consider the case $T>\thr$. At $t=\thr,$ by Lemma \ref{l:user-partition-prob}, the partitioning step recovers the user types correctly with probability at least $1-\epsilon$, \textit{i.e.,} $\Pr[B]> 1-\epsilon$.  On event $B$ all users in a partition have the same type, so by construction of the sets $\mathcal{S}_{\widehat{\tau}_U(u)}$ in Line~\ref{line:user-constructS} of \textsc{User-User}, items in $\mathcal{S}_{\widehat{\tau}_U(u)}$ are liked by at least one user of the same type as $u$ and therefore also by $u$, and
\begin{equation}\label{eq:UU-upper-noiseless}
\Exp{\sum_{t=r+1}^T\sum_{u\in[N]}  \ident\big[\,L_{u,a_{u,t}}=-1, a_{u,t} \in \mathcal{S}_{\widehat{\tau}_U(u)}\,\big]\Bigg| B}  = 0\,. 
\end{equation}
Because there are $TN$ terms in the sum and $\Pr[B^c]\leq \eps$, it follows that
\begin{equation}\label{e:uuEq1}
\Exp{\sum_{t=r+1}^T\sum_{u\in[N]}  \ident\big[\,L_{u,a_{u,t}}=-1, a_{u,t} \in \mathcal{S}_{\widehat{\tau}_U(u)}\,\big]} \leq T N\epsilon=T\,.
\end{equation}

Now, we need to find an upper bound for the expected number of disliked exploration recommendations in Step~2 of the algorithm,  $\Exp{\sum_{t=r+1}^T\sum_{u\in[N]}  \ident\big[\,L_{u,a_{u,t}}=-1, a_{u,t} \notin \mathcal{S}_{\widehat{\tau}_U(u)}\,\big]}$.

 It will be useful to relate the expected number of liked and disliked explorations. To this end, we consider the event that
every user type likes at least $1/3$ of the item types: define the event 
$$
C = \Big\{\sum_{\ti\in[\qi]} \ident[\xi_{\tu,\ti}=+1]> \qi/3\text{ for all }\tu\in[\qu]
\Big\}\,.$$
A Chernoff bound (Lemma~\ref{l:Chernoff}) applied to the i.i.d. $\xi$ variables gives $\Pr[C^c]\leq \qu\exp(-\qi/36)\leq 1/N$, where the last inequality due to $\qi>18\thr$. Conditioning on event $C$, we get
\begin{align}
\Ex\bigg[\sum_{t=r+1}^T\sum_{u\in[N]}& \ident\big[\,L_{u,a_{u,t}}=-1, a_{u,t} \notin \mathcal{S}_{\widehat{\tau}_U(u)}\,\big]\bigg] \notag\\
&\leq \Exp{\sum_{t=r+1}^T\sum_{u\in[N]}  \ident\big[\,L_{u,a_{u,t}}=-1, a_{u,t} \notin \mathcal{S}_{\widehat{\tau}_U(u)}, C\,\big]} + 
NT \Pr(C^c)
\,.\label{eq:ul-disliked-exploration}
\end{align}

To obtain an upper bound for the first term, in Claim~\ref{cl:ul-liked-explor} below we will upper bound the expected number of exploration recommendations that were \emph{liked}, and on event $C$ this will provide also an upper bound for the expected number of exploration recommendations that were disliked. The number of liked explorations is easier to deal with, because of a self-limiting effect: these result in items added to sets $\{\mathcal{S}_k\}$ for exploitation, and exploration only happens when there are not enough items to be exploited. 

We now relate the expected number of liked and disliked explorations. At $t>\thr$ if $a_{u,t} \notin \mathcal{S}_{\widehat{\tau}_U(u)}$, then it means the item is an exploratory recommendation and thus $a_{u,t}$ is an independent new random item with uniformly random type $\tau_I(a_{u,t}) \in [\qi]$. 
 Hence, using the definition of event $C$,  
\begin{equation} \label{e:selfBounding1}
 p:=\Pr\big[\,L_{u,a_{u,t}}=+1| a_{u,t} \notin \mathcal{S}_{\widehat{\tau}_U(u)}, C\,\big]\geq 1/3
\end{equation}
and $1-p\leq 2p$. It follows that
\begin{align} \label{e:selfBounding}
 \Pr\big[\,L_{u,a_{u,t}}=-1,a_{u,t} \notin \mathcal{S}_{\widehat{\tau}_U(u)}, C\,\big]\leq 2 \Pr\big[\,L_{u,a_{u,t}}=+1, a_{u,t} \notin \mathcal{S}_{\widehat{\tau}_U(u)}, C\,\big]. 
\end{align}
This means that to bound the first term in 
\eqref{eq:ul-disliked-exploration}
it suffices to bound the contribution from the sum with $L_{u,a_{u,t}} = +1$, as derived in the following claim. 

\begin{claim}\label{cl:ul-liked-explor} On event $C$, the number of liked `explore' recommendations (line~13 of Algorithm~\ref{alg:user}) by time $T$ can be bounded as
	\[\sum_{t=\thr+1}^T\sum_{u\in[N]} \ident[L_{u,a_{u,t}}=+1, a_{u,t} \notin \mathcal{S}_{\widehat{\tau}_U(u)}, C] \leq T\qu + N.\]
\end{claim}

\begin{proof}
		For user partition $\tu$ and time $t$, define $\S_{\tu}^t$ to be the set of items denoted by $\S_{\tu}$ in the algorithm at time $t$, after making the time-step $t$ recommendations.
Item $a_{u,t}$ is added to $\S_{\tu}$ precisely on the event 
$\{t>\thr, \tauh_U(u)=\tu, a_{u,t} \notin \mathcal{S}^{t-1}_{k}, L_{u,a_{u,t}}=+1\}$. 
Therefore, dropping $C$ from the indicator,
\begin{equation}
\label{e:userTemp1}
\sum_{t=\thr+1}^T\sum_{{u: \tauh_U(u)=\tu }} \ident[L_{u,a_{u,t}}=+1, a_{u,t} \notin \mathcal{S}^{t-1}_k,C]\leq\sum_{t=\thr+1}^T\sum_{{u: \tauh_U(u)=\tu }} \ident[L_{u,a_{u,t}}=+1, a_{u,t} \notin \mathcal{S}^{t-1}_k] = |\S_{\tu}^T|.
\end{equation}
	The rest of the proof entails bounding $|\S_{\tu}^T|$. The number of items added to $\S_{\tu}$ at time $t$ is 
	\begin{align}\label{e:userTemp3}
|\S_{\tu}^t|-|\S_{\tu}^{t-1}| &= \sum_{u:\tauh_U(u)=\tu } \ident[L_{u,a_{u,t}}=+1, a_{u,t} \notin \mathcal{S}^{t-1}_k] 
\leq \sum_{u:\tauh_U(u)=\tu } \ident[a_{u,t} \notin \mathcal{S}^{t-1}_k]\,.
\end{align}
	If $|\S_{\tu}^{t-1}|\geq t$, then 
	$\S_{\tu}^{t-1}\setminus \{a_{u,1},\cdots,\cdots,a_{u,t-1}\}\neq \varnothing$.
	 Meanwhile, at time $t$, the exploration event (recommending
	 $a_{u,t} \notin \mathcal{S}^{t-1}_{\widehat{\tau}_U(u)}$ 
	 in line~13) happens only if there are no items left in $\S^{t-1}_{\tau_U(u)}$ for user $u$ to exploit, i.e., 
	$\S_{\tau_U(u)}^{t-1}\setminus \{a_{u,1},\cdots,a_{u,t-1}\}=\varnothing$. In this way, $|\S_{\tu}^{t-1}|\geq t$ guarantees that there is an exploitable item at time $t$ for each user in $\Pc_k$. Consequently, 
	\begin{equation}\label{e:userTemp2}
	\sum_{\substack{u: \tauh_U(u)=\tu }} 
	\ident[a_{u,t} \notin \mathcal{S}^{t-1}_{\widehat{\tau}_U(u)}]\quad 
	\begin{cases} = 0, \quad &\text{if } |\S_{\tu}^{t-1}|\geq t
	 \\
	\leq 
	|\Pc_{\tu}|, \quad &\text{otherwise.}
	\end{cases}  
	\end{equation}
	The bound $|\Pc_{\tu}|$ is due to the sum having $|\Pc_{\tu}|$ terms, each upper bounded by 1.

	Let $t^* = \max\{t: \thr\leq t< T,\, |\calS_{k}^{t-1}|<t\}$ be the last time for which we are \emph{not} guaranteed (based on the reasoning before the last displayed eqn.) to have an exploitable item. Note that the set over which we take the maximum is nonempty if $T>\thr$ since $|\calS_{k}^{\thr}|=0$. 
It follows that 
	 \begin{align*}|\S_{\tu}^T| 
	 & = |\S_{\tu}^{t^* -1}| +\sum_{t=t^*-1}^{T-1} (|\S_{\tu}^{t+1}|-|\S_{\tu}^{t}| )
	 \stackrel{(a)}{\leq} |\S_{\tu}^{t^*-1}| + (|\S_{\tu}^{t^*}| - |\S_{\tu}^{t^*-1}| )
	 \stackrel{(b)}{\leq} T + |\Pc_{\tu}|\,.
	 \end{align*}
Since for $t^*<t<T$ we have $|\calS_{k}^{t-1}|\geq t$, by~\eqref{e:userTemp3} and~\eqref{e:userTemp2} for these times we have $|\S_{\tu}^{t+1}|-|\S_{\tu}^{t}| =0$ . This gives (a) in the above display.  
	 By definition, $|\S_{\tu}^{t^*-1}|<t^*< T$. Inequality (b) uses~\eqref{e:userTemp3} and~\eqref{e:userTemp2} to bound $|\S_{\tu}^{t^*}| - |\S_{\tu}^{t^*-1}|$.

	Note that  $\sum_{\tu\in[\qu]} |\Pc_{\tu}|= N$. Using~\eqref{e:userTemp1} and summing the last displayed inequality over the (at most) $\qu$ partition indices proves the claim. 
\end{proof}

We can now complete the proof of Theorem~\ref{t:user-user}.
By the preceding claim and Equations~\eqref{e:uuEq1},~\eqref{eq:ul-disliked-exploration}, and~\eqref{e:selfBounding} we get
\begin{align*}
\Exp{\sum_{t=r+1}^T\sum_{u\in[N]}  \ident[L_{u,a_{u,t}}=-1]}& 
 = \Exp{\sum_{t=r+1}^T\sum_{u\in[N]}  \ident[L_{u,a_{u,t}}=-1, a_{u,t} \in \mathcal{S}_{\widehat{\tau}_U(u)}]} \\
&\quad + \Exp{\sum_{t=r+1}^T\sum_{u\in[N]}  \ident[L_{u,a_{u,t}}=-1, a_{u,t} \notin \mathcal{S}_{\widehat{\tau}_U(u)}]} \\
&\leq 2(\qu T+ N) + 2T\,.
\end{align*}
For $T>\thr\,,$
we can now bound the regret by combining Equation~\eqref{e:userPartitionRegret} with the previous display:
\begin{align*}
\reg(T) 
 & = \Exp{\sum_{t=1}^r \frac{1}{N} \sum_{u=1}^N \ident[L_{u,t}=-1]}
  +  \Exp{\sum_{t=r+1}^T \frac{1}{N} \sum_{u=1}^N \ident[L_{u,a_{u,t}}=-1]}\\
  &\leq  \frac{1}{2} \thr + \frac{2(\qu+1)}{N} T + 2. & \hfill\qedhere
\end{align*}
\end{proof}

\subsection{User-User Algorithm with Noisy Preferences}
We generalize the result to the scenario in which the feedback to the recommendation system is noisy. In this case, the preference of user $u$ for item $i$ is
\begin{equation}\label{eq:noisemodel}
L_{u,i}=\xi_{\tau_U(u),\tau_I(i)}\cdot z_{u,i}\,,
\end{equation}
where $z_{u,i}$ are i.i.d. random variables with $\Pr[z_{u,i}=+1]=1-\gamma$ and  $\Pr[z_{u,i}=-1]=\gamma$  (we assume $0<\gamma<1/2$). With probability $\gamma$, the preference of user $u$ for item $i$ is flipped relative to the preference of user type $\tau_U(u)$ for item type $\tau_I(i)$ in the preference matrix $\Xi$. 

To accommodate the noisy feedback, we modify the partitioning subroutine in Step~1 of \textsc{User-User} algorithm with \textsc{NoisyUserPartition} given in Algorithm~\ref{alg:partitioning-noisy-user}.
The main modification is that in Lines~\ref{Line:UU-partition1} and~\ref{Line:UU-partition2} of \textsc{User-User} algorithm, users are placed in the same partition if they rate all of the first $\thr$  items similarly. Instead, users are now placed in the same partition if they rate the majority of the first $\thr$ items similarly. The parameters $\lambda$  and $\thr$ are chosen to guarantee that the partitioning over users is consistent with their type with probability greater than $1-\epsilon$.

\begin{algorithm}[ht]
\caption{\textsc{NoisyUserPartition}($T, \qu, N, \gamma$) }\label{alg:partitioning-noisy-user}
\begin{algorithmic}[1]
\State $\epsilon \gets {1}/{N}$, $\lambda \gets \frac{2}{3}(1-2\gamma)^2$, $\thr\gets \lceil \frac{4}{(1-2\gamma)^2} \log(N^2/2\eps)\rceil$, 
\For{$t = 1, \cdots, \thr$}
\State Pick random item $i$
\State $a_{u,t} \gets i,$ for all $u\in [N]$
\EndFor
\For{$u = 1,\cdots, N $}
\For{$v= u+1,\cdots,N$}
\State $g_{u,v}\gets\ident{ \{\sum_{s=1}^{\thr} {L_{u a_{us}} L_{va_{vs}}} \geq \lambda\thr \} }$
\EndFor
\EndFor
\If{There exists a partitioning over users consistent with the variables $g_{u,v}$}\label{Line:UU-Noisy-consistency}
\State Let $\tauh_U(u)\in [\qu]$ be the label of user $u$'s partition. 
\Else
\State
$\tauh_U(u)=1, $ for all $u\in[N]$ (all users are in one partition)
\EndIf
\State $\Pc_k=\{u\in[N]:\tauh_U(u)=k\},$ for all $k\in[\qu]$
\Statex
\end{algorithmic}
\end{algorithm}
\begin{remark}
In Line~\ref{Line:UU-Noisy-consistency}, the algorithm checks whether there is a  partitioning over the users consistent with variables $g_{u,v}$. This is true precisely when the graph with edge set $g_{u,v}$ is a disjoint union of cliques. 
\end{remark}
\begin{remark}
The noisy feedback decreases performance in two ways: partitioning users correctly requires more exploration recommendations, resulting in a larger cold-start time. Additionally, in Step~2, even good exploitation recommendations can be disliked due to noise. The next theorem quantifies these observations.
\end{remark}	
\begin{theorem}\label{t:noisy-user-user}
Consider the model introduced in Section~\ref{sec:model} with $N$ users, $\qu$ user types and $\qi$ item types. Let $\thr=\big\lceil \frac{12}{(1-2\gamma)^2}\log N\big\rceil$. If $\qi>432 \log N$, then \textsc{User-User} achieves regret
	\[\reg(T)\leq\begin{cases} \frac{1}{2}T\,, \quad &\text{if } T\leq \thr\\[8pt]
 \frac{1}{2} \thr + \big[\frac{5\qu+2}{N}+\gamma\big] T + 5	\,, \quad &\text{if } T>\thr\,.\end{cases}\]
\end{theorem}
\begin{proof}
The proof of this theorem is very similar to the proof of Theorem~\ref{t:user-user}.  Lemma~\ref{l:user-partition-prob-noisy} replaces Lemma~\ref{l:user-partition-prob} to show that with the given choice of parameters in Algorithm~\ref{alg:partitioning-noisy-user}, the partitioning $\Pc_k$ is the same as partitioning over the users by their types with probability greater than $1-1/N$. Additionally,
Equation~\eqref{eq:UU-upper-noiseless} in the proof of Theorem~\ref{t:user-user} changes as follows to be consistent as a result  of noisy feedback modeled in~\eqref{eq:noisemodel}:
\[\Exp{\sum_{t=r+1}^T\sum_{u\in[N]}  \ident\big[\,L_{u,a_{u,t}}=-1, a_{u,t} \in \mathcal{S}_{\widehat{\tau}_U(u)}\,\big]\Bigg| B}  \leq TN\gamma\,. \]
Equation~\eqref{e:selfBounding1} is replaced with
\begin{equation*} 
 p:=\Pr\big[\,L_{u,a_{u,t}}=+1| a_{u,t} \notin \mathcal{S}_{\widehat{\tau}_U(u)}, C\,\big]\geq \frac{1-\gamma}{3}
\end{equation*}
and since $\gamma<1/2$, then $1-p\leq 5p$. It follows that
\begin{align} 
 \Pr\big[\,L_{u,a_{u,t}}=-1,a_{u,t} \notin \mathcal{S}_{\widehat{\tau}_U(u)}, C\,\big]\leq 5 \Pr\big[\,L_{u,a_{u,t}}=+1, a_{u,t} \notin \mathcal{S}_{\widehat{\tau}_U(u)}, C\,\big]. 
\end{align}
Claim~\ref{cl:ul-liked-explor}  bounds the right-hand side. 
Plugging in these equations in the subsequent part of the proof of Theorem~\ref{t:user-user} gives the statement of the theorem.
\end{proof}

\begin{lemma} \label{l:user-partition-prob-noisy}
Consider the user similarities computed in Step~7 of \textsc{NoisyUserPartition} (Algorithm~\ref{alg:partitioning-noisy-user}). 
Define the event 
$B_{uv}=
 \{
 g_{u,v}=\ident_{\{\tau_U(u)=\tau_U(v)\}}
\}$ that these similarities coincide with the underlying user types.
 If $\qi >144 \log(N^2/\epsilon)$, then $\Pr[B_{uv}^c]\leq 2\epsilon/N^2$. It follows that if $B = \bigcap B_{uv}$ is the event that all users are partitioned correctly, then $\Pr[B^c] >1- \epsilon$.
\end{lemma}
The proof of this lemma is similar to the proof of Lemma~\ref{l:user-partition-prob} and is deferred to Appendix~\ref{sec:user-partition-proof-noisy}.

\section{User structure only: lower bound}
\label{sec:UL}
In this section we prove a lower bound on the regret of any online recommendation system in the regime with user structure only where $\qi = 2^{\qu}$ as described in Definition~\ref{d:user}.  
\begin{theorem}\label{t:user-userL}
Let $\delta>0$ and $\thr=\lfloor\log{\qu} - \log\big({16\, (\log\qu)\, \log\frac{N}{\delta}}\big) \rfloor$. In the user structure model with $N$ users and $\qu$ user types, any recommendation algorithm must incur regret 
	\[\reg(T)\geq \begin{cases} \big(\frac{1}{2}-\delta\big) T - 4\,, \quad &\text{ if } T\leq \thr\\[10pt]
	\big[{1-\exp(-N/\qu)}\big]\frac{\qu}{2N}T\,,
	\quad &\text{ for all }\,  T\,.
	\end{cases}\]

\end{theorem}

\begin{remark}
The lower bound depends on a parameter $\delta$ that has two effects: (1) the slope of the regret curve during the cold start grows (approaching $1/2$) as the chosen parameter $\delta$ shrinks to $0$; (2) the cold start time $\thr$ is upper bounded as $\thr<\log\qu - \log\log(1/\delta)$. 

Additionally, if the number of user types satisfies $\qu/ N\to 0$, the slope of regret after the cold start time (the asymptotic rate of regret) approaches $\frac{\qu}{2N}$. This is expected since each item can be recommended at most once to each user. Hence, even if the structure in the user space is known, the algorithm should explore new items. On average, $\qu$ explorations (about half of which are disliked) are necessary for every $N$ recommendations.

The simplified version of this theorem in Section \ref{s:results} is obtained using $N\geq\qu\geq (\log N)^{1.1}$, $\delta=1/100$ and $N>N_0$ for a constant $N_0$.
\end{remark}

\subsection{Proof strategy}
At a high level, the lower bound is based on two observations:
\begin{itemize}
\item \textit{A good estimate of user types is necessary to make meaningful recommendations. Notably, estimating similarity between users requires approximately $\log \qu$ items rated in common.} 

Suppose that the preference matrix $\Xi$ (with elements $\xi_{k,j}$, the preference of user types for item types) is known (which is the case in user-structure only model). Also, suppose that we have obtained feedback from some user $u$ for $t$ items. Relative to the total number of types, user $u$ must belong to a restricted set of user types consistent with this feedback. If $t$ is small, the set of consistent types is large (for instance, if a user has rated only one item, there are roughly ${\qu}/{2}$ candidate user types for this user). At this point, user $u$ likes some item $i$ with probability proportional to the number of consistent types liking the item. Control of this count amounts to a property of the matrix we call $(t,\epsilon)$-column regularity in Definition~\ref{def:ul-reg}, which holds with high probability.

\item \textit{Even if we know the user types (i.e., clustering of users), the first time a given item is recommended to a user from a given type, the outcome is uniformly random.}

Since there is no structure in item space ($\qi=2^{\qu}$), learning the preference of a user type for an item is only achieved by recommending the item to one user from the user type. This is for the reason that the random variable $\xi_{\tau_U(u),\tau_I(i)}$ in the preference matrix is independent of all previous history in the situation described. 
\end{itemize}

\subsection{Proof of Theorem~\ref{t:user-userL}}
We separately prove the two lower bounds in the statement of the theorem, starting with the first. 
The following regularity property in submatrices of the preference matrix allows us to control the posterior probability for an item being liked.
\begin{definition}[$(r,\epsilon)$-column regularity]\label{def:ul-reg}
	Let $A\in \{-1,+1\}^{m \times n}$. 
	For ordered tuple of distinct (column) indices $w=(i_1,\dots,i_r)\in [n]^r$, let $M=(A_{\cdot i})_{i\in w}\in  \{-1,+1\}^{m \times r}$ be the matrix formed from the columns of $A$ indexed by $w$.
	For given row vector $b\in\{-1,+1\}^r$, let $K_{b,w}(A)\subseteq [m]$ be the set of rows in $M$ that are identical to $b$ and denote its cardinality by $k_{b,w}(A)$.
The matrix	$A$ is said to be \emph{$(r,\epsilon)$-column regular} if 
	\[\max_{b,w}\Big|k_{b,w}(A)-\frac{m}{2^{r}}\Big| \leq \epsilon \frac{m}{2^{r}}\,,\]
	where the maximum is over tuples $w$ of $r$ columns and $\pm 1$ vectors $b$ of size $r$.
	
	We define $\rgl_{r,\epsilon}$ to be the set of $(r,\epsilon)$-column regular matrices. 
\end{definition}

\begin{claim}\label{cl:ul-smalltregular}If  matrix $A\in\{-1,+1\}^{m\times n}$ is \emph{$(r,\epsilon)$-column regular}, then it is also \emph{$(s,\epsilon)$-column regular} for all $s<r.$\end{claim}
\begin{proof}
Suppose that $A$ is $(r,\epsilon)$-column regular. By induction it suffices to show that $A$ is $(r-1,\eps)$-column regular. We will check that 
$(1-\epsilon)\frac{m}{2^{r-1}}\leq k_{b,w}(A) \leq (1+\epsilon)\frac{m}{2^{r-1}}$ for all size $r-1$ tuples $w$ and vectors $b$. 
For any given $w\in[n]^{r-1}$ and $b\in\{-1,+1\}^{r-1}$, let  $b^+ = [b \, \, 1]\in \{-1,+1\}^{r}$ be obtained from $b$ by appending $+1$. Similarly $b^-$ is obtained from $b$ by appending $-1$. If $w'=(w,i)\in[n]^r$ for any $i\notin w$, then $K_{b,w}=K_{b^+,w'}\cup K_{b^-,w'}$ and $K_{b^+,w'}\cap K_{b^-,w'}=\varnothing$, so
 $k_{b,w}=k_{b^+,w'}+k_{b^-,w}$. Since $A$ is $(r,\epsilon)$-column regular, $(1-\epsilon)\frac{m}{2^r}\leq k_{b^+,w'},k_{b^-,w'} \leq (1+\epsilon)\frac{m}{2^r}$, hence $(1-\epsilon)\frac{m}{2^{r-1}}\leq k_{b,w} \leq (1+\epsilon)\frac{m}{2^{r-1}}$. 
\end{proof}

\begin{lemma}\label{l:ul-regularity}
	Let matrix $A\in\{-1,+1\}^{m \times n}$ have i.i.d. $\mathrm{Bern}(1/2)$ entries. 
	If $\epsilon<1$, then $A$ is $(r,\epsilon)$-column regular (\textit{i.e.}, $A\in\Omega_{r,\epsilon}$) with probability at least $$1-2(2n)^r \exp\left( - \frac{\epsilon^2}{3}\frac{m}{2^{r}}\right)\,.$$
\end{lemma}
\begin{proof}
For given column tuple $w$ and row vector $b$, the expected number of times the row vector  $b$ appears is $\frac{m}{2^r}$. A Chernoff bound (Lemma \ref{l:Chernoff}) with $\epsilon<1$ gives
\[\Pr\Big[\big|k_{b,w}(A)-\frac{m}{2^{r}}\big| \geq \epsilon\frac{m}{2^{r}}\Big] \leq 2 \exp\left( - \frac{\epsilon^2}{3}\frac{m}{2^{r}}\right)\,.\]
There are no more than $n^r$ possible choices of column tuple $w$, and $2^r$ possible choices of row vector $b$; 
the union bound yields the proof.
\end{proof}

\begin{prop}
Consider the user structure only model in Definition~\ref{d:user}.
Let $\delta>0$ and $\thr=\lfloor\log{\qu} - \log\big({16\, (\log\qu)\, \log\frac{N}{\delta}}\big) \rfloor$.
For any $T\leq \thr$, the regret is lower bounded by \[\reg(T)\geq (\tfrac{1}{2}-\delta)T - 4 \,. \]
\end{prop}
\begin{proof}
We will show that for preference matrices $\Xi$ satisfying column regularity, at any time $t\leq \thr$, most users have probability roughly half of liking any particular item given the feedback obtained thus far, even if the preference matrix is known. (Recall that the preference matrix contains the preference of each user type for each item type; there is still uncertainty in the actual type of each user or item). 
	
At time $t$, suppose that $n$ items in total have been recommended by the algorithm  ($n\leq Nt$ since each of the $N$ users rates one item per time-step). We label the set of items by $[n]=\{1,\dots,n\}$. Let $A$ be the $\qu \times n$ matrix indicating the preference of each user type for these $n$ items. 
Each item $i$ has type $\tau_I(i)\sim\mathrm{Unif}([2^{\qu}])$ and because the set of columns of the preference matrix $\Xi$ is precisely $\{-1,+1\}^{\qu}$, the columns of $A$ 
are independent and uniformly distributed in $\{-1,+1\}^{\qu}$ according to this model.   

We now focus on a particular user $u$. 
Let $w=\{a_{u,s}\}_{s\in[t-1]}$ be the items 
recommended to user $u$ up to time $t-1$, and let
 $b=(L_{u,a_{u,s}})_{s\in [t-1]}\in \{-1,+1\}^{t-1}$ be the vector of feedback for these items. 
We claim that conditional on the matrix $A$, vectors $b$ and $w$, the type $\tau_U(u)$ of user $u$ at the end of time instant $t-1$ is uniformly distributed over the set of user types $K_{b,w}(A)$ consistent with this data ($K_{b,w}(A)$ is defined in Definition~\ref{def:ul-reg}). 

Let $b^+ = [b \, \, 1]\in \{-1,+1\}^{t}$ be obtained from $b$ by appending $+1$.
Then $L_{u,a_{u,t}}= 
+1$ precisely when $\tau_U(u)\in K_{b^+, \{w,a_{u,t}\}}(A)$, which in words reads ``user $u$ is among those types that are consistent with the first $t-1$ ratings of $u$ and have preference vector with '$+1$' for the item recommended to $u$ at time $t$''. 
It follows that for any matrix $A$ corresponding to items $[n]$, 
\begin{align*}
\Pr[L_{u,a_{u,t}}= 
+1|\H_{t-1},A] &= \Pr[\tau_U(u)\in  K_{b^+, \{w,a_{u,t}\}}(A)|\H_{t-1},A] 
= \frac {k_{b^+, \{w,a_{u,t}\}}(A)} {k_{b, w}(A)}\,.
\end{align*}
The second equality is due to: i) $w$, $b$, and $a_{u,t}$ are functions of $\H_{t-1}$; ii) for fixed $w$ and $b$ the set $K_{b,w}(A)$ is determined by $A$; iii) $\tau_U(u)$ is uniformly distributed on $K_{b,w}(A)$ conditional on $A,b,w$.

Recall that $\rgl_{t,\epsilon}$ was defined as the set of $(t,\epsilon)$- column regular matrices.
It now follows by the tower property of conditional expectation that
\begin{align*}
\Pr[L_{u,a_{u, t}}= +1| A\in\rgl_{t,\epsilon}]
& = \Ex\big[\Pr[L_{u,a_{u,t}}= +1|A,\H_{t-1}]| A\in\rgl_{t,\epsilon}\big] \\
&= \Ex\bigg[\frac {k_{b^+, \{w,i\}}(A)} {k_{b, w}(A)}\bigg| A\in\rgl_{t,\epsilon}\bigg]
= \frac{1}{2}\frac{1+\epsilon}{1-\epsilon}\leq \frac{1}{2}(1+4\epsilon)\,.
\end{align*}
The last two equalities are justified as follows: if $A\in \rgl_{t,\eps}$ then by Claim~\ref{cl:ul-smalltregular}, $A\in \rgl_{t-1,\eps}$. 
By Definition~\ref{def:ul-reg}, this means that $k_{b, w}(A)\geq (1-\epsilon) m/2^{t-1} $ and  $k_{b^+, \{w,i\}}(A)\leq (1+\epsilon) m/2^t\,.$ We pick $\epsilon<1/2$ to get the last inequality.

Fix $\delta >0$ and define
$$
\epsilon_t=\sqrt{3\frac{2^t}{\qu}\log\frac{(2N\log \qu)^t}{\delta}}\,.
$$
 Lemma \ref{l:ul-regularity} shows that at time $t\leq \thr<\log\qu$, for this choice of $\eps_t$, we have $A\in\rgl_{t,\epsilon_t}$ with probability $1-\delta$. 
We get the bound
\begin{align*}
\Pr[L_{u,a_{u, t}}= 
+1] \leq \Pr[L_{u,a_{u, t}}= +1, A\in\rgl_{t,\epsilon_t}] + \Pr[A\notin\rgl_{t,\epsilon_t}]
\leq \frac{1}{2}(1+4\epsilon_t)+\delta\,.
\end{align*}
It follows from the above display and the definition of $\epsilon_t$ that for $T\leq \thr,$ 
\begin{align*}
\reg(T) & = \frac{1}{N} \sum_{t\in [T], u\in [N]} \Pr[L_{u,a_{u, t}} = 
-1] 
\,\,\overset{(a)}{\geq} \,\,
(\tfrac{1}{2}-\delta)T - \sum_{t\in[T]}2 \sqrt{3\frac{2^t}{\qu}\log\frac{(2N\log \qu)^t}{\delta}}  
\\ &
\overset{(b)}{\geq}
(\tfrac{1}{2}-\delta)T - 4 \sqrt{3\frac{2^{T+1}}{\qu}\log\frac{(2N\log \qu)^{\log\qu}}{\delta}}
\,\,\overset{(c)}{\geq} \,\,
  (\tfrac{1}{2}-\delta)T - 4 \,, &  
\end{align*}
where (a) uses the definition of $\epsilon_t$. (b) uses the summation of a Geometric series and $t\leq T<\thr<\log\qu$. (c) uses the definition of $\thr$ and $T\leq \thr< \log \qu - \log\big(12 \log \qu\big) - \log\log\frac{2N\log\qu}{\delta}$.
\end{proof}
We now proceed to the proof of the second lower bound in Theorem~\ref{t:user-userL}.
\begin{prop}\label{p:userLinear}
For any $T$, $\reg(T) \geq \big[{1-\exp(-N/\qu)}\big]\frac{\qu}{2N}T$.
\end{prop}
The main ingredient in the proof of the proposition is definition of an event that implies that the outcome of the associated recommendation is uniformly random.
Let $\BBul{u}{i}{t}$ be the event that some user of same type $\tau_U(u)$ as $u$ has rated item $i$ by time $t-1$: 
$$\BBul{u}{i}{t}=\{\exists v\in[N]\setminus\{u\}, s\in[t-1]: \tau_U(v)=\tau_U(u), a_{v,s}=i\}.$$
Note that $\BBul{u}{i}{t}$ is a function of $\H_{t-1}$ and the set of user types $(\tau_U(u))_{u\in [N]}$.
\begin{claim}\label{c:Brandom}
If no user with the same type as $u$ has rated item $i$ by time $t-1$, the probability that user $u$ likes item $i$ conditional on any history consistent with this is  
$
\Pr[L_{u,i}= -1 |(\BBul{u}{i}{t})^c,\H_{t-1}]= \frac{1}{2}
$.
\end{claim}
\begin{proof}
According to Definition~\ref{d:user}, in the user-structure only model, the matrix $\Xi$ is deterministic and has columns consisting of all sequences in $\{-1,+1\}^{\qu}$.  We will show that $\Pr[L_{u,i}= -1 |(\BBul{u}{i}{t})^c,\H_{t-1},\tau_U(\cdot)]= \frac{1}{2}$ where  $\tau_U(\cdot)$ is the sequence of all user types.

A priori $\tau_I(i)$ is uniform on $[q_I]$. Given  the sequence $\tau_U(\cdot)$, the matrix $\Xi$ and the feedback $\H_{t-1}$ up to time $t-1$, the posterior distribution of $\tau_I(i)$ is uniform over the set of all item types $j$ which are consistent with the outcome of recommending $i$ to users of various types. We call this set $\mathcal{S}_t(i)$. 

Since on the event $(\BBul{u}{i}{t})^c$, no user with the same type as $u$ has rated $i$ by time $t-1$,  and since the matrix $\Xi$ has columns consisting of all sequences in $\{-1,+1\}^{\qu}$, half of the item types in set $\mathcal{S}_t(i)$ are liked by user $u$ and half of them are disliked. 
%
%
%
\end{proof}

The final ingredient in the proof of Proposition~\ref{p:userLinear} is a lower bound on the number of items recommended for which Claim~\ref{c:Brandom} applies.
\begin{claim}
\label{cl:ul-slope-cl3} 
The expected number of times a new item is recommended to a user type by time $T$ is lower bounded as
$$
\Ex\bigg[{\sum_{\substack{t\in[T], u\in[N]}} \ident[(\BBul{u}{a_{u,t}}{t})^c]}\bigg]
\geq 
\qu T\,\Big[1-e^{(-N/\qu)}\Big] \,.$$
\end{claim}
\begin{proof}
At the end of time-step $T$ each user has been recommended $T$ items, hence each \textit{user type} has been recommended at least $T$ items. 
Let $\widetilde{\qu}$ be the number of user types in which there is at least one user.
The total number of times an item is recommended to a user type for the first time is at least $\widetilde{\qu} T$.
Applying Lemma~\ref{l:ballsbins2} shows that $\Exp{\widetilde{\qu}}\geq \qu(1-(1-1/\qu)^N)\geq \qi \big[1-e^{(-N/\qu)}\big]$.
\end{proof}

We now complete the proof of Proposition~\ref{p:userLinear}.
\begin{proof}[Proof of Prop~\ref{p:userLinear}]
	Partitioning recommendations according to $\BBul{u}{a_{u,t}}{t}$ gives
\begin{align*}
N[T - &\reg(T)]= 
\Ex\bigg[{\sum
\,\, \ident[L_{u,a_{u,t}}= +1]}\bigg]\\
& =
 \Ex\bigg[{\sum
 \,\,\Pr[L_{u,a_{u,t}}= +1, \BBul{u}{a_{u,t}}{t}|\H_{t-1}]}\bigg]
 +\Ex\bigg[{\sum
 \,\,\Pr[L_{u,a_{u,t}}= +1, (\BBul{u}{a_{u,t}}{t})^c|\H_{t-1}]}\bigg]\,
\\
 &\overset{\text{Claim~\ref{c:Brandom}}}{\leq} 
 \Ex\bigg[\sum
 \,\,\Big(\Pr[\BBul{u}{a_{u,t}}{t}|\H_{t-1}]+ \frac{1}{2}\Pr[(\BBul{u}{a_{u,t}}{t})^c|\H_{t-1}]\Big)\bigg]
\\
& = 
  NT-\frac{1}{2} \Ex\bigg[{\sum
\,\,   \ident[(\BBul{u}{a_{u,t}}{t})^c]}\bigg]\,
\overset{\text{Claim~\ref{cl:ul-slope-cl3}}}{\leq}
  NT - \frac{T}{2}\qu \big[1-e^{(-N/\qu)}\big] \,,
\end{align*}
where all the summations are over $t\in[T]$ and $u\in[N]$.
Rearranging shows that $\reg(T) \geq \big[1-e^{(-N/\qu)}\big]\frac{\qu}{2N}T$ for all $T$.
\end{proof}
\section{Item-item algorithm and analysis}
\label{s:itemitem}

This section describes a version of item-item CF with explicit exploration steps and analyzes its performance within the setup specified in Section~\ref{sec:model}. 
The algorithm is quite different from the user-user algorithm. This is due to the inherent asymmetry between users and items: multiple users can rate a given item simultaneously but each user can rate only one item at each time-step.

\subsection{Algorithm}\label{ss:itemalg}
Algorithm \textsc{Item-Item} performs the following steps (see Algorithm~\ref{alg:item-fixed}). First, items are partitioned according to type; next, each user's preference for each item type is determined; finally, items from liked partitions are recommended. These steps are now described in more detail. 

Two sets $\Mcc$ and $\Mcct$, each containing $M$ random items, are selected. In the exploration step, each item is recommended to $\thr=\lceil 2\log(\qi/\epsilon)\rceil$ random users. The feedback from these recommendations later helps to partition the items according to type. The parameter $\thr$ is chosen large enough to guarantee small probability of error in partitioning ($\leq \epsilon$ for each item). The use of two sets $\Mcc$ and $\Mcct$, as opposed to just one, is to simplify the analysis; as described next, item type representatives are selected from $\Mcct$, and are used to represent clusters of items from $\Mcc$. 

An item from each of $\ell$ \emph{explored types} is recommended to all users, where $\ell$ is a parameter determined by the algorithm.
It turns out that it is often beneficial (depending on system parameters) to learn user preferences for only a subset of the types, in which case $\ell$ is strictly less than $\qi$.  Each of the $\ell$ items chosen from $\Mcct$ is thought of as a \emph{representative} of its type. 
For each $j=1,\dots,\ell$, all items in $\Mcct$ that appear to be of the same type as the representative item $\rep_{\ti}$ are stored in a set $\St{2}_\ti$ and then removed from $\Mcct$. This guarantees that at each time $\Mcct$ does not contain items with the same type as any of the previously selected representative items. 
For each $j=1,\dots,\ell$, all items in $\Mcc$ that appear to be of the same type as  $\rep_{\ti}$ are stored in a set $\St{1}_\ti$ and then removed from $\Mcc$. 

For each user $u$, we add the items in the groups $\St{1}_\ti$ whose representative $\rep_\ti$ were liked by $u$ to the set of exploitable items $\R_u\,$ by $u$.
Finally, in the exploitation phase, each user is recommended items from $\R_u$.
We choose the number $M$ of items in each of $\Mcc$ and $\Mcct$  as a function of $\ell$ to ensure that there are enough exploitable items in $\R_u$ for all users $u$ for the entire length-$T$ time-horizon. Then, $\ell$ is chosen to minimize regret.

The algorithm description uses the following notation. 
For an item $i$ and time $t>0$, $$\rated_t(i) = \{u\in [N]: a_{u,s}= i \text{ for some } s<t\}$$ is the set of users that have rated item $i$ before time $t$. 
The time $t$ is implicit in the algorithm description, with $\rated(i)$ used to represent $\rated_t(i)$ at the time of its appearance.

\begin{algorithm}
\caption{$\textsc{Item-Item}(T, \qi, N)$ (fixed time horizon)}\label{alg:item-fixed}
\begin{algorithmic}[1]
\State $\ell =\min\Big\{\left\lceil 18\log T+ \sqrt{330\frac{\qi \thr}{ N}T}\,\,\right\rceil,\qi\Big\}$ \vspace{1mm}
		\vspace{1mm}
\State $\epsilon= \frac{1}{2 \qi N}$; $\thr = \left\lceil 2\log\frac{\qi}{\epsilon}\right\rceil$;  $M=\left\lceil \frac{64\qi}{\ell} T \right\rceil$
\vspace{.2cm}
\State $\Mcc \gets M$  random items; $\Mcct\gets M$  new random items
\vspace{.1cm}
		\vspace{1mm}
\State $\R_u\gets\varnothing, \text{ for each } u\in[N]$ \quad (items exploitable by user $u$)
\vspace{.2cm}
\State $ \{\R_u\}_{u\in [N]}\gets  \textsc{ItemExplore}(\Mcc, \Mcct, \ell)$
\vspace{.2cm}
\State $\textsc{ItemExploit}(\{\R_u\}_{u\in [N]})$
\end{algorithmic}
\end{algorithm}

\begin{algorithm}
	\caption{$\textsc{ItemExplore}(\Mcc, \Mcct, \ell)$}\label{alg:itemExplore}
	\begin{algorithmic}[1]
	\State $\St{1}_\ti\gets\varnothing, \quad \ti\in[\qi]$ \quad (initialize sets of items of type $j$ in $\Mcc$)
	\State $\St{2}_\ti\gets\varnothing, \quad \ti\in[\qi]$ \quad (initialize sets of items of type $j$ in $\Mcct$)
	
\vspace{2mm}
		\Statex \emph{Rate items by a few users (preliminary exploration)}
\State Recommend each item in $\Mcc$ and $\Mcct$ to $\thr$ users from $[N]$ (details specified in Remark~\ref{r:usersToItems}).  \label{line:ItemExplore-initialallocation}
		\vspace{2mm}
		\Statex \emph{Partition items and learn item types} 
		\For{$\ti=1,\cdots,\ell$ \textbf{if} $\Mcct\neq \varnothing$} 
		\State $\rep_\ti \gets$ a random item in $\Mcct$ (a representative item)
		\State $a_{u,t}\gets \rep_\ti$ if $u\notin\rated(\rep_{\ti})$, otherwise a random item not from $\Mcct$ 
		\State $\St{1}_\ti \gets \big\{i\in \Mcc: L_{u, i} = L_{u,\rep_\ti}\text{ for all } u\in \rated(i)\cap \rated(\rep_{\ti})\big\}$,  \quad (users agree on $i$ vs. $\rep_\ti$)
		\vspace{.07cm}
		\State $\Mcc\gets \Mcc\setminus \St{1}_\ti$ 
 		\vspace{.07cm}
			\State $\St{2}_\ti \gets \big\{i\in \Mcct: L_{u, i} = L_{u,\rep_\ti} \text{ for all } u\in \rated(i)\cap \rated(\rep_{\ti})\big\}$
			 		\vspace{.07cm}
		\State $\Mcct\gets \Mcct\setminus \St{2}_\ti $
		\EndFor  
		\State $\R_u = \bigcup_{j\in[\ell]:L_{u,\rep_\ti}=+1} \St{1}_j\,\,\,$ for each $u\in [N]$
		\label{Line-alg:ItemUpperRu}
	\end{algorithmic} 
\Return  $\{\R_u\}_{u\in [N]}$
\end{algorithm}

\begin{algorithm}
	\caption{$\textsc{ItemExploit}(\{\R_u\}_{u\in [N]})$}\label{alg:itemExploit}
	\begin{algorithmic}[1]
\For{ remaining $ t \leq T $}
\For{$u\in [N]$}
\If{there is an item $i \in\R_u$ such that $u\notin \rated_t(i)$,} 
\State $a_{u,t}\gets i$ 
\Else  { } $a_{u,t} \gets$ a random item not yet rated by $u$.
\EndIf
\EndFor
\EndFor
	\end{algorithmic}
\end{algorithm}

\begin{remark} 
The set of items $\Mcc$ and $\Mcct$ are updated throughout algorithm $\textsc{ItemExplore}$. 
In the proof,  we use the notation 
$\Mcc_0$ and $\Mcct_0$ to refer to the set of items $\Mcc$ and $\Mcct$ at the beginning of the algorithm $\textsc{Item-Item}$. 
\end{remark}
\begin{remark} \label{r:usersToItems}
Assignment of users to items for Line~\ref{line:ItemExplore-initialallocation} of {\sc ItemExplore}: This is done over $\lceil{(|\Mcc_0|+|\Mcct_0|)\thr}/{N}\rceil\leq 2M\thr/N+1$ time-steps, with additional recommendations being random new items. The main requirement in assigning users to items is to make Claim~\ref{cl:Ftilde}  (which bounds the probability of mis-classification of each item) hold.
 Since the claim addresses each item separately, we may introduce dependencies between sets of users for different items. What is important is that the set of users assigned to each specific item is uniform at random among all sets of users (or at least contains a random subset with size $\thr$). For example, one can choose a random permutation over the users, repeat this list $\lceil{(|\Mcc_0|+|\Mcct_0|)\thr}/{N}\rceil\leq 2M\thr/N+1$ times, and then assign each item to a block of $\thr$ users.
\end{remark}

\begin{theorem}\label{th:Item-upper}
	Let $\thr=\lceil 2\log(2N\qi^2) \rceil$. 
Suppose $\qi>13\log N$ and $\qu>4\thr$.
	Then \textsc{Item-Item} (Algorithm \ref{alg:item-fixed}) obtains regret per user at time $T$ upper bounded as 
\begin{equation*}
\reg(T)
< 
\min\left\{ Y(T),\, \frac{T}{2}\right\}
\end{equation*} 
where we define 
\begin{equation}\label{eq:IU-RegBound}
Y(T)=
4+ \max\left\{52\log T \, ,\, 48 \sqrt{\frac{\qi \thr}{N}T} ,\, 270 \frac{\thr}{N} T\right\}\,.
\end{equation} 
\end{theorem}

The simplified version of this theorem in Section \ref{s:results} is obtained  $2\log (N\qi)< \thr<5\log (N\qi)$ and $N>5$.

\begin{remark}\label{remark:II-Upper-TrivialAlg}
The regret bound we get for the algorithm is actually $\reg(T)\leq Y(T)$. One can obtain $T/2$ by a trivial algorithm which recommends random items independent of feedback to all users. This trivial algorithm  improves on our bound for 
the parameter range  in which $T/2 < Y(T)$. Thus in our analysis we focus on the parameter range where $Y(T)\leq T/2$; one consequence is that $\ell < T$ as chosen in Algorithm~\ref{alg:item-fixed}.
\end{remark}

\subsection{Proof of Theorem~\ref{th:Item-upper}}
We prove Theorem~\ref{th:Item-upper}, deferring several lemmas and claims to the next subsection.

The basic error event is misclassification of an item. In  \textsc{ItemExplore} (Alg.~\ref{alg:itemExplore}), $\St{1}_\ti$ is the set of items in $\Mcc$ that the algorithm posits are of the same type as the $j$-th representative $\rep_\ti\,.$ 
 Let $\EE^1_i$ be the event that item $i$ was mis-classified, 
		\begin{equation}\label{e:EE}
		\EE^1_i = \{\exists \ti : i\in\St{1}_\ti, \tau_I(i)\neq \tau_I(\rep_\ti)\}\,.
		\end{equation}
	Let $T_0$ be the number of time-steps spent making recommendations in \textsc{ItemExplore}.
	Recall that $\R_u$ is the set of items to be recommended to user $u$ in the exploit phase.
We partition the recommendations made by Algorithm \textsc{Item-Item} to decompose the regret 
as follows:
\begin{align}
N\reg(T) & =  \Exp{ \sum_{u=1}^N\sum_{t=1}^{\Tu} \ident[L_{u,a_{u,t}}=-1]}\notag\\
&\quad + \Exp{ \sum_{u=1}^N\sum_{t=\Tu+1}^T  \ident[L_{u,a_{u,t}}=-1, a_{u,t}\notin\R_u]}\notag\\
& \quad +  \Exp{ \sum_{u=1}^N\sum_{t=\Tu +1}^T  \ident[L_{u,a_{u,t}}=-1, a_{u,t}\in\R_u, \EE^1_{a_{u,t}}]}\notag\\
& \quad +  \Exp{ \sum_{u=1}^N\sum_{t=\Tu+1}^T \ident[L_{u,a_{u,t}}=-1, a_{u,t}\in\R_u, (\EE^1_{a_{u,t}})^c]} \notag\\
& =: \Asf{1} + \Asf{2} + \Asf{3} + \Asf{4}\,. \label{eq:IU-reg-decomp}
\end{align}
The first term, $\Asf{1}$, is the regret from early time-steps up to $\Tu$. The second term,  $\Asf{2}$, is the regret due to not having enough items available for the exploitation phase, which is proved to be small with high probability for sufficiently large $M$. The third term, $\Asf{3}$, is the regret due to exploiting the misclassified items. It is small since few items are misclassified with the proper choice of $\epsilon$ and $\thr$. The fourth term, $\Asf{4}$, is the regret due to exploiting the correctly classified items. It is intuitively clear and will be checked later that $\Asf{4}=0\,.$
\paragraph{Bounding $\Asf{1}$.}
 Line~\ref{line:ItemExplore-initialallocation} of \textsc{ItemExplore} takes at most $\big\lceil\frac{2M\thr}{N}\big\rceil$ units of time to rate every item in $\Mcc_0$ and $\Mcct_0$ by $\thr$ users each, since $N$ users provide feedback at each time-step. 
 Remark~\ref{r:usersToItems} above discusses the assignment of users to items in this phase.
 After this, $\ell$ representative items are rated by every user in the for loop (lines 4 through 10 of \textsc{ItemExplore}),
 which takes $\ell$ time-steps. 
 This gives
\[\Tu \leq  \bigg\lceil\frac{2M\thr}{N}\bigg\rceil + \ell  \leq  \frac{2M\thr}{N} + \ell +1 \,.\] 
For $t\leq \Tu\,,$ from the perspective of any user the items recommended to it are of random type and hence 
\begin{align}
\Asf{1} & =  \Exp{\sum_{u=1}^N \sum_{t=1}^{\Tu} \ident[L_{u,a_{u,t}}=-1]} = N \Tu \leq  2 {M\thr} +(\ell +1)N \,.\label{eq:IU-bound-A}
\end{align}
\paragraph{Bounding $\Asf{2}$.}
Time-steps $t>\Tu$ are devoted to exploitation as described in \textsc{ItemExploit}. 
During this phase, a random item $a_{u,t}\notin \R_u$ is recommended to user $u$ only when there are no items to exploit because all items in $\R_u$ have already been recommended to $u$. So, the total number of times an item $a_{u,t}\notin \R_u$ is recommended in the time interval $\Tu <  t\leq T$ is at most $(T-|\R_u|)_+$, and
\begin{align}
\Asf{2} &= \Ex\bigg[{\sum_{u\in[N]} {\sum_{t=\Tu+1}^T \ident[L_{u,a_{u,t}}=-1, a_{u,t}\notin \R_u]}}\bigg] \leq
\Ex \bigg[{  \sum_{u\in [N]}  {\sum_{t=\Tu+1}^T \ident[ a_{u,t}\notin \R_u]}}\bigg]\notag\\
&\leq \sum_{u\in [N]} \Exp{(T -|\R_{u} |)_+}
 \leq 
 3TN\exp(-{\ell}/{13})+ 30T
\,,
 \label{eq:IU-bound-B}
\end{align}
where the last inequality is from Lemma~\ref{l:IU-exp-R-u} in Section~\ref{sec:ItemUpper-Lemmas}.

\paragraph{Bounding $\Asf{3}$.}
Term $\Asf{3}$ in~\eqref{eq:IU-reg-decomp} is the expected number of mistakes made by the Algorithm {\sc ItemExploit} as a result of misclassification. Claim~\ref{cl:iu-error-indep2} below upper bounds the expected number of ``potential misclassifications'' (defined in Equation~\eqref{eq:Ftilde}) in the algorithm to provide an upper bound for this {quantity}: 
\begin{align}
	\Asf{3}& = \Ex\bigg[{\sum_{u=1}^N \sum_{t=\Tu+1}^T  \ident[L_{u,a_{u,t}}=-1, a_{u,t}\in\R_u, \EE^1_{a_{u,t}}]}\bigg]
	\notag\\
	& \leq  \Ex\bigg[{\sum_{u=1}^N \sum_{t=\Tu+1}^T \sum_{i\in\mathbb{N}} \ident[a_{u,t}=i, i\in\R_u, \EE^1_i]} \bigg]
	\overset{\textrm{Claim}~\ref{cl:iu-error-indep2}}{\leq}  2 N M \epsilon \leq \frac{M}{\qi}\,.\label{eq:IU-bound-C}
\end{align}
\paragraph{Bounding $\Asf{4}$.}
By definition of the mis-classification  event, $\EE^1_i$, given in Equation~\eqref{e:EE}, if an item $i\in\St{1}_{\ti}\,$ is correctly  classified, then $\tau_I(i)=\tau_I(\rep_{\ti})$. 
Since $L_{u,i} = \xi_{\tau_U(u),\tau_I(i)}$, all users rate $i$ the same as $\rep_{\ti}\,$. 
By construction of the sets $\R_u$ in Line~\ref{Line-alg:ItemUpperRu} of  {\sc ItemExploit}, for any item $i\in\R_u$, there is some $\ti\in[\qi]$ such that $i\in\St{1}_{\ti}$ and $u$ likes item $\rep_{\ti}\,$.
 Hence,
\begin{align}
\Pr[L_{u,i}=-1| i\in\R_u, (\EE^1_i)^c ] &= \Pr[L_{u,i}=-1| \exists \ti: L_{u,\rep_\ti}=+1, i\in\St{1}_{\ti}, \tau_I(i)= \tau_I(\rep_\ti)]\notag\\
& = \Pr[L_{u,i}=-1| \xi_{\tau_U(u),\tau_I(i)}=+1] =0\,.\label{eq:IU-bound-D}\notag
\end{align}
It follows that $\Asf{4} = 0$.
\paragraph{Combining all the bounds.}
Plugging in Equations~\eqref{eq:IU-bound-A},~\eqref{eq:IU-bound-B},~\eqref{eq:IU-bound-C} and~$\Asf{4} = 0$ into Equation~\eqref{eq:IU-reg-decomp}
 gives
\begin{align*}
N\reg(T) 
&\leq 2 M\thr + (\ell+1) N
+
3TN\exp(-{\ell}/{13})+ 30T + \frac{M}{\qi}\,.
\end{align*}
Setting $M=\frac{64T\qi}{\ell}$ gives
\begin{align*}
\reg(T)
& \overset{(a)}{\leq} 
2 + 134  \frac{T\qi \thr}{\ell N} +\ell  + 3T\exp\Big(-\frac{\ell}{13}\Big)
\overset{(b)}{\leq} 
\begin{cases}
4+ 7\log T+ 24\sqrt{\frac{\qi \thr}{ N}T\,}\, &\text{ if } \ell<\qi\\
{\qi+1}+ 135\frac{\thr}{N}T
&\text{ if } \ell = \qi
\end{cases}
\end{align*}
where (a) holds for $\ell\geq 13$ (imposed by the algorithm for $T\geq 2$) which gives  $M/\qi  \leq  5T$ and $\thr>6$ (which holds if $\qi,N\geq 3$).
If $\ell=\qi$, since $\qi>13\log N$, we have $3T\exp(-\ell/13)\leq 3T/N\leq \thr\, T/N$. 
(b) is obtained by choosing the parameter $\ell$ to be 
$$\ell =\min\left\{\left\lceil 13\log T+ \sqrt{134\frac{\qi \thr}{ N}T}\,\,\right\rceil,\qi \right\}\,.$$ 
Next, we will show that  $\reg(T)\leq Y(T)$ for $Y(T)$ defined in~\eqref{eq:IU-RegBound}. This is clearly true for $\ell<\qi$. 

To show  $\reg(T)\leq Y(T)$ for $\ell=\qi$, we consider two cases:  If $\qi \leq 135 T\thr/N$, then $\reg(T)\leq 270 T\thr/N+1$. If  $135 T\thr/N< \qi$, we know that  $\qi\leq\left\lceil 13\log T+ \sqrt{134\frac{\qi \thr}{ N}T}\,\,\right\rceil$ (which is true since $\ell=\qi$). This gives $\reg(T)\leq 2 \left\lceil 13\log T+ \sqrt{134\frac{\qi \thr}{ N}T}\,\,\right\rceil\leq Y(T)$.

\subsection{Lemmas used in the proof}\label{sec:ItemUpper-Lemmas}
In the remainder of this section we state and prove the lemmas used in the analysis above.
\begin{claim}\label{cl:iu-error-indep2}
Suppose that $\qu>4\thr$.
The expected total number of times a misclassified item is recommended in the exploitation step is upper bounded as
\[ \Exp{\sum_{u=1}^N \sum_{t=\Tu+1}^T \sum_{i\in\mathbb{N}} \ident[a_{u,t}=i, i\in\R_u, \EE^1_i]}\leq  N M  \epsilon\,.\]
\end{claim}
\begin{proof} 
Event $\EE^1_i$, defined in~\eqref{e:EE} to be the misclassification of item $i\in\Mcc$, occurs if the preferences of random users classifying item $i$ is the same as a previous representative item with a different type. 
Given matrix $\Xi$, this event is a function of the order of choosing the representative items and the choice of random users in Line~\ref{line:ItemExplore-initialallocation} of {\sc ItemExplore}. 
Instead of directly analyzing $\EE^1_i$, we define an event called ``potential error event'' $\EEb_{i,U_i}$, which will be shown to satisfy $\EE^1_i\subseteq\EEb_{i,U_i}$; an upper bound for $\Pr[\EEb_{i,U_i}]$  is given in Claim~\ref{cl:Ftilde}.

For item $i$ and subset of users $U\subseteq [N]\,,$ define 
\begin{equation}
\label{eq:Ftilde}
\EEb_{i,U} =\{\exists \ti\neq \tau_I(i): L_{u,i}=\xi_{\tau_U(u),\ti}\,, \text{ for all } u\in U\}
\end{equation}
 to be the event that the ratings of users in $U$ for item $i$ agree with some other item type. 
For item $i\in\St{1}_{\ti}$, if $t$ is the time $i$ is added to $\St{1}_{\ti}$ in the exploration phase, let $U_i=\rated_t(i)\cap\rated_t(\rep_{\ti})$ be the set of witness users whose ratings were used to conclude that $i$ and $\rep_{\ti}$ are of the same type. 
Item $i$ is added to $\St{1}_{\ti}$ only if all users in $U_{i}$ agree on $i$ vs. $\rep_{\ti}\,$, so misclassification $\EE^1_i$ (defined in Equation~\eqref{e:EE}) implies $\EEb_{i,U_i}$. 
We can now deduce the inequalities, justified below:
\begin{align*}
\Ex \bigg[\sum_{u=1}^N \sum_{t=\Tu+1}^T \sum_{i\in\mathbb{N}} \ident[a_{u,t}=i, i\in\R_u, \EE^1_i]\bigg]
& \overset{(a)}{\leq}\Ex\bigg[{\sum_{u=1}^N \sum_{t=\Tu+1}^T\sum_{i\in\Mcc_0} \ident[a_{u,t}=i, i\in\R_u, \EEb_{i,U_i}]} \bigg]\\
& \overset{(b)}{\leq}  N\Ex\bigg[{\sum_{i\in\Mcc} \ident[\EEb_{i,U_i}]}\bigg]
\overset{(c)}{\leq} 2 N M  \epsilon\,.
\end{align*}%
 Inequality (a) holds because according to Line~\ref{Line-alg:ItemUpperRu} of {\sc ItemExplore} for every user $u\,$, the set $\R_u$ is a subset of $\Mcc_0$, and also the containment $\EE_i\subseteq \EEb_{i,U_i}$;
(b) follows since each item $i$ is recommended at most $N$ times; 
(c) uses $|\Mcc_0|= M$ together with Claim~\ref{cl:Ftilde} which shows that $\Pr[\EEb_{i,U_i}]\leq 2\epsilon$.
\end{proof}

\begin{claim}\label{cl:Ftilde} Suppose that $\qu>4\thr$.
Consider the ``potential error" event $\EEb_{i,U_i}$ defined in \eqref{eq:Ftilde} with the set of users $U_i$ defined immediately after. 
Then
$\Pr[\EEb_{i,U_i}] \leq 2\epsilon$ for all $i\in\Mcc_0$.
\end{claim}

\begin{proof}
 Each representative item $\rep_j$ chosen in $\textsc{ItemExplore}$ is rated by all of the users.
Other items in $\Mcc_0$ and $\Mcct_0$ are rated by at least $\thr$ users in the exploration phase. 
Remark~\ref{r:usersToItems} describes how Line~\ref{line:ItemExplore-initialallocation} in {\sc ItemExplore} is performed to guarantee that the set of $\thr$ users assigned to each specific item is uniformly at random among all subset of users of size $\thr$. 
By Lemma~\ref{l:ballsbins}, if $\qu>4\thr$, then with probability at least $1-\exp(-\thr/2)\geq 1-\epsilon/\qi$ there are $\thr/2$ users with distinct user types in a specific~$U_i$.
 It follows that the $\thr/2$ users with distinct types (chosen independently of the feedback) that rate item $i$ of type $\tau_{I}(i)$ also have the same ratings for type $j\neq \tau_{I}(i)$
 with probability at most $2^{-\thr/2}$. 
 (Any two item types $j\neq j'$ have jointly independent columns in the preference matrix.) The choice $\thr\geq 2\log({\qi}/{\epsilon})$ and a union bound over item types $\ti$ completes the proof. 
\end{proof}

\begin{lemma}\label{l:IU-exp-R-u}
For user $u\in[N],$ let $\R_{u}$ be defined as in Line 10 of algorithm \textsc{ItemExplore}. 
Then
\[\Pr\big[\,|\R_u|\leq T\,\big] \leq 3\, e^{-\nicefrac{\ell}{18}}\,+\,\tfrac{30}{N}\,,\]
and 
\[\Exp{(T- |\R_u |)_+} \,\leq\, 3\, T\, e^{-\nicefrac{\ell}{18}}+ \tfrac{30}{N}T\,.\]
\end{lemma}

\begin{proof}
We begin with a sketch. 
Any user $u$ likes roughly half of the  $\ell$ item types  that have representatives $\{\rep_j\}_{j=1}^\ell$. The total number of items in $\Mcc_0$ is $\tfrac{64\,\qi}{\ell}T$, so there are about $\frac{64}{\ell}T$ items from each of the item types in $\Mcc_0$. 
Adding over the $\ell$ types, the set $\R_u$ of items in $\Mcc_0$ that user $u$ likes will typically have size at least~$T$. 

Making this argument rigorous requires some care for the following reasons: 
1) $\R_u$ is the union of the $\St{1}_{\ti}$'s with $L_{u,\rep_{\ti}}=+1$, but $\St{1}_{\ti}$ can be missing items
due to misclassification (there may be items of the same type as  $\rep_{\ti}$ that have been classified as being of the same type as some $\rep_{j'}$ for which $L_{u,\rep_{\ti'}}=-1$).
2) The distribution of the type of each  representative item depends on the number of remaining items of each type in $\Mcct$ when the representative is chosen.  
Again, due to misclassification, this can be different from the actual number of items of each type initially present in $\Mcct$. Moreover, because misclassification of an item depends on the ratings of users for the item, the choice of next item type to be represented is therefore dependent on ratings of users for other types. 
The effect of this dependence is addressed in Claim~\ref{cl:IU-lukbd} below.

We now proceed with the proof, bounding the size of $\R_u$ by introducing a different set.
 For $u\in[N]$, let
\begin{equation}\label{eq:IU-Ruk}
\Rt^1_{u}=\{i\in\Mcc_0: \tau_I(i)=\tau_I(\rep_{\ti}), \text{ for some }\ti\in[\ell] \text{ such that } L_{u,\rep_{\ti}}= +1\}
\end{equation} 
be the items in $\Mcc_0$ whose types are the same as one of the representatives $\rep_{\ti}$ that are liked by $u$.
Note that if an item $i\in \Rt^1_{u}$ is correctly classified by the algorithm (i.e. $(\EE^1_i)^c$ occurs, where $\EE^1_i$ is defined in Equation~\eqref{e:EE}), then $i\in \R_u$. 
Hence, $i\in \Rt^1_u\setminus \R_u$ implies $\EE^1_i$ and since $\EE^1_i\subseteq \EEb_{i,\U_i}$ we have
\begin{equation}\label{eq:IU-Rubound}
|\R_u|\geq |\Rt^1_{u}| - \sum_{i\in\Mcc_0}\ident[\EE^1_i]\geq |\Rt^1_u| - \sum_{i\in\Mcc_0}\ident[\EEb_{i,\U_i}]\,.
\end{equation}
It follows that
\begin{align}
\Pr\big[\,|\R_u|\leq T\,\big]
&  \leq 
\Pr\bigg[\,|\Rt^1_{u}| - \sum_{i\in\Mcc_0}\ident[\EEb_{i,\U_i}]\leq T\,\bigg]
\leq 
\Pr\Big[\,|\Rt^1_{u}|< \tfrac{3 T}{2} \,\Big]
+
\Pr\bigg[\,\sum_{i\in\Mcc_0}\ident[\EEb_{i,\U_i}]\geq  \tfrac{T}{2} \,\bigg]
\label{eq:IU-Rubound3}\,.
\end{align}
To bound the second term we use Claim \ref{cl:Ftilde} above, which gives
\begin{equation}\label{eq:IU-Ftildebound}
\Ex\bigg[{\sum_{i\in\Mcc_0} \ident[\EEb_{i,\U_i}]}\bigg]
\leq \,\,
2\,\big|\Mcc_0\big|\, \epsilon 
= \frac{64T}{N\ell}\,,
\end{equation}
and by Markov's Inequality\footnote{Note that application of Markov's Inequality does not require the events $\EEb_{i,\U_i}$ to be jointly independent for various values of $i\in\Mcc_0$. This is discussed in Remark~\ref{r:usersToItems}.}
\begin{equation}\label{eq:IU-Ftildebound2}
\Pr\Big[\,\sum_{i\in\Mcc_0} \ident[\EEb_{i,\U_i}]\geq \tfrac T2\,\Big] \leq \frac{128}{ \ell\, N}\,.
\end{equation}
We now bound the first term on the right-hand side of~\eqref{eq:IU-Rubound3}.
To this end, let 
\begin{equation}
\label{eq:IU-Luk}
\widetilde\Lc_u^{(\ell)}=\{\tau_I(\rep_{\ti}):\ti\in[\ell], L_{u,\rep_{\ti}}=+1\}
\end{equation}
 be the type of item representatives that are liked by user $u$. 
By definition of $\Rt^1_u$ in~\eqref{eq:IU-Ruk},
\begin{equation*}
|\Rt^1_u| = \sum_{i\in\Mcc_0} \ident[\tau_I(i)\in\widetilde\Lc^{(\ell)}_u].
\end{equation*}
Now, we claim that the set $\widetilde\Lc^{(\ell)}_u$ is independent of all types and preferences for items in $\Mcc_0$.
To see this, note that $\widetilde\Lc^{(\ell)}_u$ is determined by row $u$ of the type matrix $\Xi$ and
the items in $\Mcct_0$, their types, and randomness in the algorithm, which determines the choice of $\rep_{\ti}$ and $\tau_I(\rep_{\ti})$; these together determine $L_{u,\rep_{\ti}}= \xi_{\tau_U(u),\tau_{I}(\rep_{\ti})}$ and $\tau_I(\rep_{\ti})$.
So, conditioning on $\widetilde\Lc^{(\ell)}_u$ having cardinality $\tilde\ell_u$, $|\Rt^1_u|$ is the sum of $|\Mcc_0| = 64\frac{\qi }\ell T$ i.i.d. Bernoulli variables with parameter $\frac{\tilde\ell_u}{\qi}$ and hence $|\Rt^1_u|\sim\mathrm{Binom}(64\frac{\qi }\ell T,\frac{\tilde\ell_u}{\qi})$. 
Conditioning on $\tilde\ell_u\geq\frac{\ell}{20}$, by a Chernoff bound (Lemma~\ref{l:Chernoff}) and stochastic domination of binomials with increasing number of trials we obtain
\[\Pr\left[|\Rt^1_u| < \tfrac{3 T}{2} \,\Big|\,\tilde\ell_u\geq \tfrac{\ell}{20}\right]
\leq 
\, \exp{\big(-\tfrac{T}{8}\big)}
\,.\]
In Lemma \ref{cl:IU-lukbd} below, we will lower bound the probability that $\tilde\ell_u\geq {\ell}/{20}$. 
Combining the last displayed inequality with Lemma \ref{cl:IU-lukbd}, we get
\[\Pr\big[|\Rt^1_u| <  \tfrac{3 T}{2} \big] 
\leq 
\exp{\big(-\tfrac{T}{8}\big)} + 
\exp{\big(-\tfrac{\ell}{13}\big)} 
+ \exp{\big(-\tfrac{\qi}{8}\big)}
+ \frac{20}{N}\,.
\]
Plugging this and~\eqref{eq:IU-Ftildebound2} into~\eqref{eq:IU-Rubound3} gives 
\[\Pr\big[\, |\R_u| < T\big]
\leq 
\exp{\big(-\tfrac{T}{8}\big)} + 
\exp{\big(-\tfrac{\ell}{13}\big)} 
+ \exp{\big(-\tfrac{\qi}{8}\big)}
+ \frac{20}{N} +
\frac{128}{\ell\,N} 
 \,,\]
and since $\ell\leq T,\qi\,$\footnote{As discussed in Remark~\ref{remark:II-Upper-TrivialAlg}, the current analysis of regret assumes that  $Y(T)<\frac{T}{2}$. It can be shown that this assumption implies $\ell\leq T$.
 The choice of $\ell$ in the algorithm clearly guarantees $\ell\leq \qi$.} and $\ell\geq 13$ (for $T\geq 13$)
\[\Pr[|\R_u| < T]\leq 3\exp(-{\ell}/{13})+ \tfrac{30}{ N} \,.\]
The bound on $\Exp{(T- |\R_u |)_+} $ is an immediate consequence.
\end{proof}
\begin{lemma}\label{cl:IU-lukbd}
Fix $u\in[N]$ and suppose that $\qi>8\thr$. 
With $\widetilde\Lc^{(\ell)}_u$ defined in~\eqref{eq:IU-Luk} 
 and $\tilde\ell_u=\big|\widetilde\Lc^{(\ell)}_u\big|$,
\begin{equation*}
\Pr\big[\,\tilde\ell_u< {\ell}/{20}\,\big] \leq \exp(-{\ell}/{13}) + \exp(-{\qi}/{8})+ {20}/{N}\,.
\end{equation*}
\end{lemma}
\begin{proof}
For a given user $u$, $\tilde\ell_u$ is the number of item type of representatives liked by user $u$. 
Let $\Lc_u$ be the item types in $[\qi]$ that are liked by user $u$: 
\begin{equation}
\Lc_u= \{\ti\in[\qi]: \xi_{\tau_U(u),\ti}=+1 \}\,.
\end{equation}
The variables $\{\xi_{\tau_U(u),j}\}_{j\in [\qi]}$ are i.i.d.  $\mathrm{Bern}(1/2)$, so a Chernoff bound (Lemma \ref{l:Chernoff}) gives 
\begin{equation}
\label{eq:IU-Lcubd}
\Pr\big[\,|\Lc_u|< {\qi}/{4}\,\big]
<
 \exp(-{\qi}/{8})\,.
\end{equation}
We will show that 
\begin{equation}\label{eq:IU-ell-condLu}
\Pr\big[\tilde\ell_u \leq {\ell}/{20} \,\big|\, |\Lc_u|
\geq {\qi}/{4}\big] 
\leq 
\exp(-{\ell}/{13}) + {20}/{N}\,,
\end{equation}
which will prove the lemma by combining with \eqref{eq:IU-Lcubd}.

We now work towards defining a certain error event in Equation~\eqref{def:IU-iuerr} below.
Let the sequence of random variables $X_1 = \tau_I(\rep_1), X_2 = \tau_I(\rep_2), \dots, X_{\ell}=\tau_I(\rep_{\ell})$ denote the types of the item representatives chosen by the algorithm, so that 
$\tilde\ell_u=|\widetilde\Lc^{(\ell)}_u| = \sum_{\ti\in[\ell]}\ident[X_{\ti}\in \Lc_u]$
 (with $\widetilde\Lc^{(\ell)}_u$ defined in~\eqref{eq:IU-Luk}). 
Let $\Rb^2(\ti)$ be the set of items in $\Mcct_0$ of type $\ti$
\begin{equation}\label{IU-Rb2}
\Rb^2(\ti)= \{i\in\Mcct_0:\tau_I(i)=\ti\}\,.
\end{equation}
Later we will use the notation 
$\Rb^2(\cdot)= \{\Rb^2(j)\}_{j=1}^{\qi}$ 
for the collection of these sets. 
Now let the event $\EE^2_i$ denote misclassification of an item $i\in\Mcct_0$ (similar to event $\EE^1_i$ for $i\in\Mcc_0$ in~\eqref{e:EE}),
	\begin{equation}\label{e:EEt}
		\EE^2_i = \{\exists \ti : i\in{\mathcal{S}}^2_\ti, \tau_I(i)\neq \tau_I(\rep_\ti)\}\,.
		\end{equation}
Let $\iuerr$ be the event that for some item type $\ti$, more than a fraction $1/10$ of the items in $\Mcct_0$ of type $\ti$ are misclassified:
\begin{equation}\label{def:IU-iuerr}
\iuerr = \bigg\{\sum_{i\in\Rb^2(\ti)}\ident[\EE^2_i]> \frac{|\Rb^2(\ti)|}{10}, \text{ for some } \ti\in[\qi]\bigg\}.
\end{equation}
Conditioning on $\iuerr$ in the left-hand side of \eqref{eq:IU-ell-condLu} gives
\begin{align}
 &
 \Pr\bigg[\sum_{i=1}^{\ell}\ident\big[X_i\in\widetilde\Lc^{(\ell)}_u\big]<\ell/20
  \, \bigg|\, 
   |\Lc_u|\geq \qi/4\bigg] \notag
 \\
& \leq
 \Pr\bigg[\,\,\sum_{i=1}^{\ell}\ident\big[X_i\in\widetilde\Lc^{(\ell)}_u\big]<\ell/20 \, \bigg|\, \iuerr^c , |\Lc_u|\geq \qi/4\bigg] 
+
 \Pr\Big[\,\iuerr \,\,\Big|\, |\Lc_u|\geq \qi/4\,\Big]\,.
\label{e:errTemp}
\end{align}

\paragraph{Bound on second term of \eqref{e:errTemp}}
We will use Markov's inequality to bound 
$$
\Pr\Big[\,\iuerr \,\Big|\, |\Lc_u|\geq \qi/4\,\Big] 
=
 \Exp{\Pr\big[\,\iuerr\,\big|\,\Rb^2(\cdot), {\Lc}_u \,\big] \,\Big|\, |\Lc_u|\geq \qi/4 }\,.
 $$
We need to consider the effect of matrix $\Xi$ (and in particular its $\tau_U(u)$th row, which determines $\Lc_u$) on the probability of error in categorizing items. 
Notably, if users rate two item types similarly, the probability of misclassifying items of these types is higher. 
As a result, ${\Lc}_u$ contains some information about discriminability of distinct item types and we need to control this dependence.

Claim \ref{cl:iu-error-indep2} shows that for $i\in\Mcc_0$, the probability of $\EE^1_i$ (the event that item $i$ is miscategorized) is $\Pr[\EE^1_i]\leq 2\epsilon$. 
The same proof gives $\Pr[\EE^2_i]\leq 2\epsilon$ (with $\EE^2_i$ defined in~\eqref{e:EEt} for $i\in\Mcct_0$). 
To show that,
   let the potential error event $\EEb_{i,U_i}$, be the event that there exists an item type $\ti\neq \tau_I(i)$ such that for all the users $u$ in $U_i$, we have $L_{u,i}=\xi_{\tau_U(u),\ti}$. 
      For $i\in{\S}^2_{\ti}$, let $U_i=\rated(i)\cap\rated(\rep_{\ti})$ at the time the algorithm added $i$ to ${\S}^2_{\ti}$. 
Then the containment $\EE^2_i\subseteq\EEb_{i,U_i}$ holds. 
We will later use the inequality $\Pr[\EE^2_i | \tau_I(i), \Rb^2(\cdot),{\Lc}_u]\leq \Pr[\EEb_{i,\U_i} | \tau_I(i), \Rb^2(\cdot),{\Lc}_u]$ and focus on the set of users 
$\U_i\setminus\{u\}$, using 
$\EEb_{i,\U_i}\subseteq \EEb_{i,\U_i\setminus \{u\}}$. 
%
%

As specified in Algorithm $\textsc{ItemExplore}$, 
for each $i\in \Mcc_0\cup \Mcct_0$ the set $\U_i$ of at least $\thr$ users, chosen independently of feedback, rate item $i$. 
By Lemma~\ref{l:ballsbins}, if $\qu>4\thr$, then
there are users of at least $\thr/2$  distinct user types in $\U_i$ with probability at least $1-\exp(-\thr/2)\geq 1-\epsilon/\qi\,.$
%
Conditional on $\U_i$ having at least $\thr/2$  distinct user types, there are at least $\thr/2-1$ users of distinct types in $\U_i\setminus\{u\}$---also distinct from type of $u$---whose preferences for item type $\tau_I(i)$ are independent of ${\Lc}_u$ and $\Rb^2(\cdot)$.

Conditional on $ \tau_I(i), \Rb^2(\cdot),{\Lc}_u$, any two item types $\ti\neq \ti'$ have jointly independent user preferences by $\thr/2-1$ users with distinct types  in $\U_i\setminus\{u\}$; 
they are rated in the same way by these users with probability at most $2^{-(\thr/2-1)}$. 
A union bound over item types $\ti'\neq\tau_I(i)$  gives $\Pr[ \EEb_{i,\U_i\setminus \{u\}} | \tau_I(i), \Rb^2(\cdot),{\Lc}_u]\leq \qi 2^{-(\thr/2-1)}$.
The choice 
$\thr>2\log(\qi/\epsilon)$ 
and 
$\epsilon=\frac{1}{2\qi N}$ 
and the containment 
$\EE^2_i \subseteq \EEb_{i,\U_i} \subseteq \EEb_{i,\U_i\setminus \{u\}}$
gives 
$\Pr[\EE^2_i | \tau_I(i), \Rb^2(\cdot),{\Lc}_u]\leq \frac{1}{\qi N}$ for all $i\in\Mcct_0$.
Knowing $\Rb^2(\cdot)$ determines $\tau_I(i)$ for  $i\in\Mcct_0$. 
Hence,  
$\Pr[\EE^2_i | \Rb^2(\cdot),{\Lc}_u]\leq \frac{1}{\qi N}$ and
Markov's inequality gives
\[\Pr\bigg[\,\sum_{i\in\Rb^2(\ti)} \ident[\EE^2_i] >\frac{|\Rb^2(\ti)|}{10} \, \bigg| \, \Rb^2(\cdot), {\Lc}_u\,\bigg] \leq \frac{20}{\qi N}\,.\]
Union bounding over $\ti\in[\qi]$ and tower property gives the desired bound. 

\paragraph{Bound on first term of \eqref{e:errTemp}.}
We start by showing that conditioned on the event $\iuerr^c$ conditional on  $\Lc_u$ and variables $X_1,\cdots,X_{m-1} $ and $ \Rb^2(X_1),\cdots, \Rb^2(X_{m-1})$, the type of the $m$-th representative item, $X_m$, is almost uniform over all the item types not learned yet, $[\qi]\setminus\{X_1,\cdots,X_{m-1}\}$. 
%
Concretely, we will find upper and lower bounds on
\begin{equation}\label{eq:IU-nextrep}
\Pr\big[X_{m}=\ti \,\,\big|\, \big\{X_n\big\}_{n=1}^{m-1}, \big\{\Rb^2(X_n)\big\}_{n=1}^{m-1}, \,\Lc_u\,, \iuerr^c\,\big]
\end{equation}
for any $\ti\in [\qi]\setminus\{X_1,\cdots,X_{m-1}\}$. Later, we will focus on $\Lc_u$ such that $|\Lc_u|\geq \qi/4$.

\emph{Lower bound for~\eqref{eq:IU-nextrep}.}
Let $t$ be the time the $m$-th item type representative is chosen by the algorithm. For $\ti\in [\qi]\setminus\{X_n\}_{n=1}^{m-1}$ the probability of choosing representative of type $X_m=\ti$ is equal to the proportion of items of type $\ti$ in the remaining items in $\Mcct$ at time $t-1\,.$ 
The number of items of type $\ti$ in $\Mcct$ at time $t-1$ is at least $|\Rb^2(\ti)|- \sum_{i\in\Rb^2(\ti)}{\ident[\EE^2_i]}$. 

The number of items removed from $\Mcct$ by time $t$ is at least $\sum_{l=1}^{m-1}|\Rb(X_l)|$ (there could be additional items with types not in $\{X_n\}_{n=1}^{m-1}$ that were removed from $\Mcct$ due to misclassification). Hence, the total number of remaining items in $\Mcct$ by time $t$ is at most $M - \sum_{n=1}^{m-1}|\Rb^2(X_n)|$. Let $Z_{\EE^2}$ be the collection of indicator variables $Z_{\EE^2} = \{\ident{[\EE^2_i}]\}_{i\in\Mcct_0}$. 
For any $\ti\in [\qi]\setminus\{X_n\}_{n=1}^{m-1}$:
\begin{align*}
\Pr\Big[X_m=\ti  \,\big|\,
\big\{X_n\big\}_{n=1}^{m-1}, \big\{\Rb^2(j')\big\}_{j'=1}^{\qi}, 
\Lc_u,  Z_{\EE^2}\Big]
 \geq \frac{|\Rb^2(\ti)|- \sum_{i\in\Rb^2(\ti)}{\ident[\EE^2_i]}}{M - \sum_{n=1}^{m-1}|\Rb^2(X_n)|} \,.
\end{align*}
On the event $\iuerr^c$ (defined in~\eqref{def:IU-iuerr}), 
$\sum_{i\in\Rb^2(\ti)}{\ident[\EE^2_i]}\leq |\Rb^2(\ti)|/10$ 
for any $\ti\in [\qi]\setminus\{X_n\}_{n=1}^{m-1}$. 
Therefore, \begin{align}\label{eq:IU-luKj1}
\Pr\Big[X_m=\ti\,\big|\,
\big\{X_n\big\}_{n=1}^{m-1},  \big\{\Rb^2(j')\big\}_{j'=1}^{\qi}, 
\Lc_u, \iuerr^c\Big]
\geq \frac{9}{10}
\,\frac{|\Rb^2(\ti)|}{M - \sum_{n=1}^{m-1}|\Rb^2(X_n)|}\,.
\end{align}

\noindent\emph{Upper bound for~\eqref{eq:IU-nextrep}.} 
This time we lower bound the number of items remaining in $\Mcct$ by the number of items in types $[\qi]\setminus\{X_n\}_{n=1}^{m-1}$ minus the number of possible mistakes which removed some items from them. This gives
\begin{align*}
\Pr\Big[X_m=\ti  
\,\big|\,
\big\{X_n\big\}_{n=1}^{m-1},  \big\{\Rb^2(j')\big\}_{j'=1}^{\qi}, 
\Lc_u,  Z_{\EE^2} \Big]
 \leq 
 \frac{|\Rb^2(\ti)| }
 {M - 
 \sum\limits_{n=1}^{m-1}|\Rb^2(X_n)| - 
 \sum\limits_{\substack{i\in\Mcct_0\setminus  \cup_{n=1}^{m-1}\Rb^2(X_n)}}{\ident[\EE^2_i]}} \,.
\end{align*}
By definition of $\iuerr$ (in~\eqref{def:IU-iuerr}), we have
\begin{align}\label{eq:IU-luKj2}
\Pr\Big[X_m=\ti\,\big|\,&
\big\{X_n\big\}_{n=1}^{m-1},  \big\{\Rb^2(j')\big\}_{j'=1}^{\qi}, 
 \Lc_u, \iuerr^c\Big]
\leq  \frac{10}{9}\frac{|\Rb^2(\ti)|}{M - \sum_{n=1}^{m-1}|\Rb^2(X_n)|}\,.
\end{align}

\noindent \emph{Combining lower and upper bounds.}
The expected value of $|\Rb^2(\ti)|$ conditional on variables $\{X_n\}_{n=1}^{m-1}$, $\{\Rb^2(X_n)\}_{n=1}^{m-1}$, and $\Lc_u$ is invariant to choice of $\ti\in [\qi]\setminus\{X_n\}_{n=1}^{m-1}$. So, the conditional expectation
$$C = 
\Exp{\frac{|\Rb^2(\ti)|}
{M - \sum_{n=1}^{m-1}|\Rb^2(X_n)|}  \bigg|
\big\{X_n\big\}_{n=1}^{m-1}, \big\{\Rb^2(X_n)\big\}_{n=1}^{m-1}, 
\Lc_u,\iuerr^c},$$
 is independent of $\ti$ for $\ti\in [\qi]\setminus\{X_n\}_{n=1}^{m-1}$. 
Note that in the above display, the conditional expectation is with respect to $\Rb^2(X_n)$ for $n=1,\cdots, m-1$.
Hence, using tower property of expectation on~\eqref{eq:IU-luKj1} and~\eqref{eq:IU-luKj2} to remove the conditioning on 
$\{\Rb(j)\}_{j\notin \{X_n\}_{n=1}^{m-1}}$ 
along with the definition of $C$ above gives
\[\frac{9}{10}C
\leq 
\Pr\Big[X_m=\ti\, \Big|\,
\big\{X_n\big\}_{n=1}^{m-1}, \big\{\Rb^2(X_n)\big\}_{n=1}^{m-1}, 
\Lc_u,\iuerr^c\Big] 
\leq
\frac{10}{9}C \,.\]
Since there are $\qi - (m-1)$ types $\ti\in [\qi]\setminus\{X_n\}_{n=1}^{m-1}$, summing over $\ti$ in the second inequality of the last display gives  $C\geq\frac{9}{10}\frac{1}{\qi-(m-1)}$. Plugging this into the first inequality of the last display gives for all $\ti\in [\qi]\setminus\{X_n\}_{n=1}^{m-1}$,
	\begin{align*}
\Pr\Big[X_m=\ti\, \Big|\,
\big\{X_n\big\}_{n=1}^{m-1}, \big\{\Rb^2(X_n)\big\}_{n=1}^{m-1}, 
\Lc_u,\iuerr^c\Big] 
&>\frac{2}{5}
\frac{1}{\qi- m+1}\,.
\end{align*}
Conditional on $\Lc_u$ such that $|\Lc_u|\geq \qi/4,$ for $s\leq m\leq\ell\leq \qi$ and $s\leq\ell/20$ we have
\begin{align*}
\Pr\bigg[X_m\in\Lc_u 
\,\bigg|\,
\{X_n\}_{n=1}^{m-1}, 
\sum_{n=1}^{m-1}\ident[X_{n}\in\Lc_u] < s, \,
\Lc_u, \,
|\Lc_u|\geq \tfrac{\qi}{4}, \,
\iuerr^c\bigg] 
\geq \frac{2}{5}\frac{\frac{\qi}{4}-\min\{m,\ell/20\}}{\qi-m+1}\geq \frac{1}{12}\,.
\end{align*}
So, on the event $\sum_{n=1}^{\ell}\ident[X_{n}\in\Lc_u] \leq\frac{\ell}{20} $, 
the random variable $\sum_{n=1}^{\ell}\ident[X_{n}\in\Lc_u]$  conditional on $\Lc_u$ and events $ |\Lc_u|\geq \nicefrac{\qi}{4}$ and $ \iuerr^c$,  stochastically dominates a Binomial random variable with mean $\ell/12$. Hence, by a Chernoff bound,
\[
\Pr\bigg[
\sum_{n=1}^{\ell}\ident[X_{n}\in\Lc_u] \leq \frac{\ell}{20} 
\,\,\Big|\, \,
\Lc_u, |\Lc_u|\geq \tfrac{\qi}{4}, \iuerr^c\,\bigg] 
\leq  
\exp(-\ell/13)\,.\qedhere\] 
\end{proof}
\section{Item structure only: lower bound}
\label{s:item_lower}
In this section we prove a lower bound on the regret of any online recommendation system in the regime with item structure only where $\qu=N$ as described in Definition~\ref{d:item}.  
Throughout this section, we will assume $N>32$.
\begin{theorem}\label{t:item-itemL}
Let $\thril=\lfloor.8\log\qi - 4\log\log N\rfloor$ and $\eta=1/\log N$. In the item structure model with $N$ users and $\qi$ item types, any recommendation algorithm must incur regret
	\[\reg(T)
	\geq
	 \frac{1-3\eta}{2}
\max\left\{
1, \,\,
\frac{1}{2}\,  Z(T), 
\, \, \frac{\thr}{N}\,T
\right\}
-\frac TN \,,\]
where we define
\begin{align}\label{eq:IL-defineZT}
Z(T) 
=
\begin{cases}
\frac{T}{2}, 
\quad &\text{if }  
T< \frac{2\sqrt{\qi}}{3N}
\\[8pt]
\frac{T}{5\log \qi}\big[\log(8\qi\log\qi) - \log(NT)\big], 
\quad &\text{if }  
 \frac{2\sqrt{\qi}}{3N}\leq T< \frac{4\qi \log \qi}{N}
\\[8pt]
\frac{T}{8\log \qi},
\quad &\text{if } 
\frac{4\qi \log \qi}{N}\leq T < \frac{16\qi(\log \qi)^2}{N}
\\[8pt]
\frac{{1}}{2}\sqrt{\frac{T\qi}{N}},  
\quad &\text{if } 
\frac{16\qi(\log \qi)^2}{N}\leq T\,.
\\
\end{cases}	
\end{align}

\end{theorem}
The assumption $N>32$ implies $\eta<1/5$. 
Note that the function $Z(T)$ is continuous up to a multiplicative constant  factor\footnote{We define  a function $Z(T)$ to be continuous up to a multiplicative constant factor $C$ if for every $T_0$, \[C^{-1}  \lim_{T\to T_0^-} Z(T)< \lim_{T\to T_0^+} Z(T) <C \lim_{T\to T_0^-} Z(T) .\] }.

To get the simplified version in Section~\ref{s:results}, we bounded the value of $Z(T)$ in the second (in which $ \frac{2\sqrt{\qi}}{3N}\leq T< \frac{4\qi \log \qi}{N}$) by  $\frac{T}{8\log \qi}$.
Also, the assumption $\qi>25\,(\log N)^5$ guarantees that $\frac{1-3\eta}{2}  \,T \thril \,/N> c T/N$ for a constant $c$ and we can remove the last term in the lower bound of regret.

\subsection{Proof strategy}

We call a recommendation $a_{u,t}$ to user $u$ at time $t$ a \textit{bad} (or uncertain) recommendation when the probability of $L_{u,a_{u,t}}=+1$ given the history is close to (or smaller than) $1/2$. 
Conversely, recommendations for which the probability of $L_{u,a_{u,t}}=+1$ is close to one (much greater than $1/2$) are considered \textit{good} recommendations. Good and bad refers only to the confidence that the recommendation is liked given the history at the moment the recommendation is made: a good recommendation is not always liked and a bad recommendation is not always disliked.

We identify two scenarios in which recommendations are necessarily bad in the item structure only model introduced in Definition~\ref{d:item}: 
(i) A good estimate of item types is necessary in order to make meaningful recommendations. 
To determine whether or not two items are of the same type, approximately $\log \qi$  users should rate both of them. 
To formalize this, similar to the lower bound for the model with user structure only in Section~\ref{sec:UL}, we use the concept of $(\thril,\eta)$-row regularity (Definition~\ref{def:row-reg}).
 Lemma~\ref{cl:i-l-term1} shows that for a $(\thril,\eta)$-row regular preference matrix, items with fewer than $\thril$ ratings are liked by any user with probability roughly half, even if the preference matrix is known. 
 (ii) Even when we know the item types (i.e., clustering of items), if a given user $u$
has not rated any item with the same type as item $i$ before, the probability that user $u$ likes item $i$ is $1/2$. 
 Lemma~\ref{cl:IL-first-recomm} shows  this property. 

In Lemma~\ref{l:IL-reg-stat-bd} we bound regret in terms of the number of good recommendations and in Lemma~\ref{l:IL-reg-comb-bd} we upper bound the number of good recommendations, which entails a counting argument. 
The theorem follows immediately from Lemmas \ref{l:IL-reg-stat-bd} and \ref{l:IL-reg-comb-bd}.
	
	\subsection{Proof of Theorem~\ref{t:item-itemL}}
	\label{ss:ItemLowerProof}

\begin{definition}[Row-regularity]\label{def:row-reg}
The matrix $A\in\{-1,+1\}^{n\times m}$ is said to be $(\thril,\eta)$-\textit{row regular} if its transpose is $(\thril,\eta)$-column regular (Definition~\ref{def:ul-reg}). We write $A^\top \in \rgl_{\thril,\eta}$,  where $\rgl_{\thril,\eta}$ is the set of $(\thril,\eta)$-column regular matrices.
\end{definition}

The following lemma is an immediate corollary of Lemma~\ref{l:ul-regularity}.
\begin{lemma}\label{l:IL-reg-prob}
	Let matrix $A\in\{-1,+1\}^{n\times m}$ have i.i.d. $\mathrm{Bern}(1/2)$ entries. If $\eta<1$, then  $A$ is $(\thril,\eta)$-row regular with probability at least $$1-2(2n)^\thril  \exp\left( - \frac{\eta^2}{3}\frac{m}{2^{\thril}}\right)\,.$$
\end{lemma}

Throughout this section we will fix  
\begin{equation}\label{eq:def-thril}
\thril=\lfloor0.8 \log\qi - 4\log\log N\rfloor \quad \text{and} \quad \eta=\frac{1}{\log N}\,.
\end{equation}
Applying Lemma \ref{l:IL-reg-prob} with these choices of $\thril$ and $\eta$, we obtain that
so long as $N>32$
 the preference matrix $\Xi$ is $(\thril,\eta)$ row-regular (\textit{i.e.,} $\Xi^\top\in\rgl_{\thril,\eta}$) with probability
\begin{equation}\label{e:rowReg}
\Pr[\Xi^\top\in\rgl_{\thril,\eta}]\geq 1-\frac{1}{N}\,.
\end{equation}

	 Let $c_{i}^{t}$ be the number of times item $i$ has been rated by any user by the end of time-step $t-1,$ 
	 \begin{equation}\label{e:c}
	 	c_{i}^{t}:= \sum_{u=1}^N \sum_{s=1}^{t-1} \ident[a_{u,s}=i].
	 \end{equation}
The following lemma shows that if $c_i^t$ is small and the preference matrix is row-regular, then the outcome of recommending item $i$ to any user at time $t$ is uncertain.

\begin{lemma}\label{cl:i-l-term1}
Let $\thril$ and $\eta$ be as in Equation~\eqref{eq:def-thril}. For any user $u\in[N]$ and item $i\in \mathbb{N}\,,$
$$\Pr\big[L_{u,i}=+1 \, \big| \,a_{u,t}=i, c_{i}^{t} < \thril, \Xi^\top\in\rgl_{\thril,\eta}\big]
\leq 
\frac{1+3\eta}{2}\,.$$
\end{lemma}

 The next lemma shows that if a user $u$ has not rated any item with the same type as item $i$ before, the probability that $u$ likes item $i$ is 1/2.
	Let $\BB_{u,\tau_I(i)}^t$ denote the event that user $u$ has rated an item of the same type as item $i$ by time $t-1$, \textit{i.e.},
\begin{equation}\label{e:BBevent}
\BBil{u}{i}{t} = \{{\exists i'\in\mathbb{N}: a_{u,s}=i'} \text{ for some } s< t \text{ and } \tau_I(i) = \tau_I(i')\}\,.
\end{equation}

\begin{lemma}\label{cl:IL-first-recomm}
We have that
$\Pr[L_{u,i}=-1 \,|\, a_{u,t}=i, (\BBil{u}{i}{t})^c, c_{i}^{t}\geq \thril] = \frac{1}{2}$.
\end{lemma}

 Lemmas~\ref{cl:i-l-term1} and~\ref{cl:IL-first-recomm} (proved in Section~\ref{sec:proofLemmas-Item-lower}) identify scenarios in which recommendations are bad.
In the complementary scenario, recommendations are not necessarily bad (and may be good).  We denote the number of such recommendations by
\begin{equation}\label{eq:IL-defE}
\good(T)=\sum_{t\in[T]\atop u\in[N]}  \ident\big[ c_{a_{u,t}}^{t}\geq\thril,\BBil{u}{a_{u,t}}{t}\big]\,.
\end{equation}
The proof of the following lemma uses Lemmas~\ref{cl:i-l-term1} and~\ref{cl:IL-first-recomm} to  lower bound  regret in terms of expectation of $\good(T)$.

\begin{lemma} \label{l:IL-reg-stat-bd}
For $\eta$ and $\thril$ defined in Equation~\eqref{eq:def-thril},
\begin{align*}
N\, &\reg(T) \geq  \frac{1-3\eta}{2}\,\Big(TN-\Exp{ \good(T)}\Big) - T\,.
\end{align*}
\end{lemma}
\begin{proof}
We partition the liked recommendations based on $c_{a_{u,t}}^{t}$, $ \BBil{u}{a_{u,t}}{t}$, and row regularity of $\Xi$:
 \begin{align*}
N \big(&T-\reg(T)\big) = \,\,
\sum_{t\in[T]\atop u\in[N]} \Pr\big[L_{u,a_{u,t}}=+1\big] \notag\\
& =  {\sum_{t\in[T]\atop u\in[N]}  \Pr\big[L_{u,a_{u,t}}=+1, c_{a_{u,t}}^{t}<\thril,\Xi^\top\notin\rgl_{\thril,\eta}\big]}\notag
\quad+ {\sum_{t\in[T]\atop u\in[N]}  \Pr\big[L_{u,a_{u,t}}=+1,  c_{a_{u,t}}^{t}<\thril,\Xi^\top\in\rgl_{\thril,\eta}\big]}\notag\\
&  \quad+ {\sum_{t\in[T]\atop u\in[N]} \Pr\big[L_{u,a_{u,t}}=+1, c_{a_{u,t}}^{t}\geq\thril, \big(\BBil{u}{a_{u,t}}{t}\big)^c \big]}\notag
 \quad+ {\sum_{t\in[T]\atop u\in[N]} \Pr\big[L_{u,a_{u,t}}=+1,  c_{a_{u,t}}^{t}\geq\thril, \BBil{u}{a_{u,t}}{t} \big]}\notag\\
 &=:  \Asf{1} + \Asf{2} + \Asf{3} + \Asf{4}\,. \label{eq:IL-reg-decomp}
\end{align*}

The proof is obtained by plugging in the four bounds below and simplifying.

\paragraph{Bounding $\Asf{1}$.}
Plugging in the probability of $\Xi$ being row regular from~\eqref{e:rowReg} gives
\begin{align*}
\Asf{1} & = {\sum_{{t\in[T]\atop u\in[N]}} \Pr[L_{u,a_{u,t}}= +1, c_{a_{u,t}}^t<\thril,\Xi^\top\notin\rgl_{\thril,\eta}]}  \notag 
\leq NT\cdot \Pr[\Xi^\top\notin\rgl_{\thril,\eta}] \leq T\,.
\end{align*}

\paragraph{Bounding $\Asf{2}$.}
Multiplying the statement of Lemma \ref{cl:i-l-term1} by $\Pr[a_{u,t}=i, c_i^t<\thril, \Xi^T\in\rgl_{\thril,\eta}]$ and summing over $i$ gives
\begin{align*}
\Asf{2} &= \sum_{t\in[T]\atop u\in[N]} \sum_{i\geq 1} \Pr[L_{u,i}=+1, a_{u,t}=i, c_{i}^{t}<\thril,\Xi^\top\in\rgl_{\thril,\eta}] \notag
\\&
\leq 
\frac{1+3\eta}{2}\sum_{t\in[T]\atop u\in[N]} \Pr[ c_{a_{u,t}}^{t}<\thril,\Xi^\top\in\rgl_{\thril,\eta}]
\leq 
\frac{1+3\eta}{2}\sum_{t\in[T]\atop u\in[N]} \Pr[ c_{a_{u,t}}^{t}<\thril]\,.
\end{align*}

\paragraph{Bounding $\Asf{3}$.}
Lemma \ref{cl:IL-first-recomm} gives
\begin{align*}
&\frac{\Pr[L_{u,i}=+1, a_{u,t}=i, c_{i}^{t}\geq\thril, (\BBil{u}{i}{t})^c]} {\Pr[ a_{u,t}=i, c_{i}^{t}\geq\thril, (\BBil{u}{i}{t})^c]} 
 = \frac{1}{2}\,.
\end{align*}
Multiplying this by $\Pr[a_{u,t}=i, c_{i}^{t}\geq \thril, (\BBil{u}{i}{t})^c]$ and summing over $i$ gives 
\begin{equation*}
\Asf{3} = \frac{1}{2} \sum_{u\in[N], t\in[T]}\Pr[ c_{a_{u,t},t}\geq \thril,(\BBil{u}{a_{u,t}}{t})^c]\,. 
\end{equation*}

\paragraph{Bounding $\Asf{4}$.}
We bound by one the probability that a good recommendation is liked to obtain
\begin{align*}
\Asf{4} &= \sum_{u\in[N], t\in[T]}\Pr[L_{u,a_{u,t}}= +1, c_{a_{u,t}}^t\geq\thril,\BBil{u}{a_{u,t}}{t}]
\leq \sum_{u\in[N], t\in[T]}\Pr[ c_{a_{u,t}}^t\geq\thril,\BBil{u}{a_{u,t}}{t}] \notag
 = \Ex[\,\good(T)\,]\,. &\qedhere
\end{align*}
\end{proof}

Next, we upper bound the expected number of \textit{good} recommendations made by the algorithm in terms of parameters of the model.

\begin{lemma}[Upper Bound for expected number of good recommendations] \label{l:IL-reg-comb-bd}
For any algorithm,
\begin{align*}
\Ex[\,\good(T)\,]
\leq TN - &	\max\left\{
\frac{1}{2} N Z(T), 
 \,T \thr
, N \right\}\,.
\end{align*}
where $Z(T) $ is defined  in~\eqref{eq:IL-defineZT}.
\end{lemma}

We prove this lemma in the next subsection. 
Theorem~\ref{t:item-itemL} is an immediate consequence of Lemmas \ref{l:IL-reg-stat-bd} and \ref{l:IL-reg-comb-bd}.

\subsection{Proofs of lemmas}
\label{sec:proofLemmas-Item-lower}
\subsubsection{Proof of Lemma~\ref{cl:i-l-term1}}
	We show that if an item has been rated by fewer than $\thril$ users, its type is uncertain, because many item types are consistent with the history even if the preference matrix is known. For a row-regular preference matrix, uncertainty in the  type of an item makes it impossible to accurately predict user preferences for that  item.
	
Consider item $i$ at time $t$. Let $w=\{u\in[N]: a_{u,s} = i, s<t\}$ be the ordered tuple corresponding to the set of users that were recommended item $i$ up to time $t-1\,,$ and let $b=\{L_{u,i}\}_{u\in w}$ be the vector of feedback from users in $w$ about item $i\,.$ Note that $c_{i}^{t}<\thril$ implies $|w|<\thril\,.$ We re-introduce the notation from Definition~\ref{def:ul-reg}: if $M$ is the matrix obtained by concatenating the rows of $\Xi$ indexed by $w$, then 
	$K_{b,w}(\Xi^\top)$ is the set of columns of $M$ (corresponding to the item types) equal to $b$. This is the set of item types consistent with the ratings $b$ of users $w$ for item $i$.

	Conditional on $\Xi$,  $w$, and $b$, the type $\tau_I(i)$ of item $i$ at the end of time $t-1$ is uniformly distributed over the set of item types $K_{b,w}(\Xi^\top)$. This allows us to relate the posterior probability of $i$ being liked to row regularity of $\Xi$ as follows. 
	Let $b^+= [b\; 1] \in \{-1,+1\}^{|w|+1}$ be obtained from $b$ by appending $+1\,.$ For a given user $u\notin w\,,$ we have $L_{u,i}=+1$ precisely when $\tau_I(i)\in K_{b^+,\{w,u\}}(\Xi^\top)\,,$ which in words reads ``item $i$ is among those types that are consistent with the ratings of $i$ up to time $t-1$ and have preference vector with `+1' for user u''.
	It follows that for any preference matrix $\Xi$ and any user $u$ which has not rated $i$ up to time $t-1\,,$
	\begin{align}
	\Pr\big[L_{u,i}=+1\,\big|\,\H_{t-1},\Xi\big] 
	& =
	 \Pr\big[\tau_{I}(i)\in K_{b^+,\{w,u\}}(\Xi^\top)\,\big|\,\H_{t-1},\Xi\big]
	 = \frac{k_{b^+,\{w,u\}}(\Xi^\top)}{k_{b,w}(\Xi^\top)}\,.
	\label{eq:IL-posterior}
	\end{align}
	The second equality is due to: (i) $w$ and $b$ are functions of the history up to time $t-1$, $\H_{t-1}$; (ii) for fixed $w$ and $b$, the set $K_{b,w}(\Xi^\top)$ is determined by $\Xi^\top\,$; (ii) $\tau_I(i)$ is uniformly distributed on $K_{b,w}(\Xi^\top)$ given $\H_{t-1}$.
	
	Recall that $\Xi^\top\in\rgl_{\thril,\eta}$ if the preference matrix $\Xi$ is $(\thril,\eta)$-row regular. We have 
	\begin{align*}
	\Pr \big[L_{u,i}=+1,
	&a_{u,t}=i \, |\,  c_{i}^{t} \leq \thril, \Xi^{\top}\in \rgl_{\thril,\eta} \big] 
\,\,	\overset{\text{(a)}}{=}  \,\,
	 \Exp{\Pr[L_{u,i}=+1\,|\,\H_{t-1}, \Xi]\, \Pr[a_{u,t}=i|\H_{t-1}\,]
			 \, \big |\,
			 c_{i}^{t} \leq \thril, \Xi^{\top}\in\rgl_{\thril,\eta} }
			 \\ & 
			 \overset{\text{(b)}}{=}   
	\Exp{
			\frac{k_{b^+,\{w,u\}}(\Xi^\top)}{k_{b,w}(\Xi^\top)}\,
		\,	\ident[u\notin w]\,\,
			\Pr[a_{u,t}=i|\H_{t-1}]
			\,\, \Big |\,\,
			c_{i}^{t} \leq \thril, \Xi^{\top}\in\rgl_{\thril,\eta} }\\
	& \overset{\text{(c)}}{\leq} 
	\frac{(1+\eta) \frac{\qi}{2^{|w|+1}}}{(1-\eta) \frac{\qi}{2^{|w|}}} 
		\Pr\Big[a_{u,t}=i|c_{i}^{t} \leq \thril, \Xi^{\top}\in\rgl_{\thril,\eta} \Big]
		 \overset{\text{(d)}}{\leq}  	
	\frac{1}{2}(1+3\eta)
	\Pr\Big[a_{u,t}=i|c_{i}^{t} \leq \thril, \Xi^{\top}\in\rgl_{\thril,\eta} \Big]\,,
	\end{align*}
and this proves the lemma. It remain to justify the steps above.
	(a) follows since conditional on $\H_{t-1}$, the random variable $a_{u,t}$ is independent of all other random variables. (b) uses \eqref{eq:IL-posterior} and the fact that $\Pr[a_{u,t}=i|\H_{t-1}]$ is nonzero only if $u$ has not rated item $i$, so we may add $\ident[u\notin w]$. (c)~is justified as follows: if $\Xi^{\top}\in\rgl_{\thril,\eta}$, then by Claim \ref{cl:ul-smalltregular}, $\Xi^{\top}\in\rgl_{\thril-1,\eta}.$ By Definition \ref{def:ul-reg}, this means that $k_{b,w}(\Xi^{\top})\geq (1-\eta) \qi/ 2^{|w|}$ and $k_{b^+,\{w,u\}}(\Xi^{\top})\leq (1+\eta) \qi/ 2^{|w|+1}$. (d) If $N>32$, then $\eta$ in~\eqref{eq:def-thril} is less than $1/\log 32$ and $(1+\eta)/(1-\eta)\leq 1+3\eta$.

\subsubsection{Proof of Lemma~\ref{cl:IL-first-recomm}}
At a high level, we make two observations: (i) if user $u$ has not rated any item with type $\tau_I(i)$ before, the feedback in the history $\H_{t-1}$ is independent of the value of $\xi_{u,\tau_I(i)}$ given all other elements of matrix $\Xi$ and the item types; (ii) the types of items (function $\tau_I(\cdot)$) are independent of matrix $\Xi$, and the elements of $\Xi$ are independent. Hence, conditional on $(\BBil{u}{i}{t})^c$, the posterior distribution at time $t$ of $L_{u,i}$ is uniform on $\{-1,+1\}$. 

Concretely, we may think of revealing the entries of $\Xi$ on a ``need-to-know" basis. Conditional on $(\BBil{u}{i}{t})^c$, entry $\xi_{\tau_U(u),\tau_I(i)}$ has not yet been touched, so
	\begin{align*}
	\Pr\big[L_{u,i}=&+1 \,\big|\,a_{u,t}=i, \tau_I(i)=j, (\BBil{u}{i}{t})^c, c_{i}^{t}\geq \thril, \tau_U(u)\big]\\
	& =\Pr\big[\xi_{\tau_U(u),\ti}=+1 \,\big|\,a_{u,t}=i, \tau_I(i)=j, (\BBil{u}{i}{t})^c, c_{i}^{t}\geq \thril, \tau_U(u)\big]
	 =\frac{1}{2}\,.
	\end{align*}
	The lemma now follows by the tower property.
\subsubsection{Proof of  Lemma~\ref{l:IL-reg-comb-bd} }
 Lemma~\ref{l:IL-reg-comb-bd} upper bounds the expected number of good recommendations. To prepare for the proof of this lemma we introduce some notation. Recall that $c_{i}^{t}$ (defined in~\eqref{e:c}) is the number of users who have rated item $i$ before time $t$. 
	Let $\Fc_t=\{i: c_i^t\geq \thr\}$ be the set of items with at least $\thr$ ratings before time $t$ and $f_t := |\Fc_t|$ their number,
	$$f_t= \sum_{i\in\mathbb{N}} \ident[c_{i}^{t}\geq \thril].$$ 
	Let $\Gc_t=\{i: 0<c_i^t< \thr\}$ be the set of items that have been rated by at least one and fewer than $\thril$ users by time $t$ and $g_t := |\Gc_t|$ their number,
	$$g_t= \sum_{i\in\mathbb{N}} \ident[0 < c_{i}^{t}< \thril].$$

	In the following claim, we bound the number of good recommendations up to time $T$ in terms of $f_T$ and $g_T$.
	
	\begin{claim} \label{cl:IL-upperbdForA1}
		The number of good recommendations $\good(T)$, defined in~\eqref{eq:IL-defE}, satisfies
$
		\good(T)\leq  TN - g_T - f_T \thril
$.
	\end{claim}
	\begin{proof}
		
		Any $i\in\Fc_T$ is recommended to at least $\thril$ users by the end of time $T\,.$ So, for any $i\in\Fc_T\,,$ there are $\thril$ recommendations in which $i$ has been rated fewer than $\thril$ previous times:
		\begin{align}\label{eq:boundbadinFT}
		\sum_{t\in[T]\atop u\in[N]} \ident[ c_{i}^{t}< \thril , a_{u,t}=i] = \thr, 
		\quad\quad \quad
		\text{for } i\in\Fc_T\,.
		\end{align}
		Any $i\in\Gc_T$ has been recommended at least once, so for these items
		\begin{align*}
		\sum_{t\in[T]\atop u\in[N]} \ident[ c_{i}^{t}< \thril , a_{u,t}=i] \geq 1,
		\quad\quad \quad
		\text{for } i\in\Gc_T\,.
		\end{align*}
		All recommended items are either in $\Fc_T$ or $\Gc_T$. So, the total number of recommendations satisfies
		\begin{align*}
		NT \,\,& = \sum_{t\in[T]\atop u\in[N]} \sum_{i\geq 1}\ident[a_{u,t}=i] 
		\,\, = \sum_{t\in[T]\atop u\in[N]}  \sum_{i\geq 1}\ident[c_i^t\geq \thril, a_{u,t}=i] + \sum_{t=1}^T \sum_{u=1}^N \sum_{i\geq 1}\ident[c_i^t< \thril, a_{u,t}=i] \\
		& \geq \good(T) + \sum_{t\in[T]\atop u\in[N]}\sum_{i\in \Fc_T} \ident[ c_{i}^{t}< \thril , a_{u,t}=i]
		+ \sum_{t\in[T]\atop u\in[N]}\sum_{i\in \Gc_T} \ident[ c_{i}^{t}< \thril , a_{u,t}=i]\\
		& \geq \good(T) + f_T \thril + g_T\,.
		\end{align*}
		where we used $\good(T)\leq \sum_{t\in[T]\atop u\in[N]}  \sum_{i\geq 1  }\ident[ c_{i}^{t}\geq \thril , a_{u,t}=i]\,.$
	\end{proof}

The rest of the section contains the proof of Lemma~\ref{l:IL-reg-comb-bd}.
\begin{proof}[Proof of Lemma~\ref{l:IL-reg-comb-bd}]
The lemma consists of three bounds on $\good(T)$. 

\noindent \textbf{Proof of inequality $\Exp{\good(T)}\leq TN -T\thril$}.  This bound is the easiest, so we start with this.

\begin{align*}
\good(T)
&\overset{(a)}\leq 
\sum_{t\in[T]\atop u\in[N]}  \sum_{i\geq 1  }\ident[ c_{i}^{t}\geq \thril , a_{u,t}=i]
\,\,{=} \,\,
\sum_{t\in[T]\atop u\in[N]}  \sum_{i: c_i^T\geq \thril  }\ident[ c_{i}^{t}\geq \thril , a_{u,t}=i]
\\
&\overset{(b)}{\leq} 
\sum_{i\in \Fc_T} \sum_{t\in[T]\atop u\in[N]}  \ident[ a_{u,t}=i] - \ident[ c_{i}^{t}< \thril , a_{u,t}=i]
\,\,\overset{(c)}{\leq}\,\,
f_T (N -\thril)\,.
\end{align*}	
(a) uses the definition of $\good(T)$. For any item $i$, the sequence $c_i^t$ is nondecreasing in $t$. This shows that if for an item $c_i^t>\thril$ at $t\leq T$, then item $i$ satisfies $c_i^T\geq \thril$. (b) is derived by changing the order of summations, the definition of $\Fc_T$ and equality $ \ident[ c_{i}^{t}\geq \thril , a_{u,t}=i]= \ident[ a_{u,t}=i] - \ident[ c_{i}^{t}< \thril , a_{u,t}=i]$.
(c)	Because each item is recommended at most once to each user, each item (and specifically items in $\Fc_T$) are recommended at most $N$ times. This gives for $i\in \Fc_T$, $\sum_{t\in[T]\atop u\in[N]} \ident[ a_{u,t}=i]\leq N$ as the bound for the first term in (c). The second term in (c) is bounded by Equation~\eqref{eq:boundbadinFT}.

	Plugging this into the statement of Claim \ref{cl:IL-upperbdForA1} gives for any values of $f_T$ and $g_T$,
	\begin{equation}\label{eq:IL-bd1}
	\good(T)\leq T(N-\thril)\,.
	\end{equation}
	
	We now prove the other two bounds on $\good(T)$, this time in terms of the number of item types each user has rated by time $T$.

		Let $\Gamma_{u}^T$ be the set of item types that are recommended to user $u$ up to time $T$,
		\begin{equation}\label{e:Gamma}
		\Gamma_u^T=\bigg\{\ti\in[\qi]: \sum_{t\in[T]\atop i:\tau_I(i)=\ti}\ident[a_{u,t}=i]\neq 0\bigg\}\,,\quad \text{and}\quad \gamma_u^T = |\Gamma_u^T|\,.
		\end{equation}


	\begin{claim}\label{cl:IL-upperbdForA2}
		$\good(T)\leq N\big[T-\min_{u}\gamma_u^T \big]$.
	\end{claim}
	
	\begin{proof}
		For user $u$, the number of times an item type is rated for the first time by this user is equal to the number of item types rated by user $u$. 
		Hence, 
		$\sum_{t\in [T]} \ident[(\BBil{u}{a_{u,t}}{t})^c] = \gamma_u^T$
		and
		$\sum_{t\in [T]} \ident[\BBil{u}{a_{u,t}}{t}] = T - \gamma_u^T$.
		Summing over $u$ and using  $\good(T)\leq \sum_{t\in [T]\atop u\in [N]} \ident[ \BBil{u}{a_{u,t}}{t}]$ proves the claim.
	\end{proof}

\noindent \textbf{Proof of inequality $\Exp{\good(T)}\leq NT-N$}. This follows from Claim~\ref{cl:IL-upperbdForA2} since $\gamma_u^T\geq 1$ for all $u\in [N]$.  \\

\noindent \textbf{Proof of inequality 	$\Exp{\good(T)}\leq  NT -\frac{1}{2}NZ(T)$.} This last inequality is more involved than the others. Let $r_{\ti}^t$ be the number of items with type $j$ that have been recommended (to any user) by time~$t$:
		\begin{equation}\label{eq:defrjT}
		r_{j}^t=\big|\,\{i: \tau_I(i)=j, a_{u,s}=i \text{ for some } u\in[N] \text{ and } s\in [t]\}\, \big|.
		\end{equation}
	By time $T$, $T$ items with types in $\Gamma_u^T$ (defined in \eqref{e:Gamma}) are recommended to user $u$. Hence, $T
	\leq  \sum_{\ti\in\Gamma_{u}^T} r_{\ti}^T \leq \gamma_u^T \max_{\ti} r_{\ti}^T\,.$
	This implies that
	\begin{equation}\label{eq:il-gammau-rmax}
	\min_{u\in[N]}\gamma_u^T \geq \frac{T}{\max_{j}r_{j}^T}\,.
	\end{equation}
	We get a lower bound on $\min_{u}\,\gamma_u^T$ via an upper bound on $\max_{j}r_{j}^T$, which is in turn obtained via martingale concentration bounds for each $r_{j}^T$. This is essentially just a question of bounding the fullest bin in a balls and bins scenario, with the added complication that the time-steps in which balls are thrown is random. Thus the number of balls (i.e., $f_T + g_T$) is random, and the decision to throw a ball at a given time may depend on the configuration of balls in bins.

\begin{claim}\label{cl:rjTproof}
Let $\delta>0$. Define $k_0$, $k_1, k_2$ and the function $\Theta_{\qi}(k)$ as follows:
\begin{align}
\Theta_{\qi}(k) := \begin{cases} 
2\,,
 \,\quad&
 \text{if }\, k< k_0:=\sqrt{\qi}/3\\
  3\, \frac{\log\qi}{\log\qi - \log k}\,,
 \,\quad&
 \text{if }\, k_0\leq k< k_1:=\qi/2\\
8\log \qi\,, 
\,\quad&
\text{if }\, k_1\leq k< k_2:=2\qi \log \qi
\\
{4k}/{\qi}\,, 
\,\quad&
\text{if }\, k_2\leq k\,. \end{cases}
\end{align}
For $r_j^T$ defined in~\eqref{eq:defrjT}, $\Pr\left[\, \max_{\ti\in[\qi]}r_j^T 
\geq 
\Theta_{\qi}(f_T+g_T)\right]
	 \leq 
	 1/2$.
\end{claim}
\begin{proof}
First, we define a useful martinagle. 
Let $r^t =(r_1^t,\dots, r_{\qi}^t)$ where $r_{\ti}^t$ is defined in~\eqref{eq:defrjT}. Note that $f_t+g_t=\sum_{\ti} r_{\ti}^t$ is the total number of recommended items at the end of time $t$. Any new item has type uniformly distributed on $[\qi]$; as a consequence, the sequence $r_{j}^t - (f_t+g_t)/\qi$ is a martingale with respect to filtration $\Fc_t = \sigma(r^0,r^1,\dots, r^t)$, because $f_t+ g_t$ is incremented whenever a new item is recommended and  each new item increases $r^t_j$ by one  with probability $1/\qi$.

It turns out to be easier to work with a different martingale that considers recommendations to each user separately, so that the item counts are incremented by at most one at each step. Consider the lexicographical ordering on pairs $(t,u)$, where $(s,v)\leq(t,u)$ if either $s< t$ or $s=t$ and $v\leq u$ (such that the recommendation to user $v$ at time $s$ occurred before that of user $u$ at time $t$). For $j\in [\qi]$, let
$$
r_{j}^{t,u} 
=
\Big|
\big\{i: \tau_I(i)=j, a_{v,s}=i \text{ for some } (s,v)\leq (t,u)\big\}
\Big|\,.
$$
Let $r^{t,u} = (r_1^{t,u} ,\dots, r_{\qi}^{t,u} )$ and define $\rho^{t,u} = \sum_j	r_{j}^{t,u}$ to be the total number of items recommended by $(t,u)$, \textit{e.g.}, $\rho^{T,N}=f_T+g_T$. 
	We now define a sequence of stopping times $Z_k \in \mathbb{N}\times [N]$, 
	$$
	Z_k = \min\big\{(t,u)> Z_{k-1}: \rho^{t,u} > \rho^{t,u-1}\big\}\,,
	$$
	where  $(t,0)$ is interpreted as $(t-1,N)$ and $Z_0=(0,N)$. $Z_k$ is the first $(t,u)$ such that a new item is recommended by the algorithm for the $k$-th time, so $\rho^{Z_k}=k$. The $Z_k$ are stopping times with respect to $(\rho^{t,u})$, and observe that $k^* = \max\{k: Z_k \leq (T,N)\} = \rho^{T,N} = f_T + g_T$ since $f_T+g_T$ is the total number of items recommended by the algorithm by the end of time $T$. Also, $\rho^{Z_{k^*}}=f_T+g_T$ and $r_j^{Z_{k^*}}=r_{\ti}^{(T,N)}$ for all $\ti\in[\qi]$.
	
Fix item type $j\in[\qi]$. The sequence $M^{t,u}_j=r_j^{t,u} - \rho^{t,u}/\qi$ is a martingale with respect to the filtration $\Fc^{t,u}=\sigma(r^{1,1},\dots r^{t,u})$ \protect\daggerfootnote{To see that, define the event $\mathcal{E}^{\mathrm{new}}_{u,t}=\{\text{the item } a_{u,t} \text{ has not been recommended before to anybody}\}$ where the order is based on the lexicographic order we define in the proof of Claim~\ref{cl:rjTproof}. Then,
\begin{align*}
\Ex\Big[r_j^{t,u} - \tfrac{\rho^{t,u}}\qi \, \big|\, \Fc^{t,u-1}\Big]
&-
\Big[r_j^{t,u-1} - \tfrac{\rho^{t,u-1}}\qi \, \Big]
\overset{(a)}{=} 
\Ex\Big[
\ident[\tau_I(a_{u,t})=j] - \frac 1\qi 
\, \big|\,
\Fc^{t,u-1},\, \mathcal{E}^{\mathrm{new}}_{u,t}\Big] \,\,
\Pr\Big[\mathcal{E}^{\mathrm{new}}_{u,t}
 \, \big|\,
\Fc^{t,u-1}\Big]\overset{(b)}{=}0\,,
\end{align*} 
(a) uses the fact that
 condition on event $\big(\mathcal{E}^{\mathrm{new}}_{u,t}\big)^c$,  we have $\rho^{t,u}=\rho^{t,u-1}$ and $r_j^{t,u}=r_j^{t,u-1}$.  
Equality (b) uses the assumption in the model which states that the prior distribution of type of  an item which has not been recommended before is uniform over $[\qi]$. Hence, $\Pr\Big[
\tau_I(a_{u,t})=j
\, \Big|\,
\Fc^{t,u-1}, \mathcal{E}^{\mathrm{new}}_{u,t}\Big] = 1/\qi$. }.  
It follows that $\widetilde M^{k}_j := M^{(T,N)\wedge Z_k}_j$ is martingale as well, this time with respect to $\widetilde\Fc^k:=\Fc^{(T,N)\wedge Z_k}$. Since $Z_{k^*}\leq (T,N)$, we have $\widetilde M^{k^*}_j=M^{Z_{k^*}}_j$. We will use this notation to prove statement of the claim in three different regimes.
First, we would like to apply martingale concentration (Lemma~\ref{l:martingaleBound}) to $\widetilde M^{k}_j$, and to this end observe that $\mathrm{Var}(\widetilde M^{k}_j|\widetilde\Fc^{k-1}) \leq 1/\qi$ and $|\widetilde M^{k}_j - \widetilde M^{k-1}|\leq 1$ almost surely. 	 
	
\paragraph{Step~1}	
	For any $k\geq k_2:= 2\qi \log \qi$, Lemma~\ref{l:martingaleBound} gives
	\begin{align*}
\Pr\Big[\widetilde M_{j}^{k} \geq  \frac{3k}{\qi}\Big]
 \leq 
 \exp\bigg(\frac{-9k^2/\qi^2}{2(k/\qi + k/\qi)}\bigg) 
 =
  \exp\Big(-\frac{2k}{\qi}\Big)\,.
\end{align*}

This gives
\begin{align}
\Pr\Big[\max_{j\in[\qi]}r_j^{Z_{k^*}} &\geq \Theta_{\qi}(k^*), k^*\geq k_2\Big]
\leq 
 \Pr\Big[\exists k\geq k_2 
 \text{ s.t. } 
 \max_{j\in[\qi]}r_j^{Z_{k}} \geq \tfrac{4k}{\qi}\Big]
 \nonumber
\\ 
&
 \overset{(a)}{\leq } 
\Pr\Big[ 
\exists k\geq k_2
\text{ s.t. }
\widetilde M_j^{k} 
\geq  \tfrac{3k}{\qi}\Big] 
 \overset{(b)}\leq  
\sum_{k\geq k_2} \exp\big(-\tfrac{2 k}{\qi}\big) 
= 
\frac{\exp\big(-\tfrac{2 k_2}{\qi}\big)}{1-\exp(-2/\qi)}
  \overset{(c)} \leq 
  \frac{1}{\qi^2}\,.
\label{eq:IL-mart-k-large}
\end{align}
where (a) uses $\rho^{Z_k}=k$. (b) uses a union bound and the inequality in the last display.  (c)  uses definition of $k_2$ and ${1-\exp(-2/\qi)}> \qi^{-2}$ (which is derived using $e^{-a}\leq 1-a+a^2/2$ and $\qi>1$). 
\paragraph{Step~2} For any $k_2> k$ we get 
\begin{align*}
	\Pr\Big[\widetilde M_j^{k} \geq  6\log \qi\Big]
	 \leq
	 \exp\bigg(\frac{-36\log^2\qi}{2(k/\qi +  2\log \qi)}\bigg) 
	 \leq
	\exp\bigg(\frac{-36\log\qi}{2k_2/(\qi\log\qi) +  4}\bigg) 
	\leq 
	\frac{1}{\qi^4}\,.
\end{align*}

This gives
\begin{align}
\Pr\Big[\max_{j\in[\qi]}r_j^{Z_{k^*}} &\geq \Theta_{\qi}(k^*), k^*<k_2\Big]
 \leq
 \Pr\Big[\exists k<k_2 \text{ s.t. } \max_{j\in[\qi]}r_j^{Z_{k}} \geq 8 \log \qi\Big]
 \nonumber
\\& 
\overset{(a)}{\leq} \, 
\Pr\Big[ 
\exists k< k_2\text{ s.t. }\widetilde M_j^{k} 
\geq  6\log \qi\Big] 
\overset{(b)}\leq 
\sum_{k< k_2} \frac{1}{\qi^4}
\leq 
\frac{k_2}{\qi^4}
\leq 
\frac{1}{\qi^2}\,.
\label{eq:IL-mart-k-medium}
\end{align}
(a) uses $\rho^{Z_k}=k$. (b) uses the inequality in the above display. 

\paragraph{Step~3}
This step, $ k< k_1 := \qi/2 $, corresponds to bounding the number of balls in the fullest bin when the number of balls, $k$, is sublinear in the number of bins, $\qi$ (since $k=\qi^{1-3\delta}$ with $\delta>\frac{1}{8  \log \qi}$).  We will show that in this regime, the number of balls in the fullest bin is bounded by $1/\delta$. 
For given  $ k< k_1$, define $\delta=\frac{1}{3}\,\frac{\log\qi - \log k}{\log\qi }$ (such that $k=\qi^{1-3\delta}$). Then,
\begin{align}
\Pr\Big[  \max_{j\in[\qi]} r^{Z_{k}}_j \geq  1/\delta\Big]
 \overset{(a)}\leq
\qi\Pr\Big[ r^{Z_{k}}_1 \geq  1/\delta\Big]
\overset{(b)}\leq
\qi {{k}\choose{1/\delta}} \frac{1}{\qi^{1/\delta}}
 \overset{(c)}\leq \frac{5}{\qi ^2} 
 \label{eq:IL-mart-k-K_1_1}
\end{align}
(a) is a union bound over $\ti\in[\qi]$. (b) uses the fact that $r_1^{Z_{k+1}}=r_1^{Z_{k}}+1$ with probability $1/\qi$ and $r_1^{Z_{k+1}}=r_1^{Z_{k}}$ with probability $1-1/\qi$ independently of $r_1^{Z_{k}}$.  (c) holds for every $\delta>\frac{1}{8\log \qi}$ ( which is due to $k<k_1$) using ${k \choose 1/\delta}\leq (k e \delta)^{1/\delta}$.

\begin{align}
	\Pr\Big[\max_{j\in[\qi]} r^{T,N}_j \geq  & 3\frac{\log\qi }{\log\qi - \log k^*} , \, k^*< k_1\Big] 
	\overset{(a)}=
	\Pr\Big[\max_{j\in[\qi]} r^{Z_{k^*}}_j \geq  3\frac{\log\qi }{\log\qi - \log k^*} ,\, k^*< k_1\Big]
	\nonumber\\
&\leq
	\Pr\Big[ \exists k< k_1 \text{ s.t. }
	 \max_{j\in[\qi]} r^{Z_{k}}_j \geq  3\frac{\log\qi }{\log\qi - \log k} \Big]
		\overset{(b)}\leq
k_1 \frac{5}{\qi^2} \leq \frac{5}{\qi} 
 \label{eq:IL-mart-k-K_1}
\end{align}
(a) uses $ r^{Z_{k^*}}_j  =   r^{(T,N)}_j $.  (b) uses a union bound and~\eqref{eq:IL-mart-k-K_1_1}. Last inequality uses $k_1<\qi$.

\paragraph{Step~4}
This step uses a variation of the Birthday Paradox to bound $\max_{j\in[\qi]} r^{T,N}_j$.
\begin{align}
	\Pr\Big[\max_{j\in[\qi]} r^{T,N}_j \geq  2,\, k^*< k_0\Big] 
	&=
	\Pr\Big[\max_{j\in[\qi]} r^{Z_{k^*}}_j \geq  2,\, k^*< k_0\Big]
\leq
	\Pr\Big[ \exists k< k_0 \text{ s.t. } \max_{j\in[\qi]} r^{Z_{k}}_j \geq  2\Big]
\nonumber\\
&\overset{(a)}\leq
\Pr\Big[  \max_{j\in[\qi]} r^{Z_{k_0}}_j \geq  2\Big]
=
1- 	\Pr\Big[ r^{Z_{k_0}}_j \leq 1 \text{ for all }{j\in[\qi]} \Big]
\nonumber\\
&\overset{(b)}=
1- 	\Pr\Big[ r^{Z_{k}}_j \leq 1 \text{ for all }{j\in[\qi]} \text{ and }  k\leq k_0\Big]
\nonumber\\
& = 
1- 	
\prod_{m=1}^{k_0}\Pr\Big[ r^{Z_{m}}_j \leq 1 \text{ for all }{j\in[\qi]} \,\big|\, r^{Z_{m-1}}_j \leq 1 \text{ for all }{j\in[\qi]} \Big]
\nonumber\\
&\overset{(c)}{=}
1- \prod_{m=1}^{k_0}
\left(1-\frac{m-1}{\qi}\right)
	 \leq
	  1- \big(1-\frac{k_0}{\qi}\big)^{k_0}
	  \overset{(d)}{\leq}
	\frac{2 k_0^2}{\qi}
	\leq 
	\frac{2}{9}
	\,.\label{eq:IL-mart-k-small}
\end{align}
(a) and (b) use the fact that $r^{Z_{k}}_j$ is a nondecreasing function of $k$. 
We define $r_j^{Z_0}=0$. The type of the $(k+1)$-th drawn item is independent of the type of the previous $k$ drawn items. Hence, conditional on $r^{Z_k}=\big(r_1^{Z_k}, \cdots,r_{\qi}^{Z_k}\big)$, the random variable $r^{Z_{k+1}}$ is independent of  $r^{Z_{k-1}}$. This gives equality (c). 
(d) uses $\exp\big(k\log(1-k/\qi)\big)\geq \exp\big(-k^2/(\qi-k)\big)\geq \exp\big(-2k^2/\qi\big) \geq 1-2k^2/\qi$ for $k\leq k_0\leq \sqrt{\qi}\leq \qi/2$. 

We put it all together,
\begin{align*}
\Pr\Big[ \max_{j\in[\qi]} r_j^T& \geq \Theta_{\qi}(f_T+g_T)\Big]
 \overset{(a)}{=}
\Pr\Big[\max_{j\in[\qi]}r_j^{Z_{k^*}} \geq \Theta_{\qi}(k^*)\Big]\\
 &
 \overset{}{\leq}
\Pr\Big[\max_{j\in[\qi]}r_j^{Z_{k^*}} \geq \Theta_{\qi}(k^*), k^*\geq k_2\Big]
 + 
 \Pr\Big[\max_{j\in[\qi]}r_j^{Z_{k^*}} \geq \Theta_{\qi}(k^*),\,k_0\leq k^*<k_2\Big]
 \\
& \quad
 + 
 \Pr\Big[\max_{j\in[\qi]}r_j^{Z_{k^*}} \geq \Theta_{\qi}(k^*),\,\leq k^*<k_1\Big]
  + 
 \Pr\Big[\max_{j\in[\qi]}r_j^{Z_{k^*}} \geq \Theta_{\qi}(k^*),\,k_0\leq k^*<k_0\Big]
\\
& \overset{(b)}\leq
\frac{2}{\qi^2}+ \frac{5}{\qi}+ \frac{2}{9}\leq \frac{1}{2}
\,.
	\end{align*}
(a) uses the definition of $Z_k$ and $k^*$. 
(b) 
uses~\eqref{eq:IL-mart-k-large},~\eqref{eq:IL-mart-k-medium},~\eqref{eq:IL-mart-k-K_1} and ~\eqref{eq:IL-mart-k-small}. Last inequality uses $\qi>10$. 
\end{proof}	
	
 Plugging $\min_u \gamma_u^T\geq T/\max_{\ti\in[\qi]}r_j^T$, given in Equation~\eqref{eq:il-gammau-rmax} in Claim~\ref{cl:rjTproof} proves
\[
\Pr\left[
\min_{u\in[N]}\gamma_u^T \leq \frac{T}{\Theta_{\qi}(f_T+g_T)} \right] 
	\leq 
\Pr\left[ 
\max_{\ti\in[\qi]}r_j^T \geq \Theta_{\qi}(f_T+g_T)
\right]
\leq 
\frac{1}{2}
\,.\]
This statement and Claim \ref{cl:IL-upperbdForA2} imply that with probability at least $1/2$ we have 
$$\good(T)\leq NT- N\frac{T}{\Theta_{\qi}(f_T+g_T)}\,.$$ 
 Combining this bound and 
	Claim~\ref{cl:IL-upperbdForA1} gives that with probability at least 1/2,
	\begin{align}
\good(T)
&{\leq}
NT-\max\bigg\{f_T \thril + g_T, N\frac{T}{\Theta_{\qi}(f_T+g_T)} \bigg \}\,.
	\label{e:goodBound}
	\end{align} 

Using the definition of the function $\Theta_{\qi}(k)$  in Claim~\ref{cl:rjTproof} and some algebra, one can show that for  any $k>0,$	  
		\begin{align}
		\max\bigg\{k, &\frac{NT}{\Theta_{\qi}(k)} \bigg \}
{\geq} 
\begin{cases}
\frac{NT}{2}, 
\quad &\text{if }  
T< \frac{2\sqrt{\qi}}{3N}
\\[8pt]
\frac{NT}{5\log \qi}\big[\log(8\qi\log\qi) - \log(NT)\big], 
\quad &\text{if }  
 \frac{2\sqrt{\qi}}{3N}\leq T< \frac{4\qi \log \qi}{N}
\\[8pt]
\frac{NT}{8\log \qi},
\quad &\text{if } 
\frac{4\qi \log \qi}{N}\leq T < \frac{16\qi(\log \qi)^2}{N}
\\[8pt]
\frac{{1}}{2}\sqrt{NT\qi},  
\quad &\text{if } 
\frac{16\qi(\log \qi)^2}{N}\leq T\,,
\end{cases}							
\nonumber
	\end{align}  
	or alternatively, $\max\big\{k, \frac{NT}{\Theta_{\qi}(k)} \big \}\geq N Z(T)$   
where  $Z(T) $ is defined in~\eqref{eq:IL-defineZT}. 	To prove this, we show that if $k<Z(T)$, then $\Theta_{\qi}(k)\leq \frac{NT}{Z(T)}$ for each regime of parameter $T$. Note that the above bound is not tight, but it is chosen such that this lower bound (and consequently the function $Z(T)$ which is a scaling of the above lower bound as defined in~\eqref{eq:IL-defineZT}) is continuous up to a multiplicative constant  factor.

Equation~\eqref{eq:IL-defE} shows that $\good(T)\leq NT$ with probability one.  
	Hence, 
	\[\Exp{\good(T)}\leq  \frac12 NT + \frac12(\textrm{RHS of~\eqref{e:goodBound}}) \leq NT 
	-
	\frac{1}{2}N Z(T) 
	\,. \]
This completes the proof of the lemma.
\end{proof}

\section{Discussion}
In this paper, we analyzed the performance of online collaborative filtering within a latent variable model for the preferences of users for items. We proposed variants of user-user CF and item-item CF that explicitly explore the preference space. 
 We also proved lower bounds for regret in the extreme regimes of parameters corresponding to user-structure only (no structure in item space) and item-structure only (no structure in user space). The lower bounds showed that the proposed algorithms are almost information-theoretically optimal in these parameter regimes. 

Adaptivity to unknown time time horizon $T$, as required to bound the anytime regret, is achieved via a doubling trick whereby the algorithm is run afresh at a growing sequence of epochs. In practice one would surely benefit from using knowledge gained from exploration in earlier epochs instead of starting from scratch at each epoch. We mentioned how the user-user algorithm can be modified to be adaptive to the number $q_U$ of user types, but it is less obvious how to make the item-item algorithm adaptive without resorting to an impractical trick analogous to the doubling trick used for $T$.

It is possible to modify all of the proposed algorithms to handle i.i.d. noise in the user feedback. We did this only for the user-user algorithm, but it is straightforward to do so also for the item-item algorithm. 
A hybrid algorithm, exploiting structure in both user space and item space, that is nearly information-theoretically optimal in all regimes appears in a forthcoming paper. 

While various insights were obtained through the analysis carried out in the paper, the assumed randomly generated user preference matrix is unrealistic. A reasonable next objective is to perform a similar analysis with a more flexible model for user preferences, perhaps described by a low-rank matrix or a graphical model.

\section*{Acknowledgment}
We are indebted to Devavrat Shah, George Chen, and Luis Voloch for many inspiring discussions and thank Alexander Rakhlin for suggesting several references.
This work was supported in part by grants NSF CCF-1565516, ONR N00014-17-1-2147, and DARPA W911NF-16-1-0551. 

\appendix

\section{Concentration Lemmas} \label{s:lemmas}
The following lemma is derived by application of Chernoff bound to Binomial variables \cite{chung2006concentration}.
\begin{lemma}[Chernoff bound]\label{l:Chernoff}
Let $X_1,\cdots,X_n \in [0,1]$ be independent random variables. Let $X=\sum_{i=1}^n X_i$ and $\bar{X}=\sum_{i=1}^n \Ex X_i$. Then, for any $\epsilon>0$,
\begin{align*}\Pr\big[X\geq (1+\epsilon)\bar{X}\big]
&\leq \exp\Big( - \frac{\epsilon^2}{2+\epsilon}\bar{X}\Big)\leq \max\Big\{\exp\Big( - \frac{\epsilon^2}{3}\bar{X}\Big), \exp\Big( - \frac{\epsilon}{2}\bar{X}\Big)\Big\}
\\
\Pr\big[X\leq (1-\epsilon)\bar{X}\big]
&\leq \exp\Big( - \frac{\epsilon^2}{2}\bar{X}\Big)
\\
\Pr\big[|X-\bar{X}|\geq \epsilon\bar{X}\big]
&\leq 2 \max\Big\{\exp\Big( - \frac{\epsilon^2}{3}\bar{X}\Big), \exp\Big( - \frac{\epsilon}{2}\bar{X}\Big)\Big\}
\end{align*}
\end{lemma}

\begin{lemma}[McDiarmid~\cite{mcdiarmid1998concentration}]\label{l:martingaleBound}
	Let $X_1,\cdots,X_n$ be a martingale adapted to filtration $(\Fc_n)$ satisfying 
	\begin{enumerate}
		\item[(i)] $\mathrm{Var}(X_i|\Fc_{i-1})\leq \sigma_i^2$,  for $1\leq i\leq n$, and
				\item[(ii)]	$|X_i - X_{i-1}|\leq M$, for $1\leq i\leq n$.
\end{enumerate}
Let $X = \sum_{i=1}^n X_i$.
Then
$$
\Pr\big[X - \Ex X \geq r\big] \leq \exp\bigg( -\frac{r^2}{2\big(\sum_{i=1}^n \sigma_i^2+Mr/3\big)}\bigg)\,.
$$
\end{lemma}

\begin{lemma}[Balls and bins: tail bound for number of nonempty bins]
\label{l:ballsbins}
Suppose $m\leq n/4$. If $m$ balls are placed into $n$ bins each independently and uniformly at random, then with probability at least $1-\exp(-m/2)$ at least $m/2$ bins are nonempty.
\end{lemma}

\begin{proof}

Any configuration with at most $m/2$ nonempty bins has at least $n-m/2$ empty bins. Thus we may bound the probability of having \emph{some} set of $n-m/2$ bins be empty. There are ${n\choose n-m/2}={n\choose m/2}$ possible choices for these empty bins, and each ball has to land outside of these, which has probability $[(m/2)/n]^m$. Thus, using union bound, the probability of at most $m/2$ nonempty bins is bounded by
 $${n\choose m/2}\left(\frac{m/2}{n}\right)^m\leq \left(\frac{n\cdot e}{m/2}\right)^{m/2} \left(\frac{m/2}{n}\right)^m \leq \left(\frac{m e}{2n}\right)^{m/2}\leq \exp(-m/2)\,,$$
where we used $m\leq n/4$.
\end{proof}

The following generalization of the above lemma is used in the analysis of the recommendation system in noisy setup.

\begin{lemma}[Generalized Balls and bins: tail bound for number of nonempty bins]
\label{l:ballsbins-general}
Fix $0<a<1$ and $c>0$. Define $b\triangleq \min\{t: -(1-a/t)\log a + (1-a)/t \log(1-a)>c\}$. Suppose $m\leq n/b$. If $m$
 balls are placed into $n$ bins each independently and uniformly at random, then with probability at least $1-\exp(-c m)$ at least $na$ bins are nonempty.
\end{lemma}

\begin{proof}

Any configuration with at most $m/2$ nonempty bins has at least $n-m/2$ empty bins. Thus we may bound the probability of having \emph{some} set of $n-m/2$ bins be empty. There are ${n\choose n-m/2}={n\choose m/2}$ possible choices for these empty bins, and each ball has to land outside of these, which has probability $[(m/2)/n]^m$. Thus, the probability of at most $m/2$ nonempty bins is bounded by
 $${n\choose m/2}\left(\frac{m/2}{n}\right)^m\leq \left(\frac{n\cdot e}{m/2}\right)^{m/2} \left(\frac{m/2}{n}\right)^m \leq \left(\frac{m e}{2n}\right)^{m/2}\leq \exp(-m/2)\,,$$
where we used $m\leq n/4$.
\end{proof}

The following lemma records a simple consequence of linearity of expectation.
\begin{lemma}[Balls and bins: bound for the expected number of nonempty bins]
\label{l:ballsbins2}
If we throw $m$ balls into $n$ bins independently uniformly at random, then, the expected number of nonempty bins is $n[1-(1-1/n)^m]\,.$
\end{lemma}


\section{Converting to anytime regret} \label{s:any-time-reg-alg}
The \textit{doubling trick} converts 
an online algorithm designed for a finite known  time horizon to an algorithm that does not require knowledge of the time horizon and yet achieves the same regret (up to multiplicative constant) at any time \cite{cesa2006prediction} (i.e., \textit{anytime regret}). 

The trick is to divide time into intervals and restart algorithm at the beginning of each interval. 
Let $A(T)$ be an online algorithm taking the known time horizon as input and achieving regret $\mathrm{R}(T)$ at time $T$. 
There are two regret scalings of interest.  (1) If $\mathrm{R}(T) = O(T^\alpha)$ for some $0<\alpha < 1$, then to achieve anytime regret, the doubling trick uses time intervals of length $2, 2^2, 2^3,.. , 2^m$. This achieves regret of at most $\mathrm{R}(T)/(1-2^\alpha)$ at time $T$ for any $T$. (2) Alternatively, if $\mathrm{R}(T)= O(\log T)$, then using intervals of length $2^2, 2^{2^2}, .., 2^{2^m}$ achieves regret of at most $4\mathrm{R}(T)$ at time $T$ for any $T$.

Clearly, different scalings can be used before and after $T_1$ if the algorithm achieves regret $O (\log T)$ for $T<T_1$ and $O(\sqrt{T})$ if $T\geq T_1$, as is the case for the proposed item-item CF algorithm.

\section{Proof of Lemma~\ref{l:user-partition-prob-noisy}}
\label{sec:user-partition-proof-noisy}
We show that  $\Pr[ B_{uv}^c ]\leq 2\epsilon/(N^2)$ for any pair of users $u,v\in[N]$. Using a union bound over the ${N\choose 2}$ pairs of users  gives $\Pr[B^c]\leq \epsilon$. According to Line 7 of the Algorithm~\ref{l:user-partition-prob-noisy}, $g_{u,v}=\ident{ \{\sum_{s=1}^{\thr} {L_{u, a_{u,s}} L_{v,a_{v,s}}} \geq \lambda\thr \} }$ . At $s\leq \thr$, the same item is recommended to all users: $a_s:=a_{u,s}=a_{v,s}$. Hence, 
$L_{u,a_{u,s}}= \xi_{\tau_U(u),\tau_I(a_s)} z_{u,a_s}$
and
$L_{u,a_{u,s}}= \xi_{\tau_U(v),\tau_I(a_s)} z_{v,a_s}$.
First, we look at the users of the same type:

\begin{align*}
\Pr\big[ &g_{u,v}\neq 1 \big| \tau_U(u)=\tau_U(v)\big]
\stackrel{(a)}=
\Pr\Big[ \sum_{s=1}^{\thr} {L_{u,a_{u,s}} L_{v,a_{v,s}}} < \lambda\thr  \big| \tau_U(u)=\tau_U(v)\Big]
\\
&\stackrel{(b)} = 
\Pr\Big[ \sum_{s=1}^{\thr}   \xi_{\tau_U(u),\tau_I(a_s)} z_{u,a_s} \,\xi_{\tau_U(v),\tau_I(a_s)} z_{v,a_s} < \lambda\thr  \big| \tau_U(u)=\tau_U(v)\Big]
\\&
 \stackrel{(c)}= 
\Pr\Big[ \sum_{s=1}^{\thr}    z_{u,a_s} z_{v,a_s} < \lambda\thr  \big| \tau_U(u)=\tau_U(v)\Big]
 \stackrel{(d)}= 
\Pr\Big[ \sum_{s=1}^{\thr}    z_{u,a_s} z_{v,a_s} < \lambda\thr \Big]
\stackrel{(e)}\leq 
\exp\big( -\frac{(1-2\gamma)^2 \thr }{18} \big) 
\stackrel{(f)}\leq
 \frac{2\epsilon}{N^2}\,,
\end{align*}
where (a) uses the definition of $g_{u,v}$ according to the Line 7 of Algorithm~\ref{alg:partitioning-noisy-user}. (b) uses the noise model ~\eqref{eq:noisemodel}. If $\tau_U(u)=\tau_U(v)$ then $\xi_{\tau_U(u),\tau_I(a_s)} =\xi_{\tau_U(v),\tau_I(a_s)}$ which gives (c). The variables $z_{u,i}$ are independent of other variables in the model which implies (d).  For users $u\neq v$ and any item $i$, $\Ex[z_{u,i} z_{v,i}]=(1-2\gamma)^2$. Lemma \ref{l:Chernoff} and the choice of $\lambda$ in Algorithm~\ref{alg:partitioning-noisy-user} gives (e). The choice of $\thr$ gives the result (f).

Now, we look at the pair of users $u$ and $v$ such that $\tau_U(u)\neq \tau_U(v)$.
Define $\mathcal{L}_{u,v}=\big| \{ j : \xi_{\tau_U(u),j}\neq \xi_{\tau_U(v),j}  \}\big|$ to be the set of item types on which user types $\tau_U(u)$ and $\tau_U(v)$ disagree. 
\begin{align*}
\Pr\big[g_{u,v}\neq 0 & \big| \tau_U(u)\neq\tau_U(v)\big]
 \stackrel{(a)} = 
\Pr\Big[ \sum_{s=1}^{\thr}   \xi_{\tau_U(u),\tau_I(a_s)} z_{u,a_s} \xi_{\tau_U(v),\tau_I(a_s)} z_{v,a_s} \geq \lambda\thr  \big| \tau_U(u)\neq\tau_U(v)\Big]
\\
& \stackrel{(b)} \leq
\Pr\Big[ \,| \mathcal{L}_{u,v} | < 5\qi/12 \,\big| \,\tau_U(u)\neq\tau_U(v) \Big]
\\
& + 
\Pr\Big[\sum_{s=1}^{\thr}   \xi_{\tau_U(u),\tau_I(a_s)} \xi_{\tau_U(v),\tau_I(a_s)} z_{u,a_s} z_{v,a_s} \geq \lambda\thr \, \big|\,  | \mathcal{L}_{u,v}| > 5 \qi/12, \tau_U(u)\neq\tau_U(v)\Big]
\\
& \stackrel{(c)}\leq 
\exp\big(-\qi/144\big) + \exp \big(-\thr \,(1-2\gamma)^2\,/4\big)\leq 2\epsilon/N^2\,.
\end{align*}
(a) uses the definition of $g_{u,v}$ in Line 7 of Algorithm~\ref{alg:partitioning-noisy-user}. 
Total probability lemma gives (b). 
Conditional on $\tau_U(u)\neq\tau_U(v)$, the variables $ \xi_{\tau_U(u),j}$ and  $ \xi_{\tau_U(v),j}$ are independently uniformly distributed.  Using the definition on $\mathcal{L}_{u,v}$, this implies that $|\mathcal{L}_{u,v}|$ is the sum of $\qi$ i.i.d. uniform Bernouli random variables. Hence, Chernoff bound in Lemma~\ref{l:Chernoff} gives the bound on the first term in (c). 
To bound the second term, note that the items $a_{s}$ are chosen independently of feedback for $s\leq \thr$. Hence, conditional on $|\mathcal{L}_{u,v}|\geq 5\qi/12$,  the variables $\xi_{\tau_U(u),\tau_I(a_s)} \xi_{\tau_U(v),\tau_I(a_s)}=-1$ with probability at least $5/12$. The variables $z_{u,a_s}$ and $z_{v,a_s}$ are independent of other parameters of the model and algorithm and $z_{u,a_s}z_{v,a_s}=-1$ with probability $2\gamma(1-\gamma)$. Hence, $\xi_{\tau_U(u),\tau_I(a_s)} \xi_{\tau_U(v),\tau_I(a_s)} z_{u,a_s} z_{v,a_s}=-1$  with probability at least $[5+4\gamma(1-\gamma)]/12$. Using Chernoff bound in Lemma~\ref{l:Chernoff} gives the second term in (c). The assumption $\qi>144 \log(N^2/\epsilon)$ and the choice of $\thr$ gives the result.\qedhere

\bibliographystyle{IEEEtran}
\bibliography{RSbib.bib} 
\end{document}